\documentclass{article}

\usepackage[utf8]{inputenc} 
\usepackage[T1]{fontenc}    
\usepackage{hyperref}       
\usepackage{url}            
\usepackage{booktabs}       
\usepackage{amsfonts}       
\usepackage{nicefrac}       
\usepackage{microtype}      
\usepackage{pifont}
\usepackage{wrapfig}
\usepackage{float}
\usepackage{enumitem}
\usepackage{caption}
\usepackage{soul}

\usepackage[accepted]{icml2021}

\usepackage{filecontents}
\usepackage{mathtools, amsmath,amssymb, amsthm}
\usepackage{bbm}

\usepackage{amsfonts,bm,dsfont,color}

\usepackage{mathtools}

\newtheorem{theorem}{Theorem}
\newtheorem{proposition}{Proposition}
\newtheorem{claim}{Claim}
\newtheorem{lemma}{Lemma}
\newtheorem{definition}{Definition}
\newtheorem{remark}{Remark}

\def\H{\mathcal{H}}

\def\G{\mathcal{G}}

\def\X{\mathcal{X}}
\def\T{\mathcal{T}}

\def\D{\mathcal{D}}
\def\G{\mathcal{G}}

\def\E{\mathbb{E}}

\def\1{\mathbf{1}}
\def\P{\mathbb{P}}
\def\R{\mathbb{R}}
\def\N{\mathbb{N}}

\newcommand{\mc}[1]{\mathcal{#1}}
\newcommand{\mb}[1]{\mathbf{#1}}

\newcommand{\norm}[1]{\Big\lVert#1\Big\rVert}

\DeclareMathOperator*{\argmax}{arg\,max}
\DeclareMathOperator*{\argmin}{arg\,min}

\DeclareMathOperator*{\simp}{\mathbf{\triangle}}

\def\t{\top}
\def\one{\mathbbm{1}}
\def\zero{\bm{0}}

\newcommand{\ignore}[1]{}
\newcommand{\floor}[1]{\Big\lfloor #1 \Big\rfloor}

\newcommand{\wt}[1]{\widetilde{#1}}

\DeclareMathOperator*{\dis}{\theta}
\DeclareMathOperator*{\disr}{\text{DIS}}

\DeclareMathOperator*{\err}{\text{err}}

\newcommand{\param}{\mu}

\def\hstar{h^*}
\newcommand{\errstar}{\nu}
\newcommand{\vc}{d}

\DeclareMathOperator*{\labn}{\eta}
\DeclareMathOperator*{\Hx}{\H_{x}}
\newcommand{\rad}{{\zeta}}

\begin{document}

\twocolumn[
\icmltitle{Improved Algorithms for Agnostic Pool-based Active Classification}



\icmlsetsymbol{equal}{*}

\begin{icmlauthorlist}
\icmlauthor{Julian Katz-Samuels}{to}
\icmlauthor{Jifan Zhang}{to}
\icmlauthor{Lalit Jain}{to}
\icmlauthor{Kevin Jamieson}{to}
\end{icmlauthorlist}

\icmlaffiliation{to}{Paul G. Allen School of Computer Science and Engineering, University of Washington, Seattle, WA}

\icmlcorrespondingauthor{Kevin Jamieson}{jamieson@cs.washington.edu}


\vskip 0.3in
]



\printAffiliationsAndNotice{}  

\begin{abstract}
We consider active learning for binary classification in the agnostic pool-based setting.
The vast majority of works in active learning in the agnostic setting are inspired by the CAL algorithm where each query is uniformly sampled from the disagreement region of the current version space.
The sample complexity of such algorithms is described by a quantity known as the disagreement coefficient which captures both the geometry of the hypothesis space as well as the underlying probability space.
To date, the disagreement coefficient has been justified by minimax lower bounds only, leaving the door open for superior instance dependent sample complexities. 
In this work we propose an algorithm that, in contrast to uniform sampling over the disagreement region, solves an experimental design problem to determine a distribution over examples from which to request labels. 
We show that the new approach achieves sample complexity bounds that are never worse than the best disagreement coefficient-based bounds, but in specific cases can be dramatically smaller.
From a practical perspective, the proposed algorithm requires no hyperparameters to tune (e.g., to control the aggressiveness of sampling), and is computationally efficient by means of assuming access to an empirical risk minimization oracle (without any constraints). 
Empirically, we demonstrate that our algorithm is superior to state of the art agnostic active learning algorithms on image classification datasets.
\end{abstract}

\section{Introduction}

Most applications of machine learning have an enormous amount of unlabeled data. Yet, many powerful machine learning methods require that this data be labeled and reliable labels are costly since they require human intervention. The cost of providing labels has become one of the main bottlenecks in applications of machine learning, generating much interest in the problem of \emph{active classification} where the learner is given an unlabeled pool of examples and her goal is to identify an accurate hypothesis using the minimum number of labels possible \cite{settles2011theories}. 

One of the most popular algorithmic paradigms is disagreement-based active classification \cite{hanneke2014theory}. Under this approach, after observing $k$ labels a version space $\mc{V}_k$ of the most promising classifiers is maintained, and the learner queries an example $x$ if there are two hypotheses $h_1$ and $h_2$ belonging to $\mc{V}_k$ that disagree on the label of $x$. This approach has received much attention because it applies to generic hypothesis classes, it can be made robust to label noise, and it can be efficient by using a constrained cost-sensitive classification oracle, a problem for which there are many reasonable heuristics~\cite{agarwal2018reductions, beygelzimer2010agnostic}.

However, disagreement-based active classification suffers from two significant shortcomings. First, it queries uniformly any example on which there is disagreement even though intuitively some of these examples may be much more informative than others. Second, disagreement-based active classification algorithms tend to take a naive union bound over all hypotheses, which ignores many of the dependencies among the hypotheses. 
Indeed, recent work in pure exploration combinatorial and linear bandits has shown that such naive union bounds can be highly suboptimal and have a significant impact on empirical performance \cite{cao2019disagreement,jain2019new,katz2020empirical}.
Given that these naive union bounds are very loose and appear in the confidence bounds used by the algorithms, in practice, many works instead replace these union bounds with a constant that can be tuned to control the aggressiveness of the algorithm \cite{beygelzimer2010agnostic, huang2015efficient}.
Unfortunately, this constant introduces a hyperparameter to the active learning algorithm that is difficult to set before seeing lots of data.

We design a new algorithm for pool-based active classification that addresses these shortcomings. It optimizes a novel experimental design objective that finds the best subset of examples in the disagreement region to query in order to identify the best classifier. It avoids wasteful union bounds by adapting to the geometry of the hypothesis space and thus avoiding the need to choose hyperparameters. We introduce a new notion of sample complexity inspired by experimental design that improves on disagreement-based active classification by a factor up to $\sqrt{n}$ where $n$ is the size of the pool while being only a logarithmic factor worse than disagreement-based learning in the worst case.

\subsection{Preliminaries}

Let $\X$ denote the input space, and let $\{x_1,\ldots, x_n\} \subset \X$ denote a pool of examples.  Let $\H$ denote a class of hypotheses where each $h : \X \mapsto \{0,1\}$ assigns a label to each example in the pool. Let $\Hx:= \{(h(x_i))_{i \in [n]} : h \in \H \}$ denote the set of labelings over the pool induced by the hypothesis class $\H$. Let $\vc$ denote the VC dimension of $\H$. 
When example $i \in[n]$ is queried, the agent receives label $Y_i \sim \text{Bern}({\labn}_{i})$ where $\eta = (\eta_i)_{i=1}^n\in [0,1]^n$. We define the error of a hypothesis $h \in \H$ on the pool of examples as given by
\begin{align}
\err(h) 
&= \frac{1}{n} \sum_{i=1}^n \P(Y_i \neq h(x_i) ) \label{eqn:metric_of_merit}\\
&= \frac{1}{n}\sum_{i \in [n] } {\labn}_i(1- h(x_i)) + (1-{\labn}_i)h(x_i).\nonumber
\end{align}
Let $h^* : = \argmin_{h \in \H} \err(h)$ be the hypothesis of minimum error, and let $\errstar = \err(\hstar)$. The goal in active classification is to find an $h \in \H$ with error close to that of $h^*$ using as few label queries as possible. In this paper, we quantify performance as follows: 

\textbf{Problem. Agnostic Pool Based PAC Active Classification:} Given $\epsilon > 0, \delta \in (0,1)$, identify an $\epsilon$-good classifier, that is, an $h \in \H$ such that $\err(h) - \err(h^*) \leq \epsilon $ with probability at least $1-\delta$ using as few labels as possible.


\begin{remark}\label{prop:generalization}
The goal of finding an $\epsilon$-good classifier over a pool of examples is closely related to the goal of using an active classification algorithm to find a classifier with good generalization. Suppose $VCdim(\H) =d$ and let $\D$ be a distribution over $\X \times \{0,1\}$. For $i=1,\dots,n$ let  $(x_i,y_i) \sim \D$. If $\widehat{h}$ satisfies $\err(\widehat{h}) \leq \min_{h \in \H} \err(h) + \epsilon$, then with probability at least $1-\delta$
\begin{align*}
&\P_{(x,y) \sim \D}(\widehat{h}(x) \neq y) \leq \\
&\quad \min_{h \in \H} \P_{(x,y) \sim \D}(h(x) \neq y) + \epsilon + O\Big(\sqrt{\frac{d\ln(1/\delta)}{n}}\Big) .
\end{align*}
by standard passive generalization bounds \cite{boucheron2005theory}.
\end{remark}

\subsection{Main contributions}
We briefly summarize our contributions:
\begin{itemize}[leftmargin=*, noitemsep, topsep=0pt]
    \item We cast pool-based active binary classification as an \emph{adaptive experimental design problem} that computes an optimal sampling distribution over the pool of unlabelled examples. We demonstrate that an $\epsilon$-good classifier can be obtained with probability at least $1-\delta$ by requesting just $\gamma^*(\epsilon) + \rho^*(\epsilon) \log(1/\delta)$ labels if examples to label are drawn from the optimal design, where $\gamma^*(\epsilon)$ and $\rho^*(\epsilon)$ are problem-dependent quantities defined in the next section.
    \item Since this optimal design uses problem dependent information like $\eta$, it is not a constructive strategy or algorithm for a learner. 
    Treating the sample complexity achieved by this optimal design as a target, we design an algorithm that performs sequential stages of experimental design to match the sample complexity of the optimal design, $\gamma^*(\epsilon) + \rho^*(\epsilon) \log(1/\delta)$ up to a $\log(1/\epsilon)$ factor. The algorithm employs the use of a novel estimator that appeals to a chaining argument. Unfortunately, the method is not computationally efficient.  
    \item We propose a second algorithm that is computationally efficient given access to an empirical risk minimization oracle. The price for computational tractability is a slightly worse sample complexity.
    Besides being computationally efficient, our approach avoids the need to tune hyperparameters and the use of a \emph{constrained} empirical risk minimization oracle which are required by other active learning algorithms \cite{beygelzimer2010agnostic,huang2015efficient}.
    \item We compare our sample complexity results to those of state-of-the-art disagreement-based learning algorithms that are given in terms of the so-called disagreement coefficient. We demonstrate that our results, up to log factors, are never worse than previous results, but can be substantially better in certain cases. 
    \item Empirically, we compare our procedure to state-of-the-art algorithms for the agnostic setting including variants of the importance weighted active learning algorithm (IWAL) \cite{beygelzimer2010agnostic} and active cover \cite{huang2015efficient}. We demonstrate that our method is superior across four image classification tasks. 
\end{itemize}

\section{Experimental Design for Active Classification}
We seek to identify an $\epsilon$-good classifier by seeing as few labels as possible. To this end, we can take motivation from \emph{experimental design} to consider the \emph{optimal} sampling distribution over our pool of unlabeled examples $[n]$. 
For an arbitrary distribution $\lambda\in \triangle_n:=\{p\in \mathbb{R}^n: p_i\geq 0, \forall i\in [n]; \sum_{i=1}^n p_i = 1\}$ suppose we sampled $I_1, \cdots, I_t\sim \lambda$ and then observed $y_s$ for each $s \in [t]$. 
Then an unbiased natural estimator for the error of a classifier $h\in \mc{H}$ defined by \eqref{eqn:metric_of_merit} is given by 
\begin{align*}
    \wt{\err}(h) = \frac{1}{t}\sum_{s=1}^t \frac{1/n}{\lambda_{I_s}} \one\{h(x_{I_s}) \neq y_s\}.
\end{align*}
Indeed, by i.i.d. sampling from $\lambda$, we have for any $s \in [t]$ 
\begin{align*}
    \E[\wt{\err}(h))] &= \E\Big[ \frac{1/n}{\lambda_{I_s}} \one\{h(x_{I_s}) \neq y_s\} \Big] \\
    &= \sum_{i=1}^n \P( I_s = i ) \frac{1/n}{\lambda_{i}} \E[ \one\{h(x_{i}) \neq y_s\} | I_s = i ] \\
    &= \frac{1}{n} \sum_{i=1}^n \P(Y_i \neq h(x_i) ) = \err(h) 
\end{align*}
since by definition, $\P(I_s = i) = \lambda_i$.
Likewise, an estimator for the excess risk is given by
\begin{align}
\wt{\err}(h) &- \wt{\err}(\hstar) = \label{eq:standard_est} \\
&\frac{1}{t} \sum_{s=1}^t \frac{1/n}{\lambda_i}( \one\{h(x_{I_s}) \neq y_s\} - \one\{\hstar(x_{I_s}) \neq y_s\}). \nonumber
\end{align}
It is straightforward to show that the variance of $\wt{\err}(h) - \wt{\err}(\hstar)$ is upper bounded by 
    $\frac{1}{n} \sum_{i=1}^n \frac{1}{\lambda_i n^2}\one\{h^*(x_i) \neq h(x_i)\}$, 
using the upper bound $\one\{h(x_{I}) \neq y_s\} - \one\{\hstar(x_{I}) \neq y_s\})\leq \one\{h^*(x_I) \neq h(x_I)\}$.
Applying Bernstein's inequality (and ignoring the $1/t$ term) with probability at least $1-\delta$
\begin{align}
& |\wt{\err}(h) - \wt{\err}(\hstar) - (\err(h) - \err(\hstar))| \lessapprox \nonumber \\
&\quad \sqrt{ \frac{ \sum_{i=1}^n \frac{1}{\lambda_i n^2} \one\{h^*(x_i) \neq h(x_i)\}   \log(|\mc{H}_x|/\delta) }{t} } \label{eq:bernstein_bound_motive}.
\end{align}
This then suggests that to estimate the excess error of this particular $h$ with probability at least $1-\delta$, it suffices to take $t$ large enough to make the RHS of \eqref{eq:bernstein_bound_motive} less than $\epsilon$.
To upper bound the excess risk of every $h \in \mc{H}$ simultaneously, it suffices to take $t \geq \sup_{h \in \mc{H}}  \frac{ \sum_{i=1}^n \frac{1}{\lambda_i n^2} \one\{h^*(x_i) \neq h(x_i)\} }{\max\{ \epsilon^2, (\err(h)-\err(\hstar))^2\}} \log(|\mc{H}|/\delta)$.
If we seek to \emph{minimize} the total number of observations, we simply minimize over all $\lambda \in \triangle_n$, motivating the complexity measure:
\begin{align*}
\rho^*(\epsilon) & := \inf_{\lambda \in \simp_n}  \sup_{h \in \H \setminus \{\hstar \}} \frac{ \sum_{i=1}^n \frac{1}{\lambda_i n^2} \one\{h^*(x_i) \neq h(x_i)\} }{\max(\err(h) - \err(\hstar), \epsilon)^2 } .
\end{align*}
Thus, we'd expect that if $t \geq \rho^*(\epsilon) \log(|\mc{H}|/\delta)$ samples are drawn from the $\lambda$ that minimizes $\rho^*(\epsilon)$, then $\widehat{h} = \arg\min_{h \in \mc{H}} \wt{\err}(h)$ will be $\epsilon$-good.

\subsection{Sidestepping the Naive Union Bound}
A significant shortcoming of the standard approach of applying Bernstein's inequality with a naive union bound is that the the naive union bound incurs an additional factor of $\log(|\Hx|)$ in the sample complexity. For infinite classes, $\log(|\Hx|)$ can be replaced by the VC-dimension of $\Hx$, however this can still be very loose. In practice, active learning algorithms replace $\log(|\Hx|)$ by a tunable parameter $C_0$ \citep{beygelzimer2010agnostic, huang2015efficient}. Ideally $C_0$ would be chosen via cross-validation but since our data is being chosen adaptively, under an active algorithm that depends on $C_0$, it is unclear how to make the choice a priori.

To improve upon the naive union-bound we appeal to results from empirical process theory. Appealing to the Talagrand/Bousquet inequality \cite{boucheron2005theory}, for all $h\in \mc{H}$, especially the empirical risk minimizer $\widehat{h} = \argmin_{h\in \mc{H}}\wt{\err}(h)$, we have
\begin{align*}
\hspace{.15in}&\hspace{-.15in}\wt{\err}(h) - \wt{\err}(h^{\ast})-( \err(h) - \err{h^{\ast}}) \\
\leq& 2\E[\sup_{h\in \H} |\wt{\err}(h) - \wt{\err}(h^{\ast})-( \err(h) - \err{h^{\ast}})|] \\
&+ \sqrt{\frac{\sup_{h\in \H}\sum_{i=1}^n \frac{1}{\lambda_i n^2} \one\{h^*(x_i) \neq h(x_i)\}\log(1/\delta)}{t}}\\
&+\frac{4\sup_{i\in [n]} 1/\lambda_i\log(1/\delta)}{3t}.
\end{align*}
Traditionally, we compute the expectation of the suprema using symmeterization to obtain the Rademacher complexity of $\mc{H}\backslash\{\hstar\}$. In general, the Rademacher complexity is within a $\log(n)$ factor of the Gaussian Width \cite{bartlett2002rademacher}. In particular, 
\begin{align*}
    \E[&\sup_{h\in \H} |\wt{\err}(h) - \wt{\err}(h^{\ast})-( \err(h) - \err{h^{\ast}})|]\\
    &\leq   \frac{1}{\sqrt{t}} \E_{\rad \sim N(0,I)} \Big[ \sup_{h \in \H } \sum_{i\in [n]} \frac{\rad_i}{n\lambda^{1/2}}(\hstar(x_i) - h(x_i))  \Big].
\end{align*}
Using the same argument that motivated $\rho^*(\epsilon)$ but applying Bousquet's inequality instead of Bernstein's inequality, we introduce the following new complexity measure for active classification:
\begin{align*}
\gamma^*(\epsilon) & := \inf_{\lambda \in \simp_n} \E_{\rad} \Big[ \sup_{h \in \H} \frac{\sum_{i\in [n]} \frac{\rad_i (\hstar(x_i) - h(x_i))}{n\lambda_i^{1/2}} }{ \max(\err(h) - \err(\hstar) ,\epsilon)} \Big]^2.
\end{align*}
 Analogous to above, if we ignore the $1/t$ term, we'd expect that if $t \geq \gamma^*(\epsilon) + \rho^*(\epsilon) \log(1/\delta)$ samples are drawn from the $\lambda$ that minimizes the maximum of  $\gamma^*(\epsilon)$ and $\rho^*(\epsilon)$, then $\widehat{h} = \arg\min_{h \in \mc{H}} \wt{\err}(h)$ will be $\epsilon$-good.


We can relate $\gamma^*(\epsilon)$ to $\rho^*(\epsilon)$ in the following way.
\begin{proposition}[\citet{katz2020empirical}]\label{prop:rho_gamma_rel}
$ \gamma^*(\epsilon) \leq c\log(|\Hx|) \rho^*(\epsilon) \leq c\vc \log(\frac{n}{\vc})  \rho^*(\epsilon)$. 
\end{proposition}
The first inequality parallels the application of Massart's finite class lemma to bound the Rademacher complexity in statistical learning theory and the second inequality follows from the Sauer-Shelah Lemma. 
\citet{katz2020empirical} also demonstrates a lower bound on $\gamma^*(\epsilon)$ that is dominated by $\rho^*(\epsilon)$.
In the appendix, we show that $\gamma^*(\epsilon)$ matches the minimax rates for classification given for the hypothesis class of thresholds in \cite{castro2008minimax}. 

\textbf{Main Takeaway:} In Section~\ref{sec:upper_bound} we will establish an algorithm that achieves a sample complexity of $(\gamma^{\ast}(\epsilon) + \rho^{\ast}(\epsilon)\log(1/\delta))\log(1/\epsilon)$ to obtain an $\epsilon$-good classifier with probability greater than $1-\delta$. 
In the next section we compare this result to disagreement based methods.
Note that we will write $\rho^* := \rho^*(0)$ and $\gamma^* := \gamma^*(0)$.

\subsection{Comparison with the Disagreement Coefficient} \label{sec:disagreement}
To date, theoretically grounded active learning algorithms in the agnostic setting are disagreement region sampling methods.
At the beginning of each round $t$ these algorithms construct a version space $\mc{V}\subset\H$ which is defined to be the set of classifiers that have yet to be ruled out by the algorithm using the observed labels up to round $t-1$. These algorithms then choose $x_{I_t}$ to be uniformly sampled from $\disr(\mc{V})$, the disagreement region, which is the set of points on which any two hypotheses in $\mc{V}$ disagree:
\begin{align*}
\disr(\mc{V}) & = \{i : \exists h,h^\prime \in \mc{V} \text{ s.t. } h(x_i) \neq h^\prime(x_i) \}.
\end{align*}
In the notation of the previous section, these algorithms are sampling from $\lambda_t$ where $\lambda_t$ is the uniform distribution supported on $\disr(\mc{V})$ \cite{hanneke2014theory}.

The main complexity measure considered for disagreement based algorithms is the disagreement coefficient defined as 
\begin{align*}
\dis(\xi) & = \sup_{r \geq \xi  }  \frac{|\disr(B(h_*,r))|/n}{r} 
\end{align*}
where $B(h_*,r)$ is the ball of radius $r$ centered at $\hstar$:
\begin{align*}
B(h_*,r) & = \{h \in \H : \frac{ \sum_{i \in [n]} \one\{\hstar(x_i) \neq h(x_i) \}}{n} \leq r \}.
\end{align*}

We consider sample complexity results for finding an $h$ with $\err(h)\leq \nu + \epsilon$, where $\nu = \err(h^{\ast})$  under two common settings. 
\begin{enumerate}[leftmargin=*, noitemsep, topsep=0pt]
\item \textbf{The Agnostic Setting}: we make no assumptions on $\eta\in [0,1]^n$. In this case the best known sample complexities scale like
\begin{align*}
\dis(\epsilon )( \frac{\nu^2}{\epsilon^2} + \log(1/\epsilon)) \vc
\end{align*}
where $d$ is the VC dimension of $\mc{H}$ \cite{hanneke2014theory}. Note that the noiseless setting of $\eta\in \{0,1\}^n$ is a special case.
\item \textbf{The Tsybakov noise condition:} for some $a \in [1, \infty)$ and $\alpha \in (0,1]$ every $h \in \H \setminus \{\hstar\}$ satisfies 
\begin{align*}
\frac{ \sum_{i \in [n]} \one \{\hstar(x_i) \neq h(x_i)\} }{n} & \leq a (\err(h) - \err(\hstar))^\alpha.
\end{align*}
In this case the best known sample complexities scale like:
\begin{align*}
a^2 \frac{1}{\epsilon^{2-2 \alpha}} \dis(a \epsilon^\alpha) \vc \log(1/\epsilon).
\end{align*}
\end{enumerate}



We now compare our claimed sample complexity of $\gamma^{\ast}(\epsilon)+\rho^{\ast}(\epsilon)\log(1/\delta)$ to these known sample complexity results. Define $\Delta_{min} := \min_{h \in \H \setminus \{h_*\}} \err(h) - \err(\hstar)$.

\begin{proposition}\label{prop:disag_coeff_rhostar}
\begin{itemize}[leftmargin=*, noitemsep, topsep=0pt]
\item Suppose that $\labn \in \{0,1\}^n$. 
\begin{align*}
    \rho^*(\epsilon)\leq c\log( n\Delta_{min}^{-1} \vee  \epsilon^{-1})\dis(\epsilon )[1  + \frac{\nu^2}{\epsilon^2}].
\end{align*}

\item Suppose that the Tsabokov noise condition holds for some $a \in [1, \infty)$ and $\alpha \in (0,1]$. Then,
\begin{align*}
\rho^*(\epsilon ) \leq c a^2 \frac{1}{\epsilon^{2-2 \alpha}} \dis(a \epsilon^\alpha) \log( n\Delta_{min}^{-1} \vee  \epsilon^{-1}).
\end{align*}

\end{itemize}
\vspace{-\intextsep}

\end{proposition}

Recall that Propositions \ref{prop:rho_gamma_rel} shows $\gamma^{\ast}(\epsilon) \leq cd\rho^{\ast}(\epsilon)\log(n/d)$. Hence from Proposition \ref{prop:disag_coeff_rhostar}, we see that our sample complexity, $\gamma^* + \rho^*\log(1/\delta)$ is always as good as the state-of-the-art sample complexities of disagreement-based learning up to logarithmic factors in $n$ and $\epsilon^{-1}$ in both settings. 

However, the converse is not true. In general the disagreement based active classification sample complexities can be substantially larger than $\rho^*$ and $\gamma^*$. 

\begin{proposition}\label{prop:disag_coeff_loose}
There exists an instance where for sufficiently small $\xi$, $\dis(\xi) \geq \Omega(n^{1/2})$ while $\rho^* = O(1)$ and $\gamma^* = \log(n)$.
\end{proposition}
We emphasize that this is not just a feature of the analysis; any algorithm that selects queries uniformly at random in the region of disagreement will perform poorly on the instance in the proposition. 
This gap demonstrates a provable improvement over prior art.

\section{Fixed Confidence Algorithm}\label{sec:upper_bound}

\begin{algorithm}[t]\small

\begin{algorithmic}
\STATE {\bfseries Input:} Confidence level $\delta \in (0,1)$.

\STATE$\H_1 \longleftarrow \H$, $k \longleftarrow 1$, $\delta_k \longleftarrow \delta/2k^2$.

\WHILE{$|\H_k| > 1$}

    \STATE Let $\lambda_k$ and $\tau_k$ be the solution and value of the following optimization problem
        \begin{align}
            \hspace{-.05cm} \inf_{\lambda \in \simp_n} & \E_{\rad \sim N(0,I)}\Big[ \max_{h \in \H_k} \sum_{i \in [n]} h(x_i) \frac{\rad_i}{n \lambda_i^{1/2}}\Big]^2  \label{eq:design_chaining} \\ +&2\log(\frac{1}{\delta_k}) \max_{h, h^\prime \in \H_k} \max_{h, h^\prime \in \G}  \sum_{i=1}^n \frac{1}{\lambda_i n^2} \one\{h(x_i) \neq h'(x_i)\} \nonumber
        \end{align}

    \STATE Set $N_k \longleftarrow c\tau_k2^{2(k+1)}$ where $c$ is a universal constant.

    \STATE Query $I_1, \ldots, I_{N_k} \sim \lambda_k$ and receive rewards $y_1, \ldots, y_{N_k}$.

    \STATE Let $\hat{\eta}_k:=\hat\eta(\H_k,\delta_k)$ be the estimator defined in Theorem \ref{thm:exists_est_gamma} for $\H_k$ with failure probability $\delta_k$ using the samples $\{(x_{I_s}, y_{s})\}_{s=1}^{N_k}$
    
    \STATE $\H_{k+1} \longleftarrow \H_k \setminus \{h \in \H_k : \exists h^\prime \text{ such that } \wt{\err}(h^\prime, \hat{\eta}_k) - \wt{\err}(h, \hat{\eta}_k) + \frac{1}{2^{k+1}} \leq 0 \}$.

    \STATE $k \longleftarrow k+1$
\ENDWHILE

\STATE {\bfseries Return:} $\H_{k} = \{\widehat{h}\}$.

\caption{ACED (Active Classification using Experimental Design).}
\label{alg:action_chaining}

\end{algorithmic}
\end{algorithm} 

Algorithm \ref{alg:action_chaining} is an elimination-style algorithm, in the style of $A^2$ \cite{balcan2009agnostic,dasgupta2007general,huang2015efficient,jain2019new}, but optimizes the querying distribution similarly to algorithms from the pure exploration linear bandits literature \cite{fiez2019sequential,katz2020empirical}. It chooses a distribution $\lambda_k$ over the examples in \eqref{eq:design_chaining} that minimizes the confidence bounds from Theorem \ref{thm:exists_est_gamma} and queries enough random examples from $\lambda_k$ to ensure that the estimates of the difference in error rates, $\err(h) - \err(\hstar)$, improve at least by a factor of $2$ for all remaining hypotheses $h \in \H_k$. Using these improved estimates of the gaps, it then eliminates all hypotheses that can be shown to be suboptimal using the confidence bound in Theorem \ref{thm:exists_est_gamma}. 

Given an estimator $\hat{\eta}$ for $\eta$, denote the induced estimate for the error as 
\begin{align}
\wt{\err}(h, \hat{\eta}) &= \frac{1}{n}\sum_{i \in [n] } {\hat{\labn}}_i(1- h(x_i)) + (1-{\hat{\labn}}_i)h(x_i).\nonumber
\end{align}



\begin{theorem}
\label{thm:exists_est_gamma}
Let $\G \subset \H$. There exists an estimator $\hat{\eta}(\G,\delta)$ for $\eta$
constructed from $t$ samples drawn i.i.d. from $\lambda$ such that with probability at least $1-\delta$,
\begin{align*}
\hspace{.1in}&\hspace{-.1in}\sup_{h,h^\prime \in \G} |[\wt{\err}(h, \hat{\eta}) - \wt{\err}(h^\prime,\hat{\eta})] - [\err(h) - \err(h^\prime)]  | \\
\lesssim& \sqrt{\frac{\log(2/\delta) \max_{h, h^\prime \in \G}  \sum_{i=1}^n \frac{1}{\lambda_i n^2} \one\{h(x_i) \neq h'(x_i)\}}{t}}  \\
&+  \sqrt{\frac{\E[\sup_{h \in \G} \sum_{i \in [n]} h(x_i) \frac{\rad_i}{n \lambda_i^{1/2}} ]^2}{t}} .
\end{align*}
\end{theorem}
For now, we treat the estimator in Theorem \ref{thm:exists_est_gamma} as a black-box and defer its discussion until Section \ref{sec:estimators}. Note that unlike the Talagrand/Bousquet inequality presented before \eqref{eq:bernstein_bound_motive}, \emph{this confidence interval does not have a term depending on the inverse of the worst case importance weight}.

Algorithm \ref{alg:action_chaining} attains the following sample complexity.

\begin{theorem}\label{thm:chaining_ub}
Let $\delta \in (0,1)$ and $\epsilon > 0$. With probability at least $1-\delta$ Algorithm \ref{alg:action_chaining} returns $\widehat{h} \in \H$ after $\tau$ samples where $\err(\widehat{h}) \leq \err( h^*) + \epsilon$ and 
\begin{align*}
\tau \lesssim \log(1/\epsilon) [\log(1/\delta) \rho^*(\epsilon) + \gamma^*(\epsilon)].
\end{align*}
\end{theorem}

\section{Fixed Budget Algorithm}
\begin{algorithm}[t]\small

\begin{algorithmic}

\STATE {\bfseries Input:} Budget $T$, tolerance $\epsilon > 0$ 

\STATE $N \longleftarrow \floor{T/\log_2(\epsilon^{-1})}$, and $\hat\eta_{0} = 0$ 

\FOR{$k=1,2,\ldots, \floor{\log_2(\epsilon^{-1})}$}

    \STATE $\tilde{h}_{k} \longleftarrow \argmin_{h \in \H} \wt{\err}(h, \hat{\eta}_{k-1})$.

    \STATE Let $\lambda_k $ be the solution of the following optimization problem
        \begin{align}
        \hspace{-.5cm}\inf_{\lambda \in \simp_n} \!\E_{\rad \sim N(0,I)}\Big[ \max_{h \in \H} \frac{ \sum_{i \in [n]} (\tilde{h}_k(x_i)- h(x_i)) \frac{\rad_i}{n\lambda_i^{1/2}} }{2^{-k+1} +  \wt{\err}(h, \hat{\eta}_{k-1}) - \wt{\err}( \tilde{h}_k, \hat{\eta}_{k-1})  }\Big] \label{eq:action_comp_2}
        \end{align}
    \vspace{-\intextsep} 
    
    \STATE Sample $\{x_{I_1}, \ldots, x_{I_N} \}  \sim \lambda_k$.

    \STATE Query $x_{I_1}, \ldots, x_{I_N} $ and observe $y_1, \ldots, y_{N}$.

    \STATE Compute an estimate $\hat\eta_k$.
\ENDFOR

\STATE {\bfseries Return:} $ \argmin_{h \in \H} \wt{\err}(h, \hat{\eta}_{k})$

\end{algorithmic}

\caption{Fixed Budget ACED.}
\label{alg:fixed_budget_comp}
\end{algorithm}

In many applications, the agent is given a budget of $T$ queries and a performance target $\epsilon > 0$, and the goal is to maximize the probability of outputing a classifier $\widehat{h} \in \H$ such that $\err(\widehat{h}) \leq \err(\hstar) + \epsilon$. We design a new algorithm for this setting that can be made computationally efficient given access to a weighted classification oracle (defined shortly).

Algorithm \ref{alg:fixed_budget_comp} splits the budget into $\floor{\log( \epsilon^{-1})}$ phases. In each phase, the algorithm computes the design that optimizes \eqref{eq:action_comp_2}, the objective of which approximates $\E\Big[\max_{h \in \H \setminus \{\hstar\}} \frac{\sum_{i\in [n]} \frac{\rad_i}{n\lambda^{1/2}}(\hstar(x_i) - h(x_i)) }{ \max(\err(h) - \err(\hstar) ,2^{-k+1}  )}\Big]^2$. 
The algorithm can use any estimator $\hat{\eta}_k$ at each round $k$. 
The next theorem uses the estimator of Theorem~\ref{thm:exists_est_gamma}.


\begin{theorem}\label{thm:fixed_budget}
Let $T \in \N$ and $\epsilon > 0$. 
Let $\widehat{h}$ denote the $h \in \H$ returned by Algorithm \ref{alg:fixed_budget_comp}. There exists an estimator $\hat{\eta}_k$ using the samples $\{(x_{I_s}, y_{s})\}_{s=1}^{N}$ in round $k$ of Algorithm~\ref{alg:fixed_budget_comp} such that for an absolute constant $c >0$
\begin{align*}
\P(\err(\widehat{h}) &\geq \err(\hstar) + \epsilon)\\ 
&\leq \log(n \epsilon^{-1})  \exp\Big(-\frac{c T}{\log( \epsilon^{-1}) [\gamma^*(\epsilon) + \rho^*(\epsilon)]}\Big).
\end{align*}
\end{theorem}

If $T \geq c\log(\log( \epsilon^{-1})) \log(1/\delta) \log( \epsilon^{-1}) [\gamma^*(\epsilon) + \rho^*(\epsilon)]$, then with probability at least $1-\delta$, Algorithm \ref{alg:fixed_budget_comp} outputs $\widehat{h} \in \H$ such that $\err(\widehat{h}) \leq \err(\hstar) + \epsilon$. The proof of Theorem \ref{thm:fixed_budget} leverages the estimator defined in Theorem \ref{thm:exists_est_gamma} for $\H$ and failure probability $\delta_k = \exp(-\Theta( N / \gamma_k ) )$ with $\gamma_k$ equal to the value of \eqref{eq:action_comp_2}. 

\begin{remark}
Given $\{(I_t,y_t)\}_{t=1}^T$ where $I_t \sim \lambda$ define
\begin{align}\label{eqn:importance_estiamtor}
    \hat\eta_\gamma^{(\text{Importance})} = \frac{1}{T} \sum_{t=1}^T \frac{y_t}{\lambda_{I_t}+\gamma} \mb{e}_{I_t}. 
\end{align}
If importance-weighted estimator $\hat\eta_\gamma^{(\text{Importance})}$ with $\gamma=0$ is used in Algorithm \ref{alg:fixed_budget_comp} (with a slightly modified objective function in \eqref{eq:action_comp_2}, see the Supplementary Material), one can obtain a computationally efficient algorithm whose probability of error scales as 
\begin{align*}
&\P(\err(\widehat{h}) \geq \err(\hstar) + \epsilon) \leq \\
&\quad \log(n \epsilon^{-1})  \exp(-\frac{T - \log(|\Hx|) \psi^*(\epsilon)}{\log( \epsilon^{-1}) [\gamma^*(\epsilon) + \rho^*(\epsilon) + \psi^*(\epsilon)]})\\[-30pt]
\end{align*}

where\\
 $\displaystyle\psi^*(\epsilon) \!:=\!\min_{\lambda \in \simp_n} \max_{ \substack{i \in [n] : \exists h \in \H \\ \hstar(x_i) \neq h(x_i)} }  \frac{1/n \lambda_i}{\max(\epsilon, \err(h) - \err(\hstar))}.$

There are instances where $\psi^*(\epsilon) \gg \gamma^*(\epsilon)$ and therefore the cost of computational efficiency is a worse sample complexity. See the appendix for more details.
\end{remark}

\subsection{Discussion of Theorem~\ref{thm:exists_est_gamma}} \label{sec:estimators}
Theorem \ref{thm:exists_est_gamma} above demonstrates the existence of an estimator that avoids any dependence on $\log(|\Hx|)$. 
The construction of the estimator in Theorem \ref{thm:exists_est_gamma} uses generic chaining, a technique that builds a highly optimized union bound to avoid extraneous logarithmic factors \cite{talagrand2014upper}. 
Generic chaining is most easily applied when a given estimator $\hat{\eta}$ satisfies the property that
$\wt{\err}(h, \hat{\eta})-\wt{\err}(h', \hat{\eta})$ is sub-Gaussian for every ``direction'' $h-h'$ of interest (e.g., see \cite{katz2020empirical}).
Though the $\hat\eta_\gamma^{(\text{Importance})}$ estimator has  sub-Gamma tails in general, ruling out its use, the following result shows that for $h-h'$ in a ball under a certain norm, we can construct an estimator for $\wt{\err}(h, \hat{\eta})-\wt{\err}(h', \hat{\eta})$ with a sub-Gaussian-like tail.
\begin{proposition}\label{prop:ips_bound_classification}
Fix $\lambda \in \triangle_n$, $\delta \in (0,1)$, and $h,h^\prime \in \H$. 
If $T$ samples are taken from $\lambda$ and $\hat\eta := \hat\eta^{(\text{Importance})}_\gamma$ is computed with $\gamma = \sqrt{ \frac{ \log(2/\delta) }{3\sum_{i=1}^n \frac{1}{\lambda_i n^2} \one\{h(x_i) \neq h'(x_i)\} }}$ then with probability at least $1-\delta$
\begin{align*}
    |[\wt{\err}&(h, \hat{\eta}) - \wt{\err}(h^\prime,\hat{\eta})] - [\err(h) - \err(h^\prime)]  | \\
    &\left(\sqrt{\tfrac{2}{3}} + 1\right) \sqrt{\frac{2 \sum_{i=1}^n \frac{1}{\lambda_i n^2} \one\{h(x_i) \neq h'(x_i)\} \log(\tfrac{2}{\delta})}{t} }.
\end{align*}
\end{proposition}
The idea behind Theorem \ref{thm:exists_est_gamma} is to apply generic chaining to all $h-h'$, but to use a different $\hat\eta$ (specifically, a different $\gamma$) based on the size of $h-h'$ prescribed by Proposition~\ref{prop:ips_bound_classification}.
Details of the technique can be found in the supplementary materials.

%

\subsection{Computationally Efficient Experimental Design} \label{sec:efficient_optimization}
In this section, we discuss how to solve \eqref{eq:action_comp_2} efficiently given access to a weighted empirical risk minimization oracle, which we will introduce shortly.
First, note that minimizing \eqref{eq:action_comp_2} is equivalent to minimizing $\E_{\rad\sim N (0, I)} [\max_{h\in \H} f(\lambda; h; \rad)]$ with respect to $\lambda$ where 
\begin{align*}
f(\lambda; h; \rad) &:=  \tfrac{ \sum_{i \in [n]} (\tilde{h}_k(x_i)- h(x_i)) \frac{\rad_i}{n\lambda_i^{1/2}} }{2^{-k+1} +  \wt{\err}(h,\hat\eta_{k-1}) - \wt{\err}(\tilde{h}_k,\hat\eta_{k-1})}\\
&:=  \tfrac{ \sum_{i \in [n]} (\tilde{h}_k(x_i)- h(x_i)) \frac{\rad_i}{n\lambda_i^{1/2}} }{2^{-k+1} +  
\sum_{i\in [n]}
(1-2\hat\eta_{k-1,i})(\tilde{h}_k(x_i) -h(x_i)) }.
\end{align*}
It is known that $\E_{\rad\sim N (0, I)} [\max_{h\in \H} f(\lambda; h; \rad)]$ is convex in $\lambda$ \cite{katz2020empirical}, hence we perform the minimization over $\lambda$ via stochastic mirror descent with stochastic gradient $g(\lambda,\rad) = \nabla f(\lambda, \widetilde{h}; \rad)$ where $\rad \sim \mc{N}(0,I)$ and $\widetilde{h} \in \arg\max_{h \in \mc{H}} f(\lambda, h; \rad)$.
To obtain $\widetilde{h}$ for a fixed $\lambda$ and  $\rad$, first note that the value $\max_{h \in \mc{H}} f(\lambda; h; \rad)$ is equal to 
\begin{align*}
& \min_{r \in \R^+} r \, \text{subject to} \enspace a r + b + \max_{h \in \H}\, \sum_{i \in [n]} (c_i r + d_i) h(x_i) \leq 0
\end{align*}
where $a = -2^{-k+1} - \sum_{i\in [n]} (1 - 2\hat{\eta}_{k, i})\tilde{h}_k(x_i)$, $b = \sum_{i\in [n]} \frac{\rad}{n\lambda_i^{1/2}}\tilde{h}_k(x_i)$, $c_i=1 - 2\hat{\eta}_{k-1, i}$ and $d_i=-\frac{\rad}{n\lambda_i^{1/2}}$.


For any fixed positive value of $r$ it suffices to check the constraint. We can then use a line search procedure to find the minimizing value of $r$ (details in Appendix~\ref{sec:implement}).

Thus we have reduced to checking the constraint for a fixed $r\in \R^+$.  Specifically, the difficulty is to solve for $\max_{h \in \H} \sum_{i \in [n]} w_i \cdot h(x_i)$ where $w_i$ are arbitrary weights. This can be reduced to weighted $0/1$-loss minimization problem that is solvable by a weighted classification oracle:
\begin{align*}
    \text{oracle}(\{ \tilde{w}_i, \tilde{x}_i, \tilde{y}_i\}_{i=1}^n)\! :=\! \argmin_{h\in \H} \! \sum_{i\in [n]} \tilde{w}_i \! \cdot \! \1\{h(\tilde{x}_i) \neq \tilde{y}_i\}
\end{align*}
for inputs $\{ \tilde{w}_i, \tilde{x}_i, \tilde{y}_i\}_{i=1}^n$.
Then, 
\begin{align*}
    \max_{h \in \H} \sum_{i \in [n]} w_i \cdot h(x_i)
    = \text{oracle}(\{|w_i|, x_i, \1\{w_i \geq 0\}\}_{i=1}^n). 
\end{align*}

\section{Implementation and Experiments} \label{sec:experiment}

In the previous section we reduced the experimental design objective of \eqref{eq:action_comp_2}  to a weighted 0/1 loss classification problem using weights that are functions of the estimated vector $\widehat{\eta}$. 
In practice we replace this 0/1 loss with a surrogate convex loss, namely the logistic loss.  
However, to implement Algorithm~\ref{alg:fixed_budget_comp} we still have to specify the choice of estimator $\hat\eta$. 
Though the estimator specified in Theorem~\ref{thm:exists_est_gamma} is theoretically grounded, it is difficult to implement in practice since it involves a costly constrained linear optimization problem over the set of hypothesis in $\H_k$. 
As described in Remark~2, it is still possible to have a theoretical guarantee for other estimators such as the IPS estimator. 
As described precisely in Appendix~\ref{sec:implement}, in our implementation we take the estimate for $\hat{\eta}_k$ to be 
\begin{align*}
    \big[ \hat\eta_k^{(\text{Naive})} \big]_i = \text{average}( \{ y_s^{(j)} : I_s^{(j)} = i, s \in [N_j], j\in [k] \}),
\end{align*}
i.e. a simple average of the labels we see. Here $I_s^{(j)}$ indexes the $s$-th query we made in round $j$. In our experiments we only considered the persistent noise setting (i.e., querying the same image more than once would always return the same label as before, or formally, $\eta_i \in \{0,1\}$).
Thus, if we sample a point $x_{I_s}^{(j)}$ (i.e., $I_s^{(j)}$) more than once, we set $y_s^{(j)}$ to be the previously observed label and we did not count this observation in our count of total labels taken.
To take advantage of all of the labels observed so far, we also employ a water-filling technique for sampling in practice (details in Appendix~\ref{sec:implement}).

\textbf{Baselines.} To validate Algorithm~\ref{alg:fixed_budget_comp} we conducted a set of experiments against the following baselines that are considered to be state-of-the-art theoretically-justified methods in disagreement based active learning. Our set of methods are chosen based on the ones considered in \cite{huang2015efficient}, the most recent work of relevance. Details on the precise implementations of these methods are available in the supplementary materials in Appendix~\ref{sec:implement}. 
\begin{itemize}[leftmargin=*, noitemsep, topsep=0pt]
    \item Passive: We considered a passive baseline where we uniformly at random choose samples from our pool, retrain our model on our current samples and report the accuracy.
    \item Importance Weighted Active Learning (IWAL) : IWAL was originally introduced in \citet{beygelzimer2009importance} and is an active learning algorithm in the streaming setting. Our implementation is based on the algorithm presented in \citet{beygelzimer2010agnostic} which we refer to as IWAL0. We also consider variants, IWAL1, and oracular versions ORA-IWAL0, ORA-IWAL1 detailed in \citet{huang2015efficient}.
    \item Online Active Cover (OAC): OAC is described in \citet{huang2015efficient}. We used the implementation of OAC that is available in Vowpal Wabbit \cite{agarwal2014reliable}. 
\end{itemize}

\textbf{Datasets.} We evaluate on the following four real datasets.
\begin{itemize}[leftmargin=*, noitemsep, topsep=0pt]
    \item \textbf{MNIST 0-4 vs 5-9} \citep{lecun1998gradient}. We considered the standard MNIST dataset but in a binary setting where digits 0-4 are labelled as 0, and 5-9 are labelled as 1. Our pool has 50000 images in total, and we classified based on the flattened images (784 dimensions).
    \item \textbf{SVHN 2 vs 7} \citep{netzer2011reading}. We considered the binary classification problem of determining whether a digit was a 2 or a 7 (ignoring all other images). To prevent the logistic classifier from overfitting to arbitrary labels and to restrict the hypothesis class $\H$, we downsample the images to 512 dimensional feature vectors through PCA. There are 16180 images in total.
    \item \textbf{CIFAR Bird vs Plane} \citep{netzer2011reading}. We considered the binary classification problem of determining whether a digit was a bird or a plane (ignoring all other images). To prevent the logistic classifier from overfitting to arbitrary labels and to restrict the hypothesis class $\H$, we downsample the images to 576 dimensional feature vectors through PCA. There are 10000 images in total.
    \item \textbf{FashionMNIST T-shirt vs Pants} \citep{xiao2017fashion}. We considered the binary classification problem of T-shirt vs Pants. Our pool has 12000 images in total, and we classified based on the flattened images (784 dimensions).
\end{itemize}

\textbf{Implementations.} We use two implementations to measure the performances of the algorithms.
\begin{itemize}[leftmargin=*, noitemsep, topsep=0pt]
    \item Implementation from Vowpal Wabbit \citep{agarwal2014reliable} that is used by \citet{huang2015efficient}. The implementation employs an online learner that only updates based on the latest queried label, therefore has time complexity that scales linearly in the number of images $n$.
    \item For our implementation in a batched setting, we retrain the entire classifier to convergence every time new labels become available. We find that the online learner of above can perform significantly better than our batched learner during the first few batches of training. However, our implementation has more stable accuracies during the course of training and performs slightly superior ($<1\%$) in final accuracy. This comes at a cost of an $O(n^2)$ time complexity, which is too expensive in some of the settings.
\end{itemize}
In particular, we only use the Vowpal Wabbit implementation for the OAC experiments and the oracular variants of IWAL algorithms for our MNIST experiement, due to the high computation cost for running these algorithms with exhaustive hyperparameter search. However, we think this is still a fair comparison when evaluating some baselines using the the two implementations since it is the best one can achieve for those baselines within a computation budget (single machine with state-of-art commercialized CPU that runs for a month).

\textbf{Hypothesis Class.} In our implementation, we took the hypothesis space to be the set of linear separators in the underlying feature space. We used the logistic regression implementation in Scikit-learn~\cite{scikit-learn} for our underlying classification oracle.

\begin{minipage}[c]{.49\linewidth}
    \centering
    \includegraphics[trim={0cm 0cm 0.8cm 0.8cm},clip,width=\linewidth]{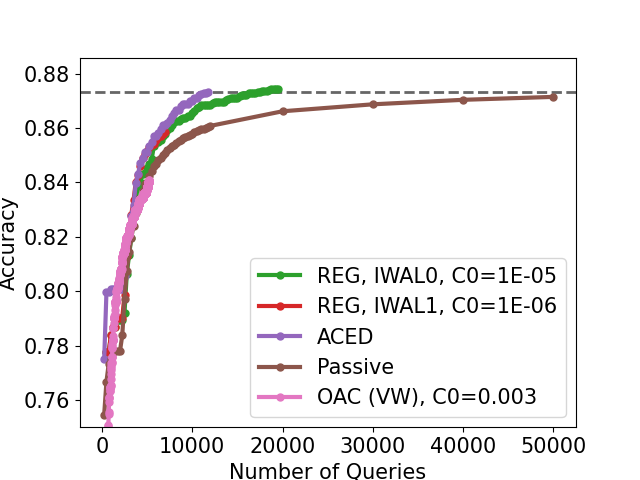}
    \captionof{figure}{MNIST Performance}
    \label{fig:mnist}
\end{minipage}
\begin{minipage}[c]{.49\linewidth}
    \centering
    \includegraphics[trim={0.3cm 0cm 0.5cm 0.8cm},clip,width=\linewidth]{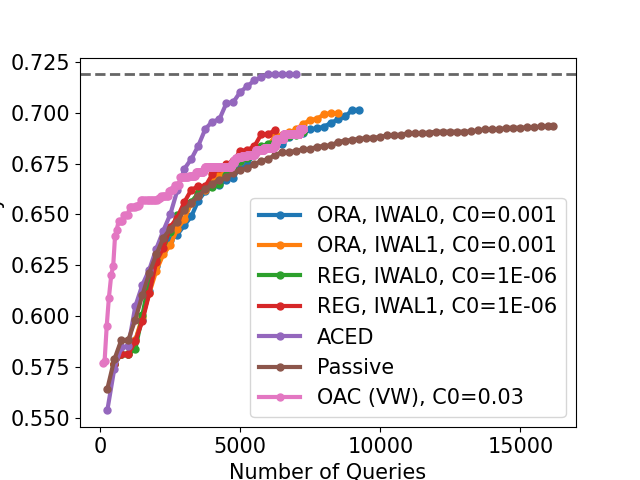}
    \captionof{figure}{SVHN Performance}
    \label{fig:svhn}
\end{minipage}
\begin{minipage}[c]{.49\linewidth}
    \centering
    \includegraphics[trim={0cm 0cm 0.8cm 0.8cm},clip,width=\linewidth]{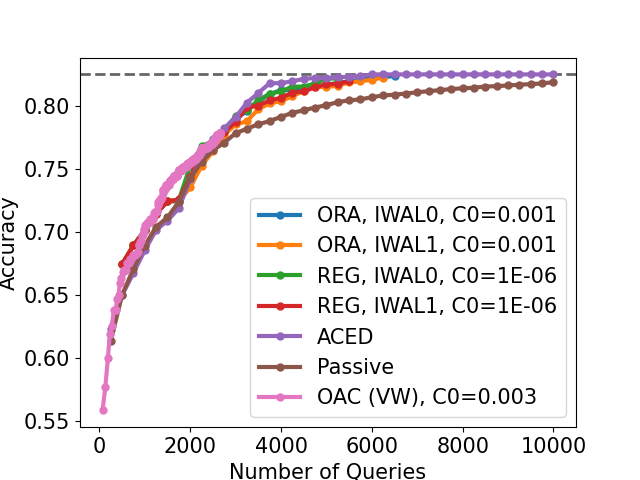}
    \captionof{figure}{CIFAR Performance}
    \label{fig:cifar}
\end{minipage}
\begin{minipage}[c]{.49\linewidth}
    \centering
    \includegraphics[trim={0.3cm 0cm 0.5cm 0.8cm},clip,width=\linewidth]{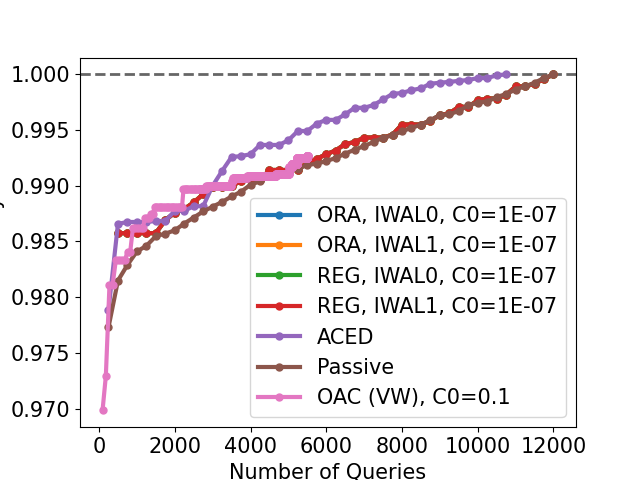}
    \captionof{figure}{FashionMNIST}
    \label{fig:fashion}
\end{minipage}

\textbf{Discussion.} For each of the binary classification datasets, we plot the running maximum accuracy on the unlabelled pool against the number of queries taken as in Figure~\ref{fig:mnist},\ref{fig:svhn},\ref{fig:fashion},\ref{fig:cifar} (full scale images included in Appendix~\ref{sec:full_scale}).
The passive curves are evaluated based on the averages of $10$ runs. In the CIFAR experiment, ACED is an average over $5$ runs. We find the curve in this setting to be very consistent, and that the standard deviations are minimal for visualization. All of the other curves are evaluated based on a single run.
For baselines algorithms proposed in the streaming setting (variants of IWAL and OAC), in each round we uniformly sample an example from the pool, and feed a fixed number of passes. We select the best $C_0$ based on which hyperparameter setting takes the least amount of queries to reach the same level of accuracy. Detailed hyperparameters considered for the baselines are included in Appendix~\ref{sec:base_hparam}. Furthermore, to demonstrate active gains in generalization, we include plots on holdout test sets in Appendix~\ref{sec:test_perf}.

On all four datasets, our algorithm outperforms other baselines by taking much fewer queries to reach the passive accuracy on the entire dataset. Sometimes the active learning algorithms even beat the passive accuracy on the whole dataset, which is a known phenomenon of active learning studied by~\citet{mussmann2018uncertainty}. For the MNIST dataset, we do not include performance curves for the oracular variants of IWAL, since the Vowpal Wabbit implementation turns out to be performing at random chance. We also notice that OAC stops taking queries very early on (no longer making queries when given more passes over the pool). However, when increasing $C_0$, the aggressiveness to make a query, OAC starts performing worst than passive pretty easily. We include Figure~\ref{fig:cifar_sensitivity} in the appendix to demonstrate how sensitive the OAC curves are to the hyperparameter $C_0$, which one cannot tune in real applications.

As a special case, on the FashionMNIST dataset, our binary classification task is linearly separable and the baseline methods fail miserably. For all of the IWAL algorithms on this dataset, we searched in an extended range of hyperparameters than the ones used in the other three tasks. When fixing the random order of the stream, however, all of the baselines become equivalent, and perform almost identical to passive. Since in practice, only one set of hyperparameters can be deployed, this again demonstrates the shortcoming of these baseline algorithms, whereas our method does not rely on any aggressiveness hyperparameter.

\section{Related Work and Discussion}

\textbf{Active Classification:} Active classification has received much attention with a large number of theoretical and empirical works (see \cite{hanneke2014theory} and \cite{settles2011theories} for excellent surveys). \citet{cohn1994improving} initiated research into the study of disagreement based active classification algorithms, proposing CAL for the realizable setting. \citet{balcan2009agnostic} extended disagreement-based active classification to the agnostic case, introducing the method, $A^2$. \citet{hanneke2007bound} provided a general analysis of $A^2$ in terms of the disagreement coefficient, with follow-up works improving on the sample complexity of this approach \cite{dasgupta2007general, hanneke2009adaptive, hanneke2011rates, koltchinskii2010rademacher, hanneke2014theory}. The results in Section~\ref{sec:disagreement} show that our sample complexities are never worse than the ones obtained by this line of work.

An extension of this line of work has aimed to attain similar sample complexities, while leveraging an empirical risk minimization oracle to design more practical algorithms \cite{dasgupta2007general, hsu2010algorithms, beygelzimer2010agnostic, huang2015efficient}. With the exception of \citet{huang2015efficient}, these methods tend to have a conservative query policy that samples uniformly in the disagreement region, leading to an onerous label requirement. While \citet{huang2015efficient} has a more aggressive query policy that does not sample uniformly in the disagreement region, their sample complexity result could also be obtained by sampling uniformly in the disagreement region and, therefore, their theoretical result does not reflect gains from a careful selection of points in the disagreement region. In particular, the dominant term is still the disagreement coefficient and, hence, it can be much worse than our sample complexity on instances such as the one in Proposition \ref{prop:disag_coeff_loose}. 

Recently, \citet{jain2019new} showed that active classification in the pool-based setting is an instance of combinatorial bandits, an observation that is central to our analysis. They provided the first analysis that shows the contribution of each example to the sample complexity providing a more fine-grained result than the disagreement coefficient. We improve on this work by optimizing the sampling distribution in the region of disagreement and using improved estimators such as the one in Theorem \ref{thm:exists_est_gamma}. Proposition 4 of \citet{katz2020empirical} implies that our sample complexity is always better than the sample complexity in \citet{jain2019new}. 


\textbf{Linear and Combinatorial Bandits.}  $\rho^*$ has been shown to be  the dominant term in a lower bound for pure exploration linear bandits and combinatorial bandits \cite{soare2014best,chen2017nearly,fiez2019sequential}. Recently \citet{katz2020empirical} introduced the notion of $\gamma^*$ for linear and combinatorial bandits, showing that it is a lower bound for any non-interactive oracle MLE algorithm. One of our contributions is making the connection between the active classification and linear/combinatorial bandit literature, and showing that we can leverage the results from this work to obtain improved sample complexities for agnostic active classification. 

\bibliography{refs}
\bibliographystyle{icml2021}

\onecolumn

\appendix

\tableofcontents

\section{Generalization}

\begin{proof}[Proof of Remark \ref{prop:generalization}]
Define the event
\begin{align*}
\mc{E} := \{ \forall h \in \H : |\err(h) - \P(h(x) \neq f(x))| \leq c [\sqrt{\frac{\ln(1/\delta)}{n}} + \sqrt{\frac{\vc}{n}}]\}.
\end{align*}
where the randomness is over the draw of the pool $\{x_1, \ldots, x_n\} \sim \D_{\X}$. Using the bounded differences inequality and 3.4 in \cite{boucheron2005theory}, we have by a standard argument that $\P(\mc{E}) \geq 1- \delta$. Suppose $\mc{E}$ holds for the remainder of the proof. Let $\bar{h} = \argmin_{h \in \H} \P_{(x,y) \sim \D}(h(x) \neq f(x))$. Then,
\begin{align*}
 \P(\widehat{h}(x) \neq f(x)) & \leq \err(\widehat{h}) + c [\sqrt{\frac{\ln(1/\delta)}{n}} + \sqrt{\frac{\vc}{n}}]  \\
 & \leq \min_{h \in \H} \err(h) + \epsilon + c [\sqrt{\frac{\ln(1/\delta)}{n}} + \sqrt{\frac{\vc}{n}}]  \\
 & \leq  \err(\bar{h}) + \epsilon + c [\sqrt{\frac{\ln(1/\delta)}{n}} + \sqrt{\frac{\vc}{n}}] \\
 & \leq   \P(\bar{h}(x) \neq f(x)) + \epsilon + 2c [\sqrt{\frac{\ln(1/\delta)}{n}} + \sqrt{\frac{\vc}{n}}].
\end{align*}
\end{proof}

\section{Reduction to Combinatorial Bandits}

We state and prove our results for active classification in the language of combinatorial bandits, a strictly more general problem, which we now introduce.

\textbf{Combinatorial Bandits:} There are $n$ distributions $\nu_{1}, \ldots, \nu_{n}$ supported on $[-1,1]$ with mean $\param_i = \E_{\rad \sim \nu_i} \rad$. $\H $ is a collection of subsets of $[n]$. At each round $t$, the agent queries a distribution $I_t$ and observes $y_t \sim \nu_{I_t}$. Given $\epsilon > 0, \delta \in (0,1)$, the goal is to identify $h \in \H$ that satisfies 
\begin{align*}
\sum_{i \in h} \param_i \geq \sum_{i \in h^*} \param_i - \epsilon
\end{align*} 
with probability at least $1-\delta$ using as few samples as possible. 

We also write $\mu := (\mu_1,\ldots, \mu_n)^\t$. We interchangely treat each $h \in \H$ as a set  in $[n]$ or as a vector in $\{0,1\}^n$ with $h_i = 1 $ if $i \in h$ and $h_i = 0$ otherwise. Using this vector notation, we often write $h^\t \mu = \sum_{i \in h} \mu_i$.

\textbf{Reduction to combinatorial bandits:} We use the reduction of active classification to combinatorial bandits from \cite{jain2019new}. Note that
\begin{align*}
\err(h) & = \frac{1}{n}[ \sum_{i \in [n] : h(x_i) = 0 } {\labn}_i + \sum_{i \in [n] : h(x_i) = 1 } (1-{\labn}_i)] = \frac{1}{n} [\sum_{i \in [n]} {\labn}_i - \sum_{i \in [n] : h(x_i) = 1 } \param_i]
\end{align*}
where $\param_i := 2{\labn}_i-1$. Thus, treating each $h \in \H$ as a set where $i \in h$ iff $h(x_i) = 1$, we observe $\argmin_{h \in \H} \err(h) = \argmax_{h \in \H} \sum_{i \in h} \param_i$ and that finding a hypothesis $h$ such that
\begin{align*}
\err(h) - \min_{h^\prime \in \H} \err(h^\prime) \leq \epsilon  
\end{align*}
is equivalent to finding $h$ such that $\sum_{i \in h} \param_i \geq \max_{h^\prime \in \H}  \sum_{i \in h^\prime} \param_i - n\epsilon$. Thus, active classification can be viewed as an instance of combinatorial bandits where each $\nu_i$ is a random variable supported on $\{-1,1\}$ with mean $\param_i = 2{\labn}_i-1$. Therefore, \emph{any algorithm for $\epsilon$-good arm identification for combinatorial bandits yields an algorithm for active binary classification in the pool-based setting}. 

\section{Disagreement Coefficient}

We now introduce equivalent definitions of $\rho^*(\epsilon), \gamma^*(\epsilon)$ and the disagreement coefficient in the combinatorial bandit setting. 
\begin{align*}
\rho^*(\epsilon) & := \inf_{\lambda \in \simp_n}  \sup_{h \in \H \setminus \{\hstar \}} \frac{\norm{\hstar-h}^2_{A(\lambda)^{-1}}}{ \max(\param^\t (\hstar-h),\epsilon)^2} \\
\gamma^*(\epsilon) & := \inf_{\lambda \in \simp_n} \E_{\rad \sim N(0,I)}[ \sup_{h \in \H \setminus \{\hstar \}} \frac{(\hstar-h)^\t A(\lambda)^{-1/2} \rad}{ \max(\param^\t(\hstar-h ),\epsilon)}]^2.
\end{align*}
Define the ball of radius $r$ centered at $\hstar$
\begin{align*}
B(h_*,r) & = \{h \in \H : \frac{|h \Delta h_*|}{n} \leq r \}
\end{align*}
and
\begin{align*}
\disr(B(h_*,r)) & = \{i : \exists h,h^\prime \in B(h_*,r) \text{ s.t. } i \in h \Delta h^\prime \}.
\end{align*}
The disagreement coefficient is defined as
\begin{align*}
\dis(\epsilon) & = \sup_{r \geq \epsilon  }  \frac{|\disr(B(h_*,r))|/n}{r} \\
& = \sup_{r \geq \epsilon  }   \frac{|\{i : i \in h \Delta h^\prime \text{ for some } h,h^\prime \in \H \text{ s.t. } \max(|h_* \Delta h |, |h_* \Delta h^\prime |) \leq nr\}|}{nr}.
\end{align*}

The proof of Proposition~\ref{prop:disag_coeff_rhostar} follows by a peeling argument and the application of the following lemma.

\begin{lemma}\label{lem:disagreement_peeling_lemma}
Let $\epsilon \in [\frac{\Delta_{min}}{n},1]$. 
\begin{itemize}
\item Suppose the noiseless case holds, i.e., $\labn \in \{0,1\}^n$. If $\epsilon \in [\frac{\Delta_{min}}{n}, \nu)$, then
\begin{align*}
\sup_{\xi \geq \epsilon} \min_{\lambda} \max_{h : \Delta_h \leq n\xi} \frac{\norm{h_* - h}_{A(\lambda)^{-1}}^2}{(n\xi)^2} \leq 9\dis(\epsilon) \frac{\nu^2}{\epsilon^2} 
\end{align*} 
and if $\epsilon \in [\nu, 1]$, then
\begin{align*}
\sup_{\xi \geq \epsilon}  \min_{\lambda} \max_{h : \Delta_h \leq n\xi} \frac{\norm{h_* - h}_{A(\lambda)^{-1}}^2}{(n\xi)^2} \leq 9\dis(\epsilon) .
\end{align*} 
\item Suppose that the Tsabokov noise condition holds for some $a \in [1, \infty)$ and $\alpha \in (0,1]$. Then,
\begin{align*}
\min_{\lambda} \max_{h : \Delta_h \leq n\epsilon} \frac{\norm{h_* - h}_{A(\lambda)^{-1}}^2}{(n\epsilon)^2} \leq c a^2 \frac{1}{\epsilon^{2-2 \alpha}} \dis(a \epsilon^\alpha) \log( n\Delta_{min}^{-1} \vee  \epsilon^{-1})
\end{align*}
\end{itemize}

\end{lemma}

\begin{proof}
\textbf{Noiseless Case: $\labn \in \{0,1\}^n$}:
Under the noiseless assumption, we have that
\begin{align}
\err(h) = \frac{1}{n} |\labn \Delta h|  = \frac{1}{n}[\sum_{i \in \labn \setminus h} 1 + \sum_{i \in h \setminus \labn} 1 = \frac{1}{n} [\sum_{i \in \labn}1 + \sum_{i \in h \setminus \labn} 1 - \frac{1}{n} \sum_{i \in \labn \setminus h} 1] = \frac{1}{n} |\labn| - \frac{1}{n } \sum_{i \in h} \param_i . \label{eq:reduction_rel}
\end{align}
Recall $\errstar = \err(\hstar)$. 

Fix $\xi \geq \epsilon$. Suppose $\Delta_h \leq n \xi$. We have by \eqref{eq:reduction_rel} $\frac{1}{n} \sum_{i \in h} \param_i = \frac{1}{n} |\labn| - \frac{1}{n} |\labn \Delta h|$ and thus
\begin{align*}
n \xi \geq \Delta_h = |\labn \Delta h| - |\labn \Delta h_*| = |\labn \Delta h| -n \nu
\end{align*}
and thus $|\labn \Delta h| \leq n(\xi + \nu)$. Thus, if $\Delta_h \leq n \xi$, we have that
\begin{align}
|h_* \Delta h|  \leq |h_* \Delta \labn| + |h \Delta \labn| \leq n(2 \nu + \xi) \label{eq:dis_sym_diff_ub}
\end{align}

Furthermore, by \eqref{eq:dis_sym_diff_ub}
\begin{align}
\min_{\lambda} \max_{h : \Delta_h \leq n\xi} \frac{\norm{h_* - h}_{A(\lambda)^{-1}}^2}{(n\xi)^2} & = \min_{\lambda} \max_{h : \Delta_h \leq n\xi} \frac{\sum_{i \in h_* \Delta h} \frac{1}{\lambda_i}}{(n\xi)^2} \nonumber \\
& \leq  \frac{|\{i : i \in h_* \Delta h \text{ for some } h \in \H \text{ s.t. }  \Delta_h \leq n\xi\}| \cdot \max_{h : \Delta_h \leq n \xi} |h_* \Delta h|}{(n\xi)^2}   \nonumber \\
&  \leq  \frac{|\{i : i \in h_* \Delta h \text{ for some } h \in \H \text{ s.t. }  |h_* \Delta h|  \leq n(2 \nu + \xi)\}| \cdot n(2 \nu + \xi)}{(n\xi)^2} . \label{eq:disag_basic_inequality}
\end{align}
where the first inequality takes $\lambda$ to be the uniform distribution over $|\{i : i \in h_* \Delta h \text{ for some } h \in \H \text{ s.t. }  \Delta_h \leq n\xi\}|$.

Suppose $\epsilon \geq \nu$. Then, $\xi \geq \epsilon \geq \nu$, and we have
\begin{align*}
& \frac{|\{i : i \in h_* \Delta h \text{ for some } h \in \H \text{ s.t. }  |h_* \Delta h|  \leq n(2 \nu + \xi)\}| n(2 \nu + \xi)}{(n\xi)^2} \\
& \leq 3\frac{|\{i : i \in h_* \Delta h \text{ for some } h \in \H \text{ s.t. }  |h_* \Delta h|  \leq 3n \xi)\}| }{n\xi} \\
& = 9 \frac{|\{i : i \in h_* \Delta h \text{ for some } h \in \H \text{ s.t. }  |h_* \Delta h|  \leq 3n \xi)\}| }{3n\xi}
\end{align*} 
and, using \eqref{eq:disag_basic_inequality} in addition, 
\begin{align*}
\sup_{\xi \geq \epsilon} \min_{\lambda} \max_{h : \Delta_h \leq n\xi} \frac{\norm{h_* - h}_{A(\lambda)^{-1}}^2}{(n\xi)^2} & \leq \sup_{\xi \geq \epsilon} 9 \frac{|\{i : i \in h_* \Delta h \text{ for some } h \in \H \text{ s.t. }  |h_* \Delta h|  \leq 3n \xi)\}| }{3n\xi} \\
& \leq 9 \dis(3 \epsilon) \\
& \leq 9\dis(\epsilon).
\end{align*}

Now, suppose $ \epsilon  \in [\frac{\Delta_{min}}{n}, \nu]$. Suppose $\xi \geq \epsilon$ satisfies $\xi  \in [\frac{\Delta_{min}}{n}, \nu]$. Then,
\begin{align*}
& \frac{|\{i : i \in h_* \Delta h \text{ for some } h \in \H \text{ s.t. }  |h_* \Delta h|  \leq n(2 \nu + \xi)\}| n(2 \nu + \xi)}{(n\xi)^2} \\
& = \frac{|\{i : i \in h_* \Delta h \text{ for some } h \in \H \text{ s.t. }  |h_* \Delta h|  \leq n(2 \nu + \xi)\}| n(2 \nu + \xi)}{(n\nu)^2} \frac{\nu^2}{\xi^2} \\
& \leq 3\frac{|\{i : i \in h_* \Delta h \text{ for some } h \in \H \text{ s.t. }  |h_* \Delta h|  \leq 3n \nu)\}| }{n\nu} \frac{\nu^2}{\xi^2} \\
& \leq 9 \frac{|\{i : i \in h_* \Delta h \text{ for some } h \in \H \text{ s.t. }  |h_* \Delta h|  \leq 3n \nu)\}| }{3n\nu} \frac{\nu^2}{\xi^2}  \\
& \leq 9 \dis(3 \nu) \frac{\nu^2}{\xi^2} \\
& \leq 9 \dis(\xi) \frac{\nu^2}{\xi^2}.
\end{align*} 
Combining this with \eqref{eq:disag_basic_inequality}, this implies that
\begin{align*}
\sup_{\xi \geq \epsilon} \min_{\lambda} \max_{h : \Delta_h \leq n\xi} \frac{\norm{h_* - h}_{A(\lambda)^{-1}}^2}{(n\xi)^2} & \leq \sup_{\xi \geq \epsilon} \min_{\lambda} \max_{h : \Delta_h \leq n\xi} \frac{\norm{h_* - h}_{A(\lambda)^{-1}}^2}{(n\xi)^2} \\
& \leq 9\dis(\epsilon) \frac{\nu^2}{\epsilon^2}.
\end{align*} 

\textbf{Case 2: Tsabakov Noise}: Suppose Tsabokov's noise condition is satisfied with $a \in [1,\infty)$ and $\alpha \in (0,1]$. Fix $\xi \geq \epsilon$. If $\Delta_h \leq n \xi$, then Tsabokov's noise condition implies that
\begin{align*}
\frac{|\hstar \Delta h| }{n} \leq a (\frac{\Delta_h}{n})^\alpha \leq a \xi^\alpha
\end{align*}

 Then,
\begin{align*}
\min_{\lambda} \max_{h : \Delta_h \leq n\xi} \frac{\norm{\hstar - h}_{A(\lambda)^{-1}}^2}{(n\xi)^2} & = \min_{\lambda} \max_{h : \Delta_h \leq n\xi} \frac{\sum_{i \in \hstar \Delta h} \frac{1}{\lambda_i}}{(n\xi)^2} \\
& \leq  \frac{|\{i : i \in \hstar \Delta h \text{ for some } h \in \H \text{ s.t. }  \Delta_h \leq n\xi\}| \cdot \max_{h : \Delta_h \leq n \xi} |\hstar \Delta h|}{(n\xi)^2}   \\
& \leq  \frac{|\{i : i \in \hstar \Delta h \text{ for some } h \in \H \text{ s.t. }  \frac{|h_* \Delta h|}{n} \leq a \xi^\alpha \}| \cdot \max_{h : \Delta_h \leq n \xi} |\hstar \Delta h|}{(n\xi)^2}   \\
& \leq  \frac{|\{i : i \in \hstar \Delta h \text{ for some } h \in \H \text{ s.t. }  \frac{|h_* \Delta h|}{n} \leq a \xi^\alpha \}| a n \xi^\alpha}{(n\xi)^2}   \\
& =a^2  \frac{|\{i : i \in \hstar \Delta h \text{ for some } h \in \H \text{ s.t. }  \frac{|h_* \Delta h|}{n} \leq a \xi^\alpha \}| }{a \xi^\alpha n} \frac{1}{\xi^{2 -2 \alpha}}   \\
& \leq a^2 \frac{1}{\xi^{2-2 \alpha}} \dis(a \xi^\alpha)
\end{align*}
Taking the sup over $\xi \geq \epsilon$ of both sides implies the result.

\end{proof}

\begin{proof}[Proof of Proposition \ref{prop:disag_coeff_rhostar}]
The argument follows by a peeling argument. Define
\begin{align*}
    \lambda^{(k)} := \min_{\lambda} \max_{h : \Delta_h \in (n 2^{-k-1},n 2^{-k}]} \frac{\norm{\hstar - h}_{A(\lambda)^{-1}}^2}{\Delta_h^2} \\
    \bar{\lambda} := \frac{1}{\lceil \log_2(n \Delta_{\min}^{-1} \vee \epsilon^{-1})\rceil} \sum_{k=0}^{\lceil \log_2(n \Delta_{\min}^{-1} \vee \epsilon^{-1})\rceil} \lambda^{(k)}.
\end{align*}
Notice that $\frac{1}{\lceil \log_2(n \Delta_{\min}^{-1} \vee \epsilon^{-1})\rceil} A(\lambda^{(k)}) \preceq A(\bar{\lambda})$, which implies that 
\begin{align}
    A(\bar{\lambda})^{-1} \preceq \lceil \log_2(n \Delta_{\min}^{-1} \vee \epsilon^{-1})\rceil A(\lambda^{(k)})^{-1}. \label{eq:disag_compare_alloc}
\end{align}
Thus,
\begin{align}
\rho^*(n \epsilon) & = \min_{\lambda} \max_{h \in \H \setminus \{\hstar \}} \frac{\norm{\hstar - h}_{A(\lambda)^{-1}}^2}{\max(n \epsilon, \Delta_h)^2} \nonumber \\ 
& \leq \max_{h \in \H \setminus \{\hstar \}} \frac{\norm{\hstar - h}_{A(\bar{\lambda})^{-1}}^2}{\max(n \epsilon, \Delta_h)^2} \nonumber \\
& = \max_{k=0,1,\dots,\lceil \log_2(n \Delta_{min}^{-1} \vee \epsilon^{-1})\rceil} \max_{h : \Delta_h  \in (n 2^{-k-1},n 2^{-k}]} \frac{\norm{\hstar - h}_{A(\bar{\lambda})^{-1}}^2}{\max(n \epsilon, \Delta_h)^2} \nonumber \\ 
& \leq \lceil \log_2(n \Delta_{\min}^{-1} \vee \epsilon^{-1})\rceil \max_{k=0,1,\dots,\lceil \log_2(n \Delta_{min}^{-1} \vee \epsilon^{-1})\rceil} \max_{h : \Delta_h \in (n 2^{-k-1},n 2^{-k}]} \frac{\norm{\hstar - h}_{A(\lambda^{(k)})^{-1}}^2}{\max(n \epsilon, \Delta_h)^2} \label{eq:pos_def} \\ 
&=  \lceil \log_2(n \Delta_{\min}^{-1} \vee \epsilon^{-1})\rceil \max_{k=0,1,\dots,\lceil \log_2(n \Delta_{min}^{-1} \vee \epsilon^{-1})\rceil} \min_{\lambda} \max_{h : \Delta_h \in (n 2^{-k-1},n 2^{-k}]} \frac{\norm{\hstar - h}_{A(\lambda^{(k)})^{-1}}^2}{\max(n \epsilon, \Delta_h)^2} \label{eq:def_lambda_k} \\ 
& \leq   2 \lceil \log_2(n \Delta_{\min}^{-1} \vee \epsilon^{-1})\rceil \sup_{\xi \geq \epsilon \vee \frac{\Delta_{min}}{n}} \min_{\lambda} \max_{h : \Delta_h \leq n \xi} \frac{\norm{\hstar - h}_{A(\lambda)^{-1}}^2}{(n \xi)^2} \nonumber
\end{align}
where inequality \eqref{eq:pos_def} follows by \eqref{eq:disag_compare_alloc} and \eqref{eq:def_lambda_k} follows by the definition of $\bar{\lambda}^{(k)}$.

The result now follows by applying Lemma \ref{lem:disagreement_peeling_lemma} in each case (noiseless $\labn \in \{0,1\}$ with $\epsilon > \nu$ and $\epsilon < \nu$, and Tsabokov noise condition).
\end{proof}

\begin{proof}[Proof of Proposition \ref{prop:disag_coeff_loose}]
\textbf{Step 1: Define the instance.}  Let $m \in \N$. Define $h_i = [m] \cup \{m +i \}$ for $i = 1, \ldots, m^2$ and let $n = m + m^2$. Define $h_0 = \emptyset$. Let $\param_i = -1$ for all $i \in [n]$. Note that $h_0$ is the best set and that $\param^\t (h_0-h_i) = m +1 $ for all $i \neq 0$. 

\textbf{Step 2: Compute problem-dependent quantities.} We have that 
\begin{align*}
\rho^* = \inf_\lambda \max_{i=1, \ldots, m^{2}} \frac{\sum_{j \in [m] \cup \{m+i \} \frac{1}{\lambda_j}}}{(m+1)^2} \leq \frac{2 m^2 + 2m^{2}}{(m+1)^2} \leq O(1)
\end{align*}
where we used
\begin{align*}
\lambda_i & = \begin{cases} 
       \frac{1}{2m} & i \in [m]  \\
      \frac{1}{2m^{2}} & i \in \{m +1, \ldots, m +m^{2} \}
         \end{cases}.
\end{align*}

Let $\xi$ such that $n \xi \leq m +1$. Letting $r = m+1$, we have that
\begin{align*}
\dis(\xi) & \geq \frac{|\{i : i \in h_* \Delta h \text{ for some } h \in \H \text{ s.t. } |h_* \Delta h | \leq r\}|}{r} \\
& = \frac{m +m^{2}}{m+1} \\
& \geq \frac{1}{2}	m \\
& \geq \frac{1}{2\sqrt{2}}\sqrt{n}.
\end{align*}

\end{proof}

\section{Ridge IPS Estimator Tail Bound}

Now, we introduce the ridge IPS estimator, which we leverage in constructing our generic chaining estimator. 
\begin{proposition}\label{prop:ips_bound}
Fix $\mc{X} = \{\mb{e}_i : i \in [n]\}$ and $\H \subset \{0,1\}^n$ as well as some $\param \in [-1,1]^n$. Let $v \in \R^n$.  Fix some $\lambda \in \triangle_n$ and draw $\{ x_s \}_{i=1}^t \sim \lambda$ and then observe $y_s$ with mean $ x_s^\t \param$ and $|y_i| \leq 1$ with probability $1$ for $s=1,\ldots, t$.

 For any $\alpha \in \R_+^n$ define
\begin{align*}
    A(\alpha) := \sum_{i=1}^n \alpha_i \mb{e}_i \mb{e}_i^\top
\end{align*}
For some $s > 0$ define
\begin{align*}
    \widehat{\param} &= ( A(t \lambda) + s I)^{-1}  X^\top y .
\end{align*}
and for any $s >0$ let $A(\alpha + s) := A(\alpha) + s I_n$. For some $s > 0$ define
\begin{align*}
    \widehat{\param} &= ( A(t \lambda) + s I)^{-1}  X^\top y \\
\end{align*}
where $X$, $y$, $\epsilon$ are the $\{(x_i,y_i,\epsilon_i)\}_i$ stacked.
If we take $s = \sqrt{ \frac{ \log(2/\delta) }{3 \| v \|_{( n A(\lambda))^{-1}}^2} }$ then 
\begin{align*}
    |\langle v, \widehat{\param} - \param \rangle| &\leq (\sqrt{2/3} + 1) \sqrt{\frac{2 \| v \|_{A(\lambda)^{-1}}^2 \log(2/\delta)}{t} }
\end{align*}
\end{proposition}

\begin{proof}[Proof of Proposition \ref{prop:ips_bound}]
Fix $\mc{X} = \{\mb{e}_i : i \in [n]\}$ and $\H \subset \{0,1\}^n$ as well as some $\param \in [-1,1]^n$.
Note that for any $v \in \{-1,0,1\}^n$ we have
\begin{align*}
    \left| \E[\langle v, \widehat{\param} - \param \rangle ] \right| &= \left| \langle v, ( A(t \lambda) + s I)^{-1} A(t \lambda) \param - \param \rangle \right|\\
    &= \left| \langle v, ( A(t \lambda) + s I)^{-1} (A(t \lambda) + s I - s I) \param - \param \rangle \right|\\
    &=  \left| s \langle v, ( A(t \lambda) + s I)^{-1} \param \rangle \right| \\
    &\leq s \| v \|_{( t A(\lambda) + s I)^{-1}}^2
\end{align*}
using the fact that $\param \in [-1,1]^n$.
Define
\begin{align*}
    S &= \langle v, \widehat{\param} \rangle \\
    &= \sum_{i=1}^n v^\top ( A(t \lambda) + s I)^{-1}  x_i y_i \\
    &=: \sum_{i=1}^n X_i
\end{align*}
Note that 
\begin{align*}
    \sum_{i=1}^n \E[X_i^2] 
    &= \sum_{i=1}^n \E[ \langle v, ( A(t \lambda) + s I)^{-1} x_i y_i \rangle^2 ] \\
    &\leq v^\top  ( A(t \lambda) + s I)^{-1} A(t \lambda) ( A(t \lambda) + s I)^{-1} v\\
    &\leq \| v \|_{( n A(\lambda) + s I)^{-1}}^2
\end{align*}
and 
\begin{align*}
    |X_i| \leq |\langle v, ( A(t \lambda) + s I)^{-1} x_i y_i \rangle| \leq 1/s
\end{align*}
for all $i$, we have by Bernstein's inequality that
\begin{align*}
    \sum_{i=1}^n X_i - \E[X_i] \leq \sqrt{ 2 \| v \|_{( n A(\lambda) + s I)^{-1}}^2 \log(1/\delta)} + \frac{\log(1/\delta)}{3s}.
\end{align*}
Thus,
\begin{align*}
    \langle v, \widehat{\param}\rangle &\leq\E[ \langle v, \widehat{\param}  \rangle ] + \sqrt{ 2 \| v \|_{( n A(\lambda) + s I)^{-1}}^2 \log(1/\delta)} + \frac{\log(1/\delta)}{3s} \\
    &= \langle v, \param \rangle + \E[ \langle v, \widehat{\param} - \param \rangle ] + \sqrt{ 2 \| v \|_{( n A(\lambda) + s I)^{-1}}^2 \log(1/\delta)} + \frac{\log(1/\delta)}{3s} \\
    &\leq \langle v, \param \rangle + s \| v \|_{( n A(\lambda) + s I)^{-1}}^2 + \sqrt{ 2 \| v \|_{( n A(\lambda) + s I)^{-1}}^2 \log(1/\delta)} + \frac{\log(1/\delta)}{3s}
\end{align*}
from which we conclude, that
\begin{align*}
    |\langle v, \widehat{\param} - \param \rangle| &\leq s \| v \|_{( n A(\lambda) + s I)^{-1}}^2 + \frac{\log(2/\delta)}{3 s} + \sqrt{2 \| v \|_{( n A(\lambda) + s I)^{-1}}^2 \log(2/\delta) } \\
    &\leq s \| v \|_{( n A(\lambda))^{-1}}^2 + \frac{\log(2/\delta)}{3 s} + \sqrt{2 \| v \|_{( n A(\lambda))^{-1}}^2 \log(2/\delta) }
\end{align*}

If we take $s = \sqrt{ \frac{ \log(2/\delta) }{3 \| v \|_{( n A(\lambda))^{-1}}^2} }$ then 
\begin{align*}
    |\langle v, \widehat{\param} - \param \rangle| &\leq (\sqrt{2/3} + 1) \sqrt{\frac{2 \| v \|_{A(\lambda)^{-1}}^2 \log(2/\delta)}{n} }
\end{align*}

\end{proof}

\section{Looseness of Bernstein's Bound for Importance Sampling Estimator}

A natural approach to combinatorial bandits is to use the importance sampling estimator and apply Bernstein's bound in an algorithm like RAGE from \cite{fiez2019sequential}. Applying the standard analysis would yield a term in the sample complexity that scales as:
\begin{align*}
\inf_\lambda \max_{h\neq \hstar} \max_{j \in h_* \Delta h } \frac{\frac{1}{\lambda_j}}{\sum_{k \in h_*} \mu_k - \sum_{k \in h} \mu_k }.
\end{align*}
The following proposition shows that there exists instances where such a sample complexity is suboptimal by a polynomial factor in the dimension.

\begin{proposition}\label{prop:bernstein_loose}
There exists a combinatorial bandit problem where $\rho^* = O(1)$ and 
\begin{align*}
\inf_\lambda \max_{h\neq \hstar} \max_{j \in \hstar \Delta h } \frac{\frac{1}{\lambda_j}}{\sum_{k \in h_*} \mu_k - \sum_{k \in h} \mu_k } \geq \Omega(\sqrt{n})
\end{align*}
\end{proposition}

\begin{proof}
\textbf{Step 1: Define the instance.}  Let $m \in \N$. Define $h_i = [m] \cup \{m +i \}$ for $i = 1, \ldots, m^2$ and let $n = m + m^2$. Define $h_0 = \emptyset$. Let $\param_i = -1$ for all $i \in [n]$. Note that $\param^\t (h_0-h_i) = m +1 $ for all $i \neq 0$. 

\textbf{Step 2: Compute problem-dependent quantities.} Then,
\begin{align*}
\rho^* = \inf_\lambda \max_{i=1, \ldots, m^{2}} \frac{\sum_{j \in [m] \cup \{m+i \} \frac{1}{\lambda_j}}}{(m+1)^2} \leq \frac{2 m^2 + 2m^{2}}{(m+1)^2} \leq O(1)
\end{align*}
where we used
\begin{align*}
\lambda_i & = \begin{cases} 
       \frac{1}{2m} & i \in [m]  \\
      \frac{1}{2m^{2}} & i \in \{m +1, \ldots, m +m^{2} \}
         \end{cases}.
\end{align*}
On the other hand, 
\begin{align*}
\inf_\lambda \max_{i=1, \ldots, m^{2}} \max_{j \in [m] \cup \{m+i \} } \frac{\frac{1}{\lambda_j}}{m+1} & \geq \inf_\lambda \max_{i=1, \ldots, m^{2}} \max_{j = m+i } \frac{\frac{1}{\lambda_j}}{m+1} \\
& = \frac{m^{2}}{m+1} \\
& \geq \frac{1}{2} m \\
& \geq \frac{1}{2 \sqrt{2}} \sqrt{n} 
\end{align*}
\end{proof}

\section{Proof of Theorem \ref{thm:exists_est_gamma}}

Before proving Theorem \ref{thm:exists_est_gamma}, we introduce some machinery from the theory of generic chaining (see e.g. \cite{vershynin2019high} for more details). Fix a set $T \subset \R^n$. Consider a sequence of subset $(T_k)_{k =0}^\infty$ such that $T_k \subset T$
\begin{align*}
|T_0| = 1, \quad \quad |T_k| \leq 2^{2^k}.
\end{align*}
A sequence $(T_k)_{k =0}^\infty$ satisfying the above properties is called \emph{admissible}.

\begin{definition}
Let $d$ be a metric on $\R^n$. The $\gamma_2$-functional of $T$ is defined as
\begin{align*}
\gamma_2(T,d) & = \inf_{(T_k)} \sup_{t \in T} \sum_{k=0}^\infty 2^{k/2} d(t,T_k)
\end{align*}
where the infimum is taken over all admissible sequences.
\end{definition}

Here, we state and prove Theorem \ref{thm:exists_est_gamma_combinatorial}, the combinatorial bandit counterpart of Theorem \ref{thm:exists_est_gamma}. The proof uses the technique of generic chaining to avoid a naive union bound, which would introduce a dependence on $\log(|\G|)$. Unfortunately, the IPS estimator has an excessively large sub-Gaussian norm, and the concentration inequality for ridge IPS estimator from Proposition \ref{prop:ips_bound} decays at a suitably fast sub-Gaussian rate for only a subset of pairs $h,h^\prime \in \G$--not all pairs--implying that neither of these estimators can be used directly. To sidestep this issue, we apply the estimator from Proposition \ref{prop:ips_bound} to all pairs of $h,h^\prime \in \G$ to construct a feasibility program that yields the estimator in Theorem \ref{thm:exists_est_gamma_combinatorial}.

\begin{theorem}
\label{thm:exists_est_gamma_combinatorial}
Let $\G \subset \H$. Fix some $\lambda \in \triangle_n$ and draw $\{ x_s \}_{i=1}^t \sim \lambda$ and then observe $y_s$ with mean $ x_s^\t \param$ and $|y_s| \leq 1$ with probability $1$ for $s=1,\ldots, t$. There exists an estimator $\widehat{\param} \in [-1,1]^n$ such that we have that with probability at least $1-\delta$,
\begin{align*}
\sup_{h,h^\prime \in \G} |(h-h^\prime)^\t (\widehat{\param}-\param)| \leq  c[\log(1/\delta) \max_{h, h^\prime \in \G} \norm{h-h^\prime}_{A(t\lambda)^{-1}}  +    \E_{\rad \sim N(0,I)}[\sup_{h \in \G} h^\t A(t\lambda)^{-1/2} \rad]]
\end{align*}
\end{theorem}

\begin{proof}[Proof of Theorem \ref{thm:exists_est_gamma_combinatorial}]
\textbf{Step 1: Pick the admissible sequence.} Fix $h_0 \in \G$. Note that  for any $\tilde{\param} \in \R^n$
\begin{align}
\sup_{h,h^\prime \in \G} |(h-h^\prime)^\t (\tilde{\param}-\param)| \leq 2\sup_{h \in \G} |(h-h_0)^\t (\tilde{\param}-\param)| \label{eq:focus_h_0}
\end{align}
and thus we focus on upper bounding the RHS.

Let $(\mc{R}_k)_{k=0}^K$ be an admissible sequence of $\G$ wrt the metric $d(x, y) := \norm{x - y}_{A(t \lambda)^{-1}}$ where $\mc{R}_0 = \{h_0\}$ and $\mc{R}_K = \G$ and $\mc{R}_k \subset \G$ such that
\begin{align*}
\sup_{h \in \G} \sum_{k=1}^K 2^{k/2} \inf_{h^\prime \in {\mc{R}}_k} \norm{h-h^\prime}_{A(t\lambda)^{-1}} & \leq 2\gamma(\G,\norm{\cdot}_{A(t\lambda)^{-1}})
\end{align*}
Let $\T_k = \cup_{l=1}^k \mc{R}_l$. Note that $\T_1 \subset \T_2 \subset \ldots \subset \T_K$ and $K = \log(\log(|\G|)) \leq \log(n)$.

\textbf{Step 2: Constructing the estimator.} Let $X$, and $y$ be the $\{(x_s,y_s)\}_{s=1}^t$ stacked, respectively. For each $h,h^\prime \in \G$ and $k \in [K]$ such that $h,h^\prime \in \T_k$, define the estimator  
\begin{align*}
\widehat{\param}_{h,h^\prime,k}  = ( A(t \lambda) + s_{h,h^\prime,k} I)^{-1}  X^\top y
\end{align*}
where $s_{h,h^\prime,k} = \sqrt{1/3} \frac{u + 2^{k/2}}{\norm{h-h^\prime}_{A(t\lambda)^{-1}}}$. 
Define the events
\begin{align}
\mc{E}_{h,h^\prime,k} & = \{ |(h-h^\prime)^\t (\widehat{\param}_{h,h^\prime,k} - \param)| \leq (\sqrt{2/3} +1)\sqrt{4} (u+2^{k/2}) \norm{h-h^\prime}_{A(t\lambda)^{-1}} \} \label{eq:estimator_event_0} \\
\mc{E} & = \cap_{k=1}^K \cap_{h,h^\prime \in \T_k} \mc{E}_{h,h^\prime,k} \label{eq:estimator_event_1}
\end{align}
Applying the bound for the Ridge IPS from Proposition \ref{prop:ips_bound} and rearranging, for every $k \in [K]$ and $h,h^\prime \T_k$,
\begin{align}
\Pr( |(h-h^\prime)^\t (\widehat{\param}_{h,h^\prime,k} - \param)| > (\sqrt{2/3} +1)\sqrt{4} (u+2^{k/2}) \norm{h-h^\prime}_{A(t\lambda)^{-1}}) & \leq 2 \exp(-2 (u +2^{k/2})^2). \label{eq:chaining_ridge_bound}
\end{align}
Then, by the union bound, we have that
\begin{align}
\Pr(\mc{E}^c) & \leq \sum_{k \geq 1} |\T_k|^2 \Pr(\mc{E}_{h,h^\prime,k}^c ) \nonumber \\
& \leq c \sum_{k \geq 1} |\T_k|^2 \exp(-2(2^k + u^2)) \label{eq:chaining_ridge_bound_apply} \\
& \leq c\sum_{k \geq 1}   2^{2^k+1} \exp(-2(2^k + u^2)) \label{eq:chaining_cardinality_bound} \\
& \leq c^\prime \exp(-2 u^2). \nonumber
\end{align}
where \eqref{eq:chaining_ridge_bound_apply} follows by \eqref{eq:chaining_ridge_bound} and where line \eqref{eq:chaining_cardinality_bound} follows since by construction
\begin{align*}
|\T_k| \leq \sum_{l=1}^k |\mc{R}_k| \leq \sum_{l=1}^k 2^{2^l} \leq 2^{2^k+1}.
\end{align*}
Suppose $\mc{E}$ holds for the remainder of the proof with $u$ taking the value of $\bar{u} = \sqrt{\tfrac{\log(c^\prime/\delta)}{2}}$. 

Define the polyhedron
\begin{align*}
\widehat{\param} \in P = \{ z \in [-1,1]^n : \forall k \in [K], \forall h,h^\prime \in \T_k  : \, \, |(h-h^\prime)^\t \widehat{\param}_{h,h^\prime,k} - z)| \leq c [\bar{u}+ 2^{k/2}] \norm{h-h^\prime}_{A(t\lambda)^{-1}} \}.
\end{align*}
We define the estimator $\widehat{\param}$ to be any point in $P$ if it is nonempty and, otherwise we let $\widehat{\param}$ be any point in $\R^n$. 

Note that on the event $\mc{E}$,  $\param \in P$ and hence $P$ is nonempty. Furthermore, by the triangle inequality,  $\forall k \in [K], \forall h,h^\prime \in \T_k$
\begin{align*}
|(h-h^\prime)^\t (\widehat{\param}-\param)| & \leq |(h-h^\prime)^\t (\widehat{\param}-\widehat{\param}_{h,h^\prime,k}) | + |(h-h^\prime)^\t (\widehat{\param}_{h,h^\prime,k} - \param)| \\
& \leq 2c (\bar{u}+ 2^{k/2}) \norm{h-h^\prime}_{A(t\lambda)^{-1}}.
\end{align*}

\textbf{Step 3: Proving the inequality.} Fix $\bar{h} \in \G$. Let $k(\bar{h},l)$ be the smallest integer such that
\begin{align*}
d(\bar{h}, \T_{k(\bar{h},l)}) \leq \frac{d(\bar{h}, \T_{k(\bar{h},l-1)})}{2}.
\end{align*}
Note that
\begin{align*}
d(\bar{h}, \T_{k(\bar{h},l)}) \leq 2^{-l} \max_{h,h^\prime \in \G} \norm{h-h^\prime}_{A(t\lambda)^{-1}}.
\end{align*}

Let $\pi_l(\bar{h}) \in \T_{k(\bar{h},l)}$ such that 
\begin{align*}
\norm{\bar{h} - \pi_l(\bar{h})}_{A(t\lambda)^{-1}} = \min_{h \in \T_{k(\bar{h},l)}} \norm{\bar{h} - h}_{A(t \lambda)^{-1}}.
\end{align*}
By the triangle inequality
\begin{align}
& |(\bar{h}-h_0)^\t (\widehat{\param} -\param)| \leq \sum_{l \geq 1} |(\pi_l(\bar{h}) - \pi_{l-1}(\bar{h}))^\t (\widehat{\param} -\param)| \nonumber \\
& \leq 2c \sum_{l \geq 1} (\bar{u}+2^{k/2}) \norm{\pi_{l}(\bar{h})-\pi_{l-1}(\bar{h})}_{A(t\lambda)^{-1}}  \nonumber \\
& \leq c^\prime \sum_{l \geq 1} (\bar{u}+2^{k/2}) \norm{\pi_{l}(\bar{h})-\bar{h}}_{A(t\lambda)^{-1}}   \label{eq:chaining_main_1}  
\end{align}
where we use event $\mc{E}$ and the triangle inequality. We have by construction that 
\begin{align}
\sum_{l \geq 1} \bar{u} \norm{\pi_{l}(\bar{h})-\bar{h}}_{A(t\lambda)^{-1}} &  \leq \sum_{l \geq 1} 2^{-l} \max_{h,h^\prime \in \G} \norm{h-h^\prime}_{A(t\lambda)^{-1}} \nonumber \\
& \leq c^\prime  \bar{u}\max_{h,h^\prime \in \G} \norm{h-h^\prime}_{A(t\lambda)^{-1}}.  \label{eq:chaining_main_2}
\end{align}
Furthermore, 
\begin{align}
 \sum_{l \geq 1} 2^{k/2} \norm{\pi_{l}(\bar{h})-\bar{h}}_{A(t\lambda)^{-1}}   & \leq c^{\prime }  \gamma(\G, \norm{\cdot}_{A(t\lambda)^{-1}}) \label{eq:apply_T_k_def} \\
& \leq c^{\prime \prime } \E_{\rad \sim N(0,I)}[\sup_{h \in \G} h^\t A(t\lambda)^{-1/2} \rad]  \label{eq:chaining_main_3}
\end{align}
where \eqref{eq:apply_T_k_def} follows by the definition of $\T_k$ and \eqref{eq:chaining_main_3} follows by Talagrand's majorizing measure theorem (Theorem 8.6.1 in \cite{vershynin2019high}). Putting together \eqref{eq:focus_h_0}, \eqref{eq:chaining_main_1}, \eqref{eq:chaining_main_2}, and \eqref{eq:chaining_main_3}, and noting that $\bar{h}$ is arbitrary, the result follows.

\end{proof}

\begin{remark}
We emphasize that the construction of this estimator does not use any knowledge of $\param$. The admissible sequence $(\mc{R}_k)_{k \in \N}$ does not require knowledge of $\param$ to be chosen and the polyhedron $P$ can be defined without knowledge of $\param$.
\end{remark}

\section{Fixed Confidence Algorithms}

We restate Algorithm \ref{alg:action_chaining} in the language of combinatorial bandits.

\begin{algorithm}[h]
\small

\begin{algorithmic}
\STATE \textbf{Input: } Confidence level $\delta \in (0,1)$.

\STATE $\H_1 \longleftarrow \H$, $k \longleftarrow 1$, $\delta_k \longleftarrow \delta/2k^2$.

\WHILE{$|\H_k| > 1$}

    \STATE Let $\lambda_k$ and $\tau_k$ be the solution and value of the following optimization problem
        \begin{align*}
            &\inf_{\lambda \in \simp_n}  \E_{\rad \sim N(0,I)}[ \max_{h,h^\prime \in \H_k} (h-h^\prime)^\t A(\lambda)^{-1/2} \rad]^2 +2\log(\frac{1}{\delta_k}) \max_{h, h^\prime \in \H_k} \norm{h-h^\prime}^2_{A(\lambda)^{-1}}
        \end{align*}
    
    \STATE Set $N_k \longleftarrow c\tau_k(\frac{2^{k+1}}{n})^2$ where $c$ is a universal constant.
    
    \STATE Query $I_1, \ldots, I_{N_k} \sim \lambda_k$ and receive rewards $y_1, \ldots, y_{N_k}$.
    
    \STATE Let $\widehat{\param}_k $ be the estimator defined in Theorem \ref{thm:exists_est_gamma_combinatorial}.
    
    \STATE $\H_{k+1} \longleftarrow \H_k \setminus \{h \in \H_k : \exists h^\prime \text{ such that } (h^\prime -h)^\t \widehat{\mu}_k - \frac{n}{2^{k+1}} \geq 0 \}$.
    
    \STATE $k \longleftarrow k+1$
\ENDWHILE

\STATE\textbf{Return: } $\H_{k} = \{\widehat{h}\}$.

\end{algorithmic}

\caption{ACED for Combinatorial Bandits.}
\label{alg:action_chaining_combi}
\end{algorithm} 

We now restate Theorem \ref{thm:chaining_ub} as Theorem \ref{thm:chaining_ub_combi} in the combinatorial bandits setting. 




\begin{theorem}\label{thm:chaining_ub_combi}
Let $\delta \in (0,1)$ and $\epsilon > 0$. With probability at least $1-\delta$ Algorithm \ref{alg:action_chaining} returns $h \in \H$ such that  $\param^\t\widehat{h}  +\epsilon < \param^\t \hstar$ and uses at most
\begin{align*}
\log(1/\epsilon) [\log(1/\delta) \rho^*(\epsilon) + \gamma^*(\epsilon)]
\end{align*}
\end{theorem}

The proof of Theorem \ref{thm:chaining_ub_combi} is essentially identical to the proof of Theorem 4 in \cite{katz2020empirical}, and we therefore omit it. The key technical technical hurdle in obtaining Theorem \ref{thm:chaining_ub_combi} is the estimator given in Theorem \ref{thm:exists_est_gamma_combinatorial}.

\section{Fixed Budget Proof}

We restate the fixed budget algorithm, Algorithm \ref{alg:fixed_budget_comp}, and the Theorem \ref{thm:fixed_budget}in the language of combinatorial bandits as Algorithm \ref{alg:fixed_budget_comp_combi} and Theorem \ref{thm:fixed_budget_combi}, respectively. 

\begin{algorithm}[h]\small

\begin{algorithmic}
\STATE \textbf{Input:} Budget $T$, tolerance $\epsilon > 0$.
\STATE $\widehat{\param}_1 = \zero \in \R^n$, $N \longleftarrow \floor{T/\log_2(n\epsilon^{-1})}$.

\FOR{$k=1,2,\ldots, \floor{\log_2(n\epsilon^{-1})}$}

    \STATE $\tilde{h}_{k} \longleftarrow \argmax_{h \in \H} \widehat{\param}_{k}^\t h$

    \STATE Let $\lambda_k $ be the solution of the following optimization problem
        \begin{align}
            \inf_{\lambda \in \simp} \E_{\rad \sim N(0,I)}[ \max_{h \in \H} \frac{(\tilde{h}_k- h)^\t A(\lambda)^{-1/2} \rad}{2^{-k+1} n+ \widehat{\param}^\t_{k}(\tilde{h}_k - h) }]^2 \label{eq:action_comp_2_combi}
        \end{align}
    
    \STATE Sample $\{x_1, \ldots, x_{N} \}  \sim \lambda_k$.
    
    \STATE Query $x_1, \ldots, x_{N} $ and receive rewards $y_1, \ldots, y_{N}$. 
    
    \STATE Let $\widehat{\param}_{k+1}$ be the estimator defined in the proof of Theorem \ref{thm:fixed_budget_combi}.

\ENDFOR

\STATE \textbf{Return:} $\argmax_{h \in \H} \widehat{\param}_{k+1}^\t h$.

\end{algorithmic}
\caption{Fixed Budget ACED for Combinatorial Bandits.}
\label{alg:fixed_budget_comp_combi}
\end{algorithm} 

\begin{theorem}\label{thm:fixed_budget_combi}
Let $T \in \N$ and $\epsilon > 0$. Let $\widehat{h}$ denote the $h \in \H$ returned by Algorithm \ref{alg:fixed_budget_comp}. If $T \geq \log(n \epsilon^{-1}) [\gamma^*(\epsilon) + \rho^*(\epsilon)]$, then 
\begin{align*}
\P(\param^\t\widehat{h}  +\epsilon < \param^\t \hstar) \leq \log(n \epsilon^{-1})  \exp(-\frac{T}{\log(n \epsilon^{-1}) [\gamma^*(\epsilon) + \rho^*(\epsilon)]})
\end{align*}
\end{theorem}

The key technical challenge in the proof (see Step 1) is constructing an estimator that concentrates rapidly enough. To this end, we leverage the estimator in Theorem \ref{thm:exists_est_gamma_combinatorial}, applying it to various subsets of $\H$ based on their estimated gaps and combine them into a single estimator by defining it to belong to a polyhedron (denoted $P$ in step 1.2.2) characterizing estimators with the suitable concentration properties. We argue that on a good event, the true mean $\param$ belongs to $P$, making it feasible and thus obtaining our estimator. The next challenge is bounding the probability of error of our algorithm. The estimator constructed at round $k+1$ is chosen to have failure probability 
\begin{align*}
\delta_{k+1} =  \exp(-c\frac{N}{ \E[\sup_{h \in \H} \frac{(\wt{h}_k-h)^\t A(\lambda)^{-1/2} \rad}{(\wt{h}_k-h)^\t\widehat{\param}_k + 2^{-k+1} n}]^2}),
\end{align*}
for a suitably large universal constant $c>0$, which we note is a function of \eqref{eq:action_comp_2_combi}. Thus, this step of the proof (step 3) shows that $\delta_{k+1} \leq \exp(-\frac{T}{\log(n \epsilon^{-1}) [\gamma^*(\epsilon) + \rho^*(\epsilon)]})$. We note that while this step of the proof and algorithm style are novel to our knowledge, the mechanics of bounding the various quantities appearing in this step are quite similar to arguments in \cite{katz2020empirical}.

\begin{proof}[Proof of Theorem \ref{thm:fixed_budget_combi}]
\textbf{Step 1: Construction of the Estimator.} Define the sets
\begin{align*}
S_k = \{h \in \H : \Delta_h \leq n 2^{-k+1}\}.
\end{align*} 
Let $\widehat{\param}_k$ be some estimator formed from data collected in $k$th round. Define the event 
\begin{align*}
    \mc{E}_k(\widehat{\param}_k) & = \{  \forall h \in S_k^c, \, \, |(h_* -h)^\t(\widehat{\param}_k - \param)| \leq  \frac{\Delta_h}{8} \} \\
    & \cap \{ \forall h \in S_k, \, \,  |(h_* -h)^\t(\widehat{\param}_k - \param)| \leq  \frac{2^{-k+1} n}{8}\}
\end{align*}

We show that at every round $k$, we can construct an estimator $\widehat{\param}_k \in [-1,1]^n$ such that $\Pr(\mc{E}_k^c(\widehat{\param}_k) | \mc{E}_{k-1}(\widehat{\param}_{k-1}), \ldots, \mc{E}_1(\widehat{\param}_1)) \leq \delta_k \log(n \epsilon^{-1})$ where $\delta_k > 0$ will be chosen at the $k$th round.

\textbf{Step 1.1: Base Case}

Let $\widehat{\param}_1$ be the estimator from Theorem \ref{thm:exists_est_gamma} applied with $\delta_1$ to $\H$ (to be chosen later). Then, we have that
\begin{align}
\sup_{h,h^\prime \in \H} |(h-h^\prime)^\t (\widehat{\param}_1-\param)| & \leq  c[\log(2/\delta_1) \max_{h, h^\prime \in \H} \norm{h-h^\prime}_{A(N \lambda)^{-1}}  +    \E_{\rad \sim N(0,I)}[\sup_{h \in \H} h^\t A(N \lambda)^{-1/2} \rad]] \nonumber \\
& \leq c[\frac{\pi}{2}\log(2/\delta_1) \E[\sup_{h \in \H} (h-h^\prime)^\t A(N \lambda)^{-1/2} \rad]  +    \E_{\rad \sim N(0,I)}[\sup_{h \in \H} h^\t A(N \lambda)^{-1/2} \rad]] \label{eq:fixed_budget_diam_upper_bound} \\
& \leq c^\prime \sqrt{\log(2/\delta_1)} \E_{\rad \sim N(0,I)}[\sup_{h \in \H} h^\t A(N \lambda)^{-1/2} \rad]. \nonumber
\end{align}
where in line \eqref{eq:fixed_budget_diam_upper_bound} we used Lemma \ref{lem:rho_gamma}. Now, we have that
\begin{align*}
\frac{\sup_{h,h^\prime \in \H} |(h-h^\prime)^\t (\widehat{\param}-\param)|}{n} &  \leq c \sqrt{\log(2/\delta_1)} \E[\frac{\sup_{h \in \H} h^\t A(N \lambda)^{-1/2} \rad]}{n}] \\
& \leq \frac{1}{8}
\end{align*}
where we chose 
\begin{align*}
\delta_1 =  2\exp(-c^{\prime \prime} \frac{N}{ \E[\frac{\sup_{h \in \H} h^\t A( \lambda)^{-1/2} \rad]}{n}]^2})
\end{align*}
for a universal constant $c^{\prime \prime} > 0$ large enough. 
This proves the base case for both $h \in S_1^c$ and $h \in S_1$.

\textbf{Step 1.2: Inductive Step.} Next, we show the inductive step. Suppose that at round $k$, the hypothesis is satisfied, i.e., the algorithm has constructed estimators $\widehat{\param}_1,\ldots, \widehat{\param}_k$ such that $\mc{E}_{k}(\widehat{\param}_{k}) \cap \ldots \cap \mc{E}_1(\widehat{\param}_1)$ holds.  Now, we construct an estimator $\widehat{\param}_{k+1}$ for round $k+1$.  Define for every $l \in [k] \cup \{0\}$, the set
\begin{align*}
\widehat{S}_l = \{h : (\wt{h}_k - h)^\t \widehat{\param}_k \leq 2^{-l+1}n \} .
\end{align*} 

We will construct an estimator each subset $\widehat{S}_l$ and then combine these into a single estimator. $\widehat{S}_l$ can be thought of as an estimate for $S_l$ as suggested by the following claim.

\begin{claim}\label{cl:subset_rel}
$S_{l+1} \subset \widehat{S}_l \subset S_{l-1}$ for all $l \in [k]$.
\end{claim}

\begin{proof}[Proof of Claim \ref{cl:subset_rel}]
Since $\mc{E}_k$ holds, by Lemma \ref{lem:comp_effic_stat_claim}, we have that 
 \begin{enumerate}
\item for all $h \in S_k^c$,
\begin{align}
|(\tilde{h}_k-h)^\t\widehat{\param}_k -(h_*-h)^\t \param | & \leq  \frac{1}{2} \Delta_h.  \label{eq:gap_est_1}
\end{align}
\item for all $h \in S_k$,
\begin{align}
|(\tilde{h}_k-h)^\t\widehat{\param}_k -(h_*-h)^\t \param | & \leq \frac{1}{2} 2^{-k+1} n. \label{eq:gap_est_2}
\end{align}
\end{enumerate}

Suppose $h \in \widehat{S}_l$, that is, $(\tilde{h}_k -h)^\t\widehat{\param}_k \leq 2^{-l+1}n$. We show that $h \in S_{l-1}$. If $h \in S_k$, then we automatically have that $h \in S_{l-1}$ since $l \in [k]$ and $S_k \subset S_{l-1}$. Thus, suppose $h \in S_k^c$. Then, \eqref{eq:gap_est_1} implies that $\Delta_h \leq (\tilde{h}_k-h)^\t\widehat{\param}_k + \tfrac{\Delta_h}{2} \leq 2^{-l+1}n + \tfrac{\Delta_h}{2}$. Rearranging, we have that $\Delta_h \leq 2^{-l+2} n$, implying that $h \in S_{l-1}$. We conclude that $\widehat{S}_l \subset S_{l-1}$. 


Now, suppose that $h \in S_{l+1}$. If $h \in S_k^c$, then we have that $(\tilde{h}_k -h)^\t\widehat{\param}_k \leq 3/2 \Delta_h \leq 2^{-l+1} n$ and hence $h \in \widehat{S}_l$. Suppose $h \in S_k$. \eqref{eq:gap_est_2} implies that
\begin{align*}
(\tilde{h}_k -h)^\t \widehat{\param}_k \leq \Delta_h + 2^{-k}n \leq 2^{-k+1}n + 2^{-k}n \leq 2^{-k+2}n
\end{align*}
where the second inequality follows since $h \in S_k$. Thus, $h \in \widehat{S}_{k-1}$. Showing that $S_{k+1} \subset \widehat{S}_k$ follows by a similar argument and we conclude that $S_{l+1} \subset \widehat{S}_l$. This shows  the claim.

\end{proof}

\textbf{Step 1.2.1: Constructing the estimator for $\widehat{S}_l$.}  Let $l \in [k] \cup \{0\}$. We use Theorem \ref{thm:exists_est_gamma} to construct an estimator $\widehat{\param}_{k+1,l}(\delta_{k+1})$ for each $\widehat{S}_l $ such that for all $l \in [k] \cup \{0\}$, the event
\begin{align*}
\mc{E}_{k+1,l} = \{\sup_{h,h^\prime \in \widehat{S}_l } |(h-h^\prime)^\t (\widehat{\param}_{k+1,l}(\delta_{k+1})-\param)| & \leq  c[\log(2/\delta_{k+1}) \max_{h, h^\prime \in \widehat{S}_l } \norm{h-h^\prime}_{A(N \lambda)^{-1}}  \\
& +    \E_{\rad \sim N(0,I)}[\sup_{h \in \widehat{S}_l } h^\t A(N\lambda)^{-1/2} \rad]] \}
\end{align*}
holds with probability at least $1-\delta_{k+1}$ (that will be chosen later). We assume that $\cap_{l \in [k] \cup \{0\}} \mc{E}_{k+1,l} $ holds for the remainder of the proof, which by the union bound holds with probability at least $1-\delta \log(n \epsilon^{-1})$.

Fix $l \in [k]$. We have that
\begin{align}
\sup_{h,h^\prime \in S_{l+1} } |(h-h^\prime)^\t (\widehat{\param}_{k+1,l}(\delta_{k+1})-\param)| & \leq \sup_{h,h^\prime \in \widehat{S}_l } |(h-h^\prime)^\t (\widehat{\param}_{k+1,l}(\delta_{k+1})-\param)| \label{eq:budget_induct_1} \\
&  \leq  c[\log(2/\delta_{k+1}) \max_{h, h^\prime \in \widehat{S}_l } \norm{h-h^\prime}_{A(N \lambda)^{-1}}  +    \E_{\rad \sim N(0,I)}[\sup_{h \in \widehat{S}_l } h^\t A(N\lambda)^{-1/2} \rad] \nonumber \\
& \leq c \sqrt{\log(1/\delta_{k+1}) \E[\sup_{h \in \widehat{S}_l } h^\t A(N\lambda)^{-1/2} \rad]^2}  \label{eq:budget_induct_2} \\
& \leq c \sqrt{\log(1/\delta_{k+1}) \E[\sup_{h \in S_{l-1} } h^\t A(N\lambda)^{-1/2} \rad]^2}  \label{eq:budget_induct_3} 
\end{align}
where inequality \eqref{eq:budget_induct_1} follows by $S_{l+1} \subset \widehat{S}_l$ from Claim \ref{cl:subset_rel}, inequality \eqref{eq:budget_induct_2} follows by Lemma \ref{lem:rho_gamma}, and the inequality \eqref{eq:budget_induct_3} follows from $ \widehat{S}_l \subset S_{l-1}$ from Claim \ref{cl:subset_rel}. 

Now, we have that
\begin{align}
\frac{\sup_{h,h^\prime \in S_{l+1} } |(h-h^\prime)^\t (\widehat{\param}_{k+1,l}(\delta_{k+1})-\param)| }{2^{-l }n} & \leq c \sqrt{\log(1/\delta_{k+1}) \E[\sup_{h \in S_{l-1} } \frac{h^\t A(N\lambda)^{-1/2} \rad}{2^{-l} n}]^2} \nonumber \\
& \leq c^\prime \sqrt{\log(1/\delta_{k+1}) \frac{\E[\sup_{h \in S_{l-1}} \frac{(\tilde{h}_k-h)^\t A(\lambda)^{-1/2} \rad}{2^{-l+2} n}]^2}{N}} \nonumber \\
& \leq  c^{\prime \prime} \sqrt{\ \log(1/\delta_{k+1})\frac{\E[\sup_{h \in S_{l-1}} \frac{(\tilde{h}_k-h)^\t A(\lambda)^{-1/2} \rad}{\Delta_{h} + 2^{-k+1} n}]^2}{N}} \label{eq:comp_ub_ind_step_1} \\
& \leq c^{\prime \prime } \sqrt{  \log(1/\delta_{k+1})\frac{\E[\sup_{h \in \H} \frac{(\tilde{h}_k-h)^\t A(\lambda)^{-1/2} \rad}{\Delta_{h} + 2^{-k+1} n}]^2}{N}}  \\
& \leq c^{\prime \prime \prime} \sqrt{   \log(1/\delta_{k+1})\frac{\E[\sup_{h \in \H} \frac{(\tilde{h}_k-h)^\t A(\lambda)^{-1/2} \rad}{(\tilde{h}_k-h)^\t\widehat{\param}_k + 2^{-k+1} n}]^2}{N}} \label{eq:comp_ub_ind_step_2} 
\end{align}
Since $\tilde{h}_k \in S_{l-1}$ by Lemma \ref{lem:comp_effic_stat_claim} and for all $h \in S_{l-1}$, $2^{-l+2}n \geq \Delta_h + 2^{-k+1} n$,, we may apply Lemma \ref{lem:width_scale_set} to obtain line \eqref{eq:comp_ub_ind_step_1}.  \eqref{eq:comp_ub_ind_step_2} follows since we assumed $\mc{E}_k$ holds and Lemma \ref{lem:comp_effic_stat_claim}.

Now, we choose 
\begin{align*}
\delta_{k+1} =  \exp(-c^\prime\frac{N}{ \E[\sup_{h \in \H} \frac{(\wt{h}_k-h)^\t A(\lambda)^{-1/2} \rad}{(\wt{h}_k-h)^\t\widehat{\param}_k + 2^{-k+1} n}]^2})
\end{align*}
for a universal constant $c^\prime > 0$ large enough to guarantee that with probability at least $1-\delta_{k+1}$,
\begin{align*}
\frac{\sup_{h,h^\prime \in S_{l+1} } |(h-h^\prime)^\t (\widehat{\param}_{k+1,l}(\delta_{k+1})-\param)|}{2^{-l} n}  \leq \frac{1}{32}.
\end{align*}

\textbf{Step 1.2.2: Combining the estimators into a single estimator $\widehat{\param}_{k+1}$.}  Now, define the polyhedron based on the estimators $\widehat{\param}_{k+1,l}(\delta_{k+1})$ for $l = 0,1,\ldots, k$:
\begin{align*}
P & = \{ y \in [-1,1]^n : \forall l \in [k] \cup \{0\}, \forall h,h^\prime \in \widehat{S}_l :  \frac{|(h-h^\prime)^\t (\widehat{\param}_{k+1,l}(\delta_{k+1})-y)|}{2^{-l} n}  \leq \frac{1}{32} \}.
\end{align*}
We define the estimator $\widehat{\param}_{k+1}$ as follows: if $P$ is nonempty, then let $\widehat{\param}_{k+1}$ be any point in $P$; otherwise, let $\widehat{\param}_{k+1}$ be any point in $\R^n$. 

On the event $\cap_{l \in [k] \cup \{0\}} \mc{E}_{k+1,l} $ , $P$ is nonempty since $\param \in P$ and, thus, $\widehat{\param}_{k+1} \in P$. 

Let $l \in [k] \cup \{0\}$ and $h,h^\prime \in \widehat{S}_{l}$. By the triangle inequality the event $\cap_{l \in [k] \cup \{0\}} \mc{E}_{k+1,l} $, and $\widehat{\mu}_{k+1} \in P$, we have that
\begin{align*}
\frac{|(h-h^\prime)^\t (\widehat{\param}_{k+1} - \param)|}{2^{-l} n} & \leq \frac{|(h-h^\prime)^\t(\widehat{\param}_{k+1}-\widehat{\param}_{k+1,l})|}{2^{-l} n } +\frac{|(h-h^\prime)^\t(\param-\widehat{\param}_{k+1,l})|}{2^{-l} n } \leq \frac{1}{16}.
\end{align*}

By the union bound, with probability at least $1-\delta_{k+1} \log(n \epsilon^{-1})$, for every $l \in [k] \cup \{0\}$,
\begin{align*}
\sup_{h,h^\prime \in S_{l+1} } \frac{|(h-h^\prime)^\t (\widehat{\param}_{k+1}-\param)| }{2^{-l} n} \leq \sup_{h,h^\prime \in \widehat{S}_l } \frac{|(h-h^\prime)^\t (\widehat{\param}_{k+1}-\param)| }{2^{-l} n} \leq \frac{1}{16}.
\end{align*}
where we used $S_{l+1} \subset \widehat{S}_l$ from Claim \ref{cl:subset_rel}. Furthermore, since $\widehat{\param}_k \in [-1,1]^n$ note that $\widehat{S}_0 = \H$ and so we also have that
\begin{align*}
\sup_{h,h^\prime \in \H } \frac{|(h-h^\prime)^\t (\widehat{\param}_{k+1}-\param)| }{ n} \leq \frac{1}{16}.
\end{align*}
Thus, we have shown that for all $l \in [k] \cup \{0\}$,
\begin{align}
\sup_{h,h^\prime \in S_{l} } \frac{|(h-h^\prime)^\t (\widehat{\param}_{k+1}-\param)| }{2^{-l} n} \leq  \frac{1}{16}. \label{eq:fixed_main_inequality}
\end{align}

Now, we are ready to finish the inductive step. Let $h \in S_{k+1}^c$. Let $j$ be the largest integer such that $h \in S_j$. Then, $2^{-j}n \leq \Delta_h \leq 2^{-j+1} n$. Then, the inequality \eqref{eq:fixed_main_inequality} implies that
\begin{align}
\frac{|(\hstar-h)^\t (\widehat{\param}_{k+1} - \param)|}{\Delta_h} & \leq  \frac{|(\hstar-h)^\t (\widehat{\param}_{k+1} - \param)|}{2^{-j} n} \\
& \leq  \sup_{h,h^\prime \in S_{j} } \frac{|(h^\prime-h)^\t (\widehat{\param}_{k+1} - \param)|}{2^{-j+1} n} \\
& \leq \frac{1}{16}. \label{eq:ind_step_gap_result_1}
\end{align}
A similar argument shows that if $h \in S_{k+1}$,
\begin{align}
|(\hstar -h)^\t(\widehat{\param}_{k+1} - \param)| \leq  \frac{2^{-k} n}{16}. \label{eq:ind_step_gap_result_2}
\end{align}

\textbf{Step 2: Correctness} Consider the final round $\bar{k} = \log(n \epsilon^{-1})$ and let $h$ such that $\Delta_h \geq \epsilon$. Then, by the previous step, we have that
\begin{align*}
|(h_* -h)^\t(\widehat{\param}_k - \param)| \leq  \frac{\Delta_h}{8}.
\end{align*}
Therefore, $(\hstar - h)^\t \widehat{\param}_{\bar{k}} \geq \frac{7}{8} \Delta_h > 0$. Thus, $h$ cannot be the empirical maximizer in the final round and the algorithm outputs $\bar{h} \in \H$ such that $\param^\t \bar{h} \geq \param^\t \hstar - \epsilon$.

\textbf{Step 3: Bounding the probability of error} Now, we need to bound the probability of error. We bound $\delta_k$ for all $k$. This argument is quite similar to the one given in \cite{katz2020empirical}. Fix a round $k$.
Lemma \ref{lem:comp_effic_stat_claim} yields
\begin{align}
\E_{\rad \sim N(0,I)}[ \max_{h \in \H} \frac{(\tilde{h}_k-h)^\t A(\lambda)^{-1/2} \rad}{2^{-k+1} n + \widehat{\param}^\t_{k+1}(\tilde{h}_k - h) }]^2 & \leq c \E_{\rad \sim N(0,I)}[ \max_{h \in \H} \frac{(\tilde{h}_k-h)^\t A(\lambda)^{-1/2} \rad}{2^{-k+1} n + \Delta_h }]^2 \\
& \leq c^\prime [\E_{\rad \sim N(0,I)}[ \max_{h \in \H} \frac{(h_*-h)^\t A(\lambda)^{-1/2} \rad}{2^{-k+1} n + \Delta_h }]^2 \\
& + \E_{\rad \sim N(0,I)}[ \max_{h \in \H} \frac{(h_*-\tilde{h}_k)^\t A(\lambda)^{-1/2} \rad}{2^{-k+1} n + \Delta_h }]^2] \label{eq:bound_prob_main}
\end{align}
We start by bounding the first term. Fix $h_0 \in \H \setminus \{h_*\}$. 
\begin{align}
\E_{\rad \sim N(0,I)}[& \max_{h \in \H} \frac{(h_*-h)^\t A(\lambda)^{-1/2} \rad}{2^{-k+1} n + \Delta_h }]^2 \nonumber \\
& \leq \E_{\rad \sim N(0,I)}[ \max_{h \in \H \setminus \{h_* \}} |\frac{(h_*-h)^\t A(\lambda)^{-1/2} \rad}{2^{-k+1} n + \Delta_h }|]^2 \nonumber \\
& \leq 8\E_{\rad \sim N(0,I)}[ \max_{h \in \H \setminus \{h_* \}} \frac{(h_*-h)^\t A(\lambda)^{-1/2} \rad}{2^{-k+1} n + \Delta_h }]^2 +  8\frac{\norm{h_*- h_0}_{A(\lambda)^{-1}}^2}{(2^{-k+1} n + \Delta_{h_0})^2 } \label{eq:comp_ub_samp_comp_2} \\
& \leq 8[\E_{\rad \sim N(0,I)}[ \max_{h \in \H \setminus \{h_* \}} \frac{(h_*-h)^\t A(\lambda)^{-1/2} \rad}{ \max(\epsilon,\Delta_h) }]^2 \nonumber \\
& +  \max_{h \neq h_*} \frac{\norm{h_*- h}_{A(\lambda)^{-1}}^2}{ \max(\epsilon,\Delta_h)^2 }] \label{eq:comp_ub_final_1}
\end{align}
where line \eqref{eq:comp_ub_samp_comp_2} is the consequence of exercise 7.6.9 in \citep{vershynin2019high}. 

Now, we turn to the second term. We have that
\begin{align}
\E_{\rad \sim N(0,I)}[ \max_{h \in \H} \frac{(h_*-\tilde{h}_k)^\t A(\lambda)^{-1/2} \rad}{2^{-k+1} n + \Delta_h }]^2 & \leq \E_{\rad \sim N(0,I)}[ \max(\frac{(h_*-\tilde{h}_k)^\t A(\lambda)^{-1/2} \rad}{2^{-k+1} n  },0)]^2 \nonumber \\
& \leq c \frac{\norm{h_*-\tilde{h}_k}_{A(\lambda)^{-1}}^2}{(2^{-k+1} n)^2} \nonumber \\
& \leq c^\prime \frac{\norm{h_*-\tilde{h}_k}_{A(\lambda)^{-1}}^2}{\max(\epsilon,\Delta_{\tilde{h}_k})^2} \label{eq:comp_ub_samp_comp_3} \\
& \leq c^\prime \max_{h \in \H \setminus \{h_*\}}  \frac{\norm{h_*-h}_{A(\lambda)^{-1}}^2}{\max(\epsilon, \Delta_{h})^2} \label{eq:comp_ub_final_2}
\end{align}
where we obtain line \eqref{eq:comp_ub_samp_comp_3} since $\tilde{h}_k \in S_{k+2}$ by Lemma \ref{lem:comp_effic_stat_claim}. 

Plugging in the design used at round $k$, $\lambda_{k}$, we have that
\begin{align}
\E_{\rad \sim N(0,I)}[ \max_{h \in \H} \frac{(\tilde{h}_k-h)^\t A(\lambda_k)^{-1/2} \rad}{2^{-k+1} n + \widehat{\param}^\t_{k+1}(\tilde{h}_k - h) }]^2 & = \min_{\lambda} \E_{\rad \sim N(0,I)}[ \max_{h \in \H} \frac{(\tilde{h}_k-h)^\t A(\lambda)^{-1/2} \rad}{2^{-k+1} n + \widehat{\param}^\t_{k+1}(\tilde{h}_k - h) }]^2 \label{eq:def_lambda_fb}\\
& \leq c \min_{\lambda} [\rho^*(\epsilon;\lambda) + \gamma^*(\epsilon;\lambda)] \label{eq:ub_objective_fb} \\
& \leq c^\prime [\rho^*(\epsilon) + \gamma^*(\epsilon)] \label{eq:ub_relate_to_gamma_rho}
\end{align}
where 
\begin{align*}
\rho^*(\epsilon;\lambda) & :=  \sup_{h \in \H \setminus \{\hstar \}} \frac{\norm{\hstar-h}^2_{A(\lambda)^{-1}}}{ \max(\param^\t (\hstar-h),\epsilon)^2} \\
\gamma^*(\epsilon;\lambda) & :=  \E_{\rad \sim N(0,I)}[ \sup_{h \in \H \setminus \{\hstar \}} \frac{(\hstar-h)^\t A(\lambda)^{-1/2} \rad}{ \max(\param^\t(\hstar-h ),\epsilon)}]^2.
\end{align*}
and \eqref{eq:def_lambda_fb} follows by definition of $\lambda_k$, \eqref{eq:ub_objective_fb} follows by \eqref{eq:bound_prob_main}, \eqref{eq:comp_ub_final_1}, and \eqref{eq:comp_ub_final_2}, and \eqref{eq:ub_relate_to_gamma_rho} follows by Lemma 13 of \cite{katz2020empirical}.

Putting it together, for all $k$
\begin{align*}
\delta_{k+1} & = \exp(-c^\prime\frac{N}{ \E[\sup_{h \in \H} \frac{(\wt{h}_k-h)^\t A(\lambda)^{-1/2} \rad}{(\wt{h}_k-h)^\t\widehat{\param}_k + 2^{-k+1} n}]^2}) \\
& \leq \exp(-c^{\prime \prime} \frac{N}{ [\gamma^*(\epsilon) + \rho^*(\epsilon)]}) \\
& \leq \exp(-c^{\prime \prime \prime}  \frac{T}{ \log(n \epsilon^{-1}) [\gamma^*(\epsilon) + \rho^*(\epsilon)]}) 
\end{align*}
This completes the proof.

 \end{proof}
 
 \begin{remark}
 We stress that the construction of the estimators in the proof of Theorem \ref{thm:fixed_budget_combi} is based solely on data observed by the algorithm and $\H$ and at no point is knowledge of $\param$ used. 
 \end{remark}

\subsection{Technical Lemmas}

The proof of Theorem \ref{thm:fixed_budget_combi} uses the following Lemmas, which appeared originally as Lemma 11, Lemma 1, and Lemma 13 in \cite{katz2020empirical}.

\begin{lemma}
\label{lem:rho_gamma}
Let $\G \subset \H$. Then,
\begin{align*}
\E_{\rad \sim N(0,I)}[ \max_{h,h^\prime \in \H} [A(\lambda)^{-1/2}(h-h^\prime)]^\t  \rad]^2 \geq \frac{2}{\pi} \max_{h,h^\prime \in \H} \norm{h-h^\prime}_{A(\lambda)^{-1}}^2.
\end{align*}
\end{lemma}

\begin{lemma}
\label{lem:comp_effic_stat_claim}
Let $k \geq 1$. Consider the $k$th round of Algorithm \ref{alg:fixed_budget_comp_combi}. Suppose that
 \begin{itemize}
\item if $h \in S_k^c$,
\begin{align}
|(h_* -h)^\t(\widehat{\param}_k - \param)| \leq  \frac{\Delta_h}{8} \label{eq:comp_effic_claim_hyp_1}
\end{align}
\item if $h \in S_k$,
\begin{align}
|(h_* -h)^\t(\widehat{\param}_k - \param)| \leq  \frac{2^{-k+1} n}{8}. \label{eq:comp_effic_claim_hyp_2}
\end{align}
\end{itemize}
Then, the following hold:
\begin{enumerate}
\item 
\begin{align}
\tilde{h}_k  \in S_{k+2}, \label{eq:comp_effic_claim_result_1}
\end{align}
\item if $h \in S_k^c$
\begin{align}
|(\tilde{h}_k-h)^\t\widehat{\param}_k -(h_*-h)^\t \param | & \leq  \frac{1}{2} \Delta_h. \label{eq:comp_effic_claim_result_2}
\end{align}
\item if $h \in S_k$,
\begin{align}
|(\tilde{h}_k-h)^\t\widehat{\param}_k -(h_*-h)^\t \param | & \leq \frac{1}{2} 2^{-k+1} n. \label{eq:comp_effic_claim_result_3}
\end{align}
\item There exist universal constants $c,c^\prime>0$ such that
\begin{align*}
c \E[\sup_{h \in \H} \frac{(\tilde{h}_k-h)^\t A(\lambda)^{-1/2} \rad}{\Delta_{h} + 2^{-k+1} n}]^2 & \leq \E[\sup_{h \in \H} \frac{(\tilde{h}_k-h)^\t A(\lambda)^{-1/2} \rad}{(\tilde{h}_k-h)^\t\widehat{\param}_k + 2^{-k+1} n}]^2 \\
& \leq c^\prime \E[\sup_{h \in \H} \frac{(\tilde{h}_k-h)^\t A(\lambda)^{-1/2} \rad}{\Delta_{h} + 2^{-k+1} n}]^2
\end{align*}
\end{enumerate}
\end{lemma}

\begin{lemma}
\label{lem:width_scale_set}
Let $V = \{v_1, \ldots, v_l \} \subset \R^d$
with $0 \in V$. Suppose $a_i \geq 1$ for all $i \in [l]$. Then,
\begin{align*}
\E_{\rad \sim N(0,I)}[ \sup_{v_i \in V} v_i^\t \rad] \leq \E_{\rad \sim N(0,I)} [\sup_{v_i \in V} a_i v_i^\t \rad]
\end{align*}
\end{lemma}

\section{Efficient Fixed Budget Algorithm}

\begin{algorithm}[h]\small

\begin{algorithmic}
\STATE \textbf{Input:} Budget $T$, tolerance $\epsilon > 0$.

\STATE $\widehat{\param}_1 = \zero \in \R^d$, $N \longleftarrow \floor{T/\log_2(n\epsilon^{-1})}$.

\FOR{$k=1,2,\ldots, \floor{\log_2(d\epsilon^{-1})}$}
    \STATE $\tilde{h}_{k} \longleftarrow \argmax_{h \in \H} \widehat{\param}_{k}^\t h$
    
    \STATE Let $\lambda_k^{(1)} $ be the solution of the following optimization problem
    \begin{align}
    \inf_{\lambda \in \simp} \E_{\rad \sim N(0,I)}[ \max_{h \in \H} \frac{(\tilde{h}_k- h)^\t A(\lambda)^{-1/2} \rad}{2^{-k+1} n+ \widehat{\param}^\t_{k}(\tilde{h}_k - h) }]^2 \label{eq:action_efficient_1}
    \end{align}
    
    \STATE Let $\lambda_k^{(2)}$ be the solution to 
    \begin{align}
    \inf_{\lambda \in \simp} \max_{h \in \H} \max_{i \in \tilde{h}_k \Delta h} \frac{\frac{1}{\lambda_i}}{2^{-k+1}n +  \widehat{\param}_k^\t(\tilde{h}_k - h)} \label{eq:action_efficient_2}
    \end{align}
    
    \STATE $\lambda_k \longleftarrow \frac{1}{2}(\lambda_k^{(1)} + \lambda_k^{(2)})$
    
    \STATE Sample $\{x_1, \ldots, x_{N} \}  \sim \lambda_k$.
    
    \STATE Query $x_1, \ldots, x_{N} $ and receive rewards $y_1, \ldots, y_{N}$.
    
    \STATE Let $\widehat{\param}_{k+1} = \frac{1}{N} \sum_{s=1}^N A(\lambda)^{-1} x_{s} y_s$.
\ENDFOR

\STATE \textbf{Return} $\argmax_{h \in \H} \widehat{\param}_{k+1}^\t h$

\end{algorithmic}
\caption{Fixed Budget ACED for Combinatorial Bandits (Computationally Efficient).}
\label{alg:fixed_budget_comp_combi_efficient}
\end{algorithm} 

In this section, we present Algorithm \ref{alg:fixed_budget_comp_combi_efficient}, which can be implemented in practice. Algorithm \ref{alg:fixed_budget_comp_combi_efficient} resembles Algorithm \ref{alg:fixed_budget_comp_combi} with two differences. First, it uses the IPS estimator instead of the theoretical estimator derived in the proof of Theorem \ref{thm:fixed_budget_combi}. Second, it mixes the design (given in  used in \eqref{eq:action_comp_2_combi}) with the design defined in \eqref{eq:action_efficient_2} in order to control the worst-case deviations of the IPS estimator.

Now, we briefly discuss why Algorithm \ref{alg:fixed_budget_comp_combi_efficient} can be efficiently implemented. As is standard in combinatorial bandits, we assume access to a linear maximization oracle: for any $v \in \R^d$, we can compute 
\begin{align*}
    \text{Oracle}(v) = \argmax_{z \in \mc{Z}} v^\t z
\end{align*}
efficiently.\footnote{In the setting of active classification, this is equivalent to assuming a weighted $0/1$ loss minimization oracle.} \cite{katz2020empirical} gave a procedure for efficiently solving \eqref{eq:action_efficient_1} up to a constant factor assuming a linear maximization oracle, and so we focus on \eqref{eq:action_efficient_2}. Define the function $g(\lambda) = \max_{i \in \tilde{h}_k \Delta h} \frac{\frac{1}{\lambda_i}}{2^{-k+1}n + \widehat{\param}_k^\t(\tilde{h}_k - h)}$. $g$ is convex since it is the maximum of a collection of convex functions. One can compute the subgradient of $g$ using the linear maximization oracle. For each $i \in \tilde{h}_k$, one can compute
\begin{align*}
\max_{h : i \not \in h} \widehat{\param}^\t_k h
\end{align*}
by defining 
\begin{align*}
\tilde{\param}_j^{(i)} = \begin{cases}
 \tilde{\param}_j & j \neq i \\
 -\infty & j = i
\end{cases}
\end{align*}
and use the linear maximization oracle to compute $\max_{h } h^\t \tilde{\param}^{(i)}$. Using a similar technique, for each $i \not \in \tilde{h}_k$, one can compute
\begin{align*}
\max_{h : i  \in h} \widehat{\param}^\t_k h.
\end{align*}
Then, using a standard result about subgradients, we have that $\partial g(\lambda)$ consists of $\frac{\partial }{\partial \lambda}  \frac{\frac{1}{\lambda_{\tilde{i}}}}{\widehat{\param}_k^\t(\tilde{h}_k - h)}$ where $\tilde{i}$ and $h$ attain the maximum in
\begin{align*}
\max_{h \in \H} \max_{i \in \tilde{h}_k \Delta h} \frac{\frac{1}{\lambda_i}}{2^{-k+1}n + \widehat{\param}_k^\t(\tilde{h}_k - h)}.
\end{align*}
Thus, using this trick for computing the subgradient and any optimization procedure for nonsmooth convex optimization, we can optimize \eqref{eq:action_efficient_2}. In Theorem \ref{thm:fixed_budget_efficient_combi}, for the sake of simplicity and focusing on the key ideas, we suppose that \eqref{eq:action_efficient_1} and \eqref{eq:action_efficient_2} are solved exactly, but using the optimization ideas laid out here would only affect the sample complexity up to constant factors.

The following complexity parameter is a key term in Theorem \ref{thm:fixed_budget_efficient_combi}:
\begin{align*}
\psi^*(\epsilon) : = \min_{\lambda \in \simp} \max_{h \in \H \setminus \{\hstar\}} \max_{i \in \hstar \Delta h} \frac{\frac{1}{\lambda_i}}{\max(\epsilon, \param^\t (\hstar - h))}.
\end{align*}

\begin{theorem}\label{thm:fixed_budget_efficient_combi}
Let $T \in \N$ and $\epsilon > 0$. If $T \geq c_1\log(n \epsilon^{-1}) [\gamma^*(\epsilon) + \rho^*(\epsilon) + \psi^*(\epsilon) \log(|\H|)\log(\log(|\H|))]$ , then Algorithm \ref{alg:fixed_budget_comp_combi_efficient} satisfies
\begin{align*}
\P( \param^\t\widehat{h}  +\epsilon < \param^\t \hstar ) \leq c_2\log(n \epsilon^{-1})  \exp(-\frac{T - \log(|\H|)\log(\log(|\H|)) \psi^*(\epsilon)}{\log(n \epsilon^{-1}) [\gamma^*(\epsilon) + \rho^*(\epsilon) + \psi^*(\epsilon)]})
\end{align*}
where $c_1,c_2 >0$ are universal constants.
\end{theorem}

\begin{proof}[Proof of Theorem \ref{thm:fixed_budget_efficient_combi}]
Let $\alpha > 0$ be a universal constant that will be specified later. Let $\delta > 0$ such that 
\begin{align*}
T = \alpha \log(n \epsilon^{-1}) [\log(1/\delta)(\gamma^*(\epsilon) + \rho^*(\epsilon) + \psi^*(\epsilon)) + \psi^*(\epsilon) \log(|\H|)\log(\log(|\H|))],
\end{align*}
which exists since by assumption $T \geq c\log(n \epsilon^{-1}) [\gamma^*(\epsilon) + \rho^*(\epsilon) + \psi^*(\epsilon) \log(|\H|)\log(\log(|\H|))]$. Note that
\begin{align}
N =\log(1/\delta) \alpha (\gamma^*(\epsilon) + \rho^*(\epsilon) + \psi^*(\epsilon)) + \psi^*(\epsilon) \log(|\H|)\log(\log(|\H|)). \label{eq:value_N}
\end{align}
 Recall $\widehat{\param}_k = \frac{1}{N}A(\lambda)^{-1} \sum_{s=1}^N x_{I_s} y_s$. Recall 
\begin{align*}
S_k = \{h \in \H : \Delta_h \leq n 2^{-k+1}\}.
\end{align*} 
Note that $S_0 = \H$ since $\param \in [-1,1]$ by assumption. Define the event at round $k$ for $l \in [k] \cup \{0\}$,
\begin{align*}
\mc{E}_{k,l}  = & \{\sup_{h,h^\prime \in S_l}  | (h-h^\prime)^\t (\widehat{\param}_k  -\param)|)  \leq \\
&c[\E[\sup_{h \in S_l} h^\t A(N\lambda)^{-1/2} \rad]  
  +(\log(\log(|S_l|)) \log(|S_l|) +\log(1/\delta)) \frac{\max_{h,h^\prime \in S_l} \max_{i \in h \Delta h^\prime} \frac{1}{\lambda_i}}{N} \\
  & + \norm{h-h^\prime}_{A(N \lambda)^{-1}} \sqrt{\log(1/\delta)}) \} \\
\end{align*}
Define $  \mc{E}_k = \cap_{l = 0}^k \mc{E}_{k,l}$ and $\mc{E} = \cap_{k=1}^{\log_2(n \epsilon^{-1})}\mc{E}_k$. Now, using the law of total probability, the independence of samples between rounds, and Lemma \ref{lem:sup_ips_bound}
\begin{align*}
\P(\mc{E}^c) & \leq \sum_{k \geq 1} \sum_{l=0}^k \P(\mc{E}_{k,l}^c| \cap_{j=1}^{k-1} \mc{E}_j) \leq \log(n \epsilon^{-1})^2 \delta.
\end{align*} 
Suppose that $\mc{E}$ occurs for the remainder of the proof. Note that using the same series of inequalities as in \eqref{eq:budget_induct_1}-\eqref{eq:budget_induct_3}, we have that 
\begin{align*}
& \sup_{h,h^\prime \in S_l}  | (h-h^\prime)^\t (\widehat{\param}_k  -\param)|)  \leq \\
&c^\prime[\sqrt{\log(1/\delta)})\E[\sup_{h \in S_l} h^\t A(N\lambda)^{-1/2} \rad]  
  +(\log(\log(|S_l|)) \log(|S_l|) +\log(1/\delta)) \frac{\max_{h,h^\prime \in S_l} \max_{i \in h \Delta h^\prime} \frac{1}{\lambda_i}}{N} ]
\end{align*}

We argue inductively that at every round $k \geq 2$
 \begin{enumerate}
\item for all $h \in S_k^c$,
\begin{align*}
|(h_* -h)^\t(\widehat{\param}_k - \param)| \leq  \frac{\Delta_h}{8}
\end{align*}
\item for all $h \in S_k$,
\begin{align*}
|(h_* -h)^\t(\widehat{\param}_k - \param)| \leq  \frac{2^{-k+1} n}{8}.
\end{align*}
\end{enumerate}

\textbf{Step 1: Base Case.} Let $k = 2$. Then, the event $\mc{E}_{2,0}$ implies that
\begin{align}
& \sup_{h \in \H}  \frac{| (\hstar-h)^\t (\widehat{\param}_k  -\param)|) }{2^{-1} n} \nonumber \\ 
 & \leq \sup_{h,h^\prime \in \H}  \frac{| (h-h^\prime)^\t (\widehat{\param}_k  -\param)|) }{2^{-1} n} \nonumber \\
\leq &\frac{1}{2^{-1} n} c^\prime[\sqrt{\log(1/\delta)})\E[\sup_{h \in \H} h^\t A(N\lambda)^{-1/2} \rad]  
  +(\log(|\H|) +\log(1/\delta)) \frac{\max_{h,h^\prime \in \H} \max_{i \in h \Delta h^\prime} \frac{1}{\lambda_i}}{N} ] \nonumber \\
  & = \frac{1}{2^{-1} n} c^\prime[\sqrt{\frac{\log(1/\delta) \E[\sup_{h \in \H} h^\t A(\lambda)^{-1/2} \rad]^2}{\log(1/\delta)\alpha[\gamma^*(\epsilon)+\rho^*(\epsilon)]}}  
  +(\log(|\H|) +\log(1/\delta)) \frac{\max_{h,h^\prime \in \H} \max_{i \in h \Delta h^\prime} \frac{1}{\lambda_i}}{\alpha \psi^*(\epsilon)(\log(\log(\H) )\log(\H) + \log(1/\delta))} ] . \label{eq:base_case_bound}
\end{align}
where the last line follows by \eqref{eq:value_N}. Recall that $\lambda_1 = \frac{1}{2}(\lambda^{(1)}_1 + \lambda_2^{(1)})$ where $\lambda^{(1)}_1$ is the minimizer of \eqref{eq:action_efficient_1} and $\lambda^{(1)}_1$ is the minimizer of \eqref{eq:action_efficient_2}. We have that
\begin{align}
 \E[\sup_{h \in \H} \frac{h^\t A(\lambda_1)^{-1/2} \rad}{2^{-1} n}]^2 & \leq  4\E[\sup_{h \in \H} \frac{h^\t A(\lambda_1^{(1)})^{-1/2} \rad}{2^{-1} n}]^2 \label{eq:gamma_bound_sudakov_fernique} \\
 & = 4\min_{\lambda} \E[\sup_{h \in \H} \frac{h^\t A(\lambda)^{-1/2} \rad}{2^{-1} n}]^2 \label{eq:gamma_bound_minimizer} \\
 & \leq   c[\gamma^*(\epsilon) + \rho^*(\epsilon)] \label{eq:gamma_bound_previous_arg}
\end{align}
\eqref{eq:gamma_bound_sudakov_fernique} follows by sudakov-fernique inequality since $A(\lambda_1) \succeq  \frac{1}{2} A(\lambda_1^{(1)})$ implies that $\sqrt{2} A(\lambda_1^{(1)})^{-1/2} \succeq A(\lambda_1)^{-1/2}$. Line \eqref{eq:gamma_bound_minimizer} follows by definition of $\lambda_1^{(1)}$. Finally, \eqref{eq:gamma_bound_previous_arg} follows by the same series of inequalities that established \eqref{eq:ub_relate_to_gamma_rho}.

Fix $h, h^\prime \in \H$. Then, if $i \in h \Delta h^\prime$, then either $i \in h \Delta \hstar$ or $i \in h^\prime \Delta \hstar$. Thus, 
\begin{align}
\max_{h,h^\prime \in \H} \max_{i \in h \Delta h^\prime}\frac{\frac{1}{\lambda_{1,i}}}{2^{-1} n} & \leq 2 \max_{h \in \H \setminus \{\hstar\}} \max_{i \in h \Delta \hstar}\frac{\frac{1}{\lambda_{1,i}}}{2^{-1} n} \nonumber \\
& \leq  4 \max_{h \in \H \setminus \{\hstar\}} \max_{i \in h \Delta \hstar}\frac{\frac{1}{\lambda^{(2)}_{1,i}}}{2^{-1} n} \label{eq:psi_bound_def_lambda_1} \\
& = 8\min_{\lambda} \max_{h \in \H \setminus \{\hstar\}} \max_{i \in h \Delta \hstar}\frac{\frac{1}{\lambda_i}}{2n} \label{eq:psi_bound_def_lambda_2} \\
& \leq c \psi^*(\epsilon) \label{eq:psi_bound}
\end{align}
where \eqref{eq:psi_bound_def_lambda_1} follows by definition of $\lambda_1$, \eqref{eq:psi_bound_def_lambda_2} follows by definition of $\lambda_1^{(2)}$, and \eqref{eq:psi_bound} follows by definition of $\psi^*(\epsilon)$ and since for all $h \in \H \setminus \{\hstar\}$, $\Delta_h \leq 2n$. 

Then, putting together \eqref{eq:base_case_bound}, \eqref{eq:gamma_bound_previous_arg}, and \eqref{eq:psi_bound}, we have that if the universal constant $\alpha$ is large enough, then
\begin{align*}
\sup_{h,h^\prime \in \H}  \frac{| (h-h^\prime)^\t (\widehat{\param}_k  -\param)|) }{2^{-1} n} & \leq \frac{1}{8}
\end{align*}
this proves the base case.

\textbf{Step 2: Inductive Step.} Suppose that at round $k \geq 2$
 \begin{enumerate}
\item for all $h \in S_k^c$,
\begin{align*}
|(h_* -h)^\t(\widehat{\param}_k - \param)| \leq  \frac{\Delta_h}{8}
\end{align*}
\item for all $h \in S_k$,
\begin{align*}
|(h_* -h)^\t(\widehat{\param}_k - \param)| \leq  \frac{2^{-k+1} n}{8}.
\end{align*}
\end{enumerate}
Now, we prove the statement for round $k+1$. Let $l \in [k+1]$. Using $\mc{E}_{k+1,l}$, we have that
\begin{align}
\frac{\sup_{h,h^\prime \in S_{l} } |(h-h^\prime)^\t (\widehat{\param}_{k+1}-\param)| }{2^{-l }n} \leq c \sqrt{\log(1/\delta) \E[\sup_{h \in S_{l} } \frac{h^\t A(N\lambda_k)^{-1/2} \rad}{2^{-l} n}]^2} \\
+(\log(\log(|S_l|))\log(|S_l|) + \log(1/\delta)) \frac{\max_{h,h^\prime \in S_l} \max_{i \in h \Delta h^\prime}\frac{1/\lambda_{k,i}}{2^{-l} n}}{N} \label{eq:induct_step_bound}
\end{align}
Bounding the first term, we have that
\begin{align}
\sqrt{\log(1/\delta) \E[\sup_{h \in S_{l} } \frac{h^\t A(N\lambda_k)^{-1/2} \rad}{2^{-l} n}]^2} & \leq c^\prime \sqrt{\log(1/\delta) \frac{\E[\sup_{h \in S_{l}} \frac{(\tilde{h}_k-h)^\t A(\lambda_k)^{-1/2} \rad}{2^{-l+2} n}]^2}{N}} \nonumber \\
& \leq  c^{\prime \prime} \sqrt{ \log(1/\delta)\frac{\E[\sup_{h \in S_{l}} \frac{(\tilde{h}_k-h)^\t A(\lambda_k)^{-1/2} \rad}{\Delta_{h} + 2^{-k+1} n}]^2}{N}} \label{eq:ips_comp_ub_ind_step_1} \\
& \leq c^{\prime \prime } \sqrt{  \log(1/\delta)\frac{\E[\sup_{h \in \H} \frac{(\tilde{h}_k-h)^\t A(\lambda_k)^{-1/2} \rad}{\Delta_{h} + 2^{-k+1} n}]^2}{N}}  \\
& \leq c^{\prime \prime \prime} \sqrt{   \log(1/\delta)\frac{\E[\sup_{h \in \H} \frac{(\tilde{h}_k-h)^\t A(\lambda_k)^{-1/2} \rad}{(\tilde{h}_k-h)^\t\widehat{\param}_k + 2^{-k+1} n}]^2}{N}} \label{eq:ips_comp_ub_ind_step_2} 
\end{align}
Since $\tilde{h}_k \in S_{l}$ by Lemma \ref{lem:comp_effic_stat_claim} and for all $h \in S_{l-1}$, $2^{-l+2}n \geq \Delta_h + 2^{-k+1} n$,, we may apply Lemma \ref{lem:width_scale_set} to obtain line \eqref{eq:ips_comp_ub_ind_step_1}.  \eqref{eq:ips_comp_ub_ind_step_2} follows since we assumed $\mc{E}_k$ holds and Lemma \ref{lem:comp_effic_stat_claim}. Using a similar series of inequalities to Step 3 in the proof of Theorem \ref{thm:fixed_budget}, we have that
\begin{align}
\sqrt{\log(1/\delta) \E[\sup_{h \in S_{l} } \frac{h^\t A(N\lambda_k)^{-1/2} \rad}{2^{-l} n}]^2} & \leq c^{\prime \prime \prime} \sqrt{   \log(1/\delta)\frac{\E[\sup_{h \in \H} \frac{(\tilde{h}_k-h)^\t A(\lambda_k)^{-1/2} \rad}{(\tilde{h}_k-h)^\t\widehat{\param}_k + 2^{-k+1} n}]^2}{N}} \nonumber \\
& \leq c^{\prime \prime \prime \prime} \sqrt{   \log(1/\delta)\frac{\E[\sup_{h \in \H} \frac{(\tilde{h}_k-h)^\t A(\lambda_k^{(1)})^{-1/2} \rad}{(\tilde{h}_k-h)^\t\widehat{\param}_k + 2^{-k+1} n}]^2}{N}} \label{eq:ips_comp_ind_step_sudakov_fernique} \\
& = c^{\prime \prime \prime \prime} \min_{\lambda}  \sqrt{   \log(1/\delta)\frac{\E[\sup_{h \in \H} \frac{(\tilde{h}_k-h)^\t A(\lambda)^{-1/2} \rad}{(\tilde{h}_k-h)^\t\widehat{\param}_k + 2^{-k+1} n}]^2}{N}} \label{eq:ips_comp_ind_step_def_lambda} \\
& \leq c^{\prime \prime \prime \prime \prime} \sqrt{   \log(1/\delta)\frac{\rho^*(\epsilon) + \gamma^*(\epsilon)}{N}} \label{eq:ips_comp_prev_arg} \\
& \leq \frac{1}{32} \label{eq:gamma_term_bound}
\end{align}
where \eqref{eq:ips_comp_ind_step_sudakov_fernique} follows by Sudakov-Fernique inequality and the definition of $\lambda_k$, \eqref{eq:ips_comp_ind_step_def_lambda} follows by the definition of $\lambda_k^{(1)}$, \eqref{eq:ips_comp_prev_arg} follows by the same series of inequalities used to establish \eqref{eq:ub_relate_to_gamma_rho}, and the last line follows by plugging in \eqref{eq:value_N} and letting $\alpha$ be a large enough universal constant. 

Now, we bound the second term. We have that
\begin{align}
\max_{h,h^\prime \in S_l} \max_{i \in h \Delta h^\prime} \frac{1/\lambda_{k,i}}{2^{-l} n} & \leq 2 \max_{h \in S_l} \max_{i \in \wt{h}_k \Delta h} \frac{1/\lambda_{k,i}}{2^{-l} n} \nonumber \\
& \leq 4 \max_{h \in S_l} \max_{i \in \wt{h}_k \Delta h } \frac{1/\lambda_{k,i}}{2^{-k+1} n + \Delta_h} \label{eq:bernstein_gap} \\
& \leq c \max_{h \in S_l} \max_{i \in \tilde{h}_k \Delta h} \frac{1/\lambda_{k,i}}{2^{-k+1} n + \widehat{\param}_k^\t(\tilde{h}_k -h)} \label{eq:bernstein_objective} \\
& \leq  c^\prime \max_{h \in S_l} \max_{i \in \wt{h}_k \Delta h } \frac{1/\lambda_{k,i}}{2^{-k+1} n + \Delta_h} \label{eq:bernstein_gap_2} \\
& \leq  c^{\prime \prime} \max_{h \in S_l} \max_{i \in \wt{h}_k \Delta h } \frac{1/\lambda_{k,i}^{(2)}}{2^{-k+1} n + \Delta_h} \label{eq:bernstein_def_lambda} \\
& =  c^{\prime \prime} \min_{\lambda} \max_{h \in S_l} \max_{i \in \wt{h}_k \Delta h } \frac{1/\lambda_{i}}{2^{-k+1} n + \Delta_h} \label{eq:bernstein_def_lambda_2} \\
& \leq c^{\prime \prime \prime} \psi(\epsilon) \label{eq:bernstein_def_psi}
\end{align}
where in \eqref{eq:bernstein_gap} we used that since for all $h \in S_l$, $2^{-l} n \geq \frac{2^{-k+1} n + \Delta_h}{2}$ and and in \eqref{eq:bernstein_objective} we used the inductive hypothesis and Lemma \ref{lem:comp_effic_stat_claim} and in \eqref{eq:bernstein_gap_2}, we used the inductive hypothesis and Lemma \ref{lem:comp_effic_stat_claim} again, \eqref{eq:bernstein_def_lambda} follows by the definition of $\lambda_k$, \eqref{eq:bernstein_def_lambda_2} follows by the definition of $\lambda_k^{(2)}$, and \eqref{eq:bernstein_def_psi} follows by the definition of $\psi(\epsilon)$. Thus, 
\begin{align}
(\log(\log(|S_l|)) \log(|S_l|) + \log(1/\delta)) \frac{\max_{h,h^\prime \in S_l} \max_{i \in h \Delta h^\prime}\frac{1/\lambda_{k,i}}{2^{-l} n}}{N}  & \leq c^{\prime \prime \prime}  (\log(\log(|S_l|)) \log(|S_l|) + \log(1/\delta)) \frac{ \psi(\epsilon)}{N} \\
& \leq \frac{1}{32} \label{eq:berstein_term_bound}
\end{align}
where the last line follows by plugging in \eqref{eq:value_N} and letting $\alpha$ be a large enough universal constant. Then, putting together \eqref{eq:induct_step_bound}, \eqref{eq:gamma_term_bound}, and \eqref{eq:berstein_term_bound} yields for every $l \in [k+1] \cup \{0\}$
\begin{align*}
\frac{\sup_{h,h^\prime \in S_{l} } |(h-h^\prime)^\t (\widehat{\param}_{k+1}-\param)| }{2^{-l }n} \leq \frac{1}{16}
\end{align*}

Using the same series of inequalities used to establish \eqref{eq:ind_step_gap_result_1} and \eqref{eq:ind_step_gap_result_2} completes the inductive step. 

\textbf{Step: Correctness.} This argument is the same as in the correctness step in the proof of Theorem \ref{thm:fixed_budget_combi}
\end{proof}

\subsection{Lemmas}

In this Section, we prove the main concentration inequality, Lemma \ref{lem:sup_ips_bound}, that is used in the proof of Theorem \ref{thm:fixed_budget_efficient_combi}. We need the following definition for the proof of Lemma \ref{lem:sup_ips_bound}.
\begin{definition}
Fix a set $T \subset \R^n$. Let $d$ be a metric on $\R^n$. The $\gamma_1$-functional of $T$ is defined as
\begin{align*}
\gamma_1(T,d) & = \inf_{(T_k)} \sup_{t \in T} \sum_{k=0}^\infty 2^{k} d(t,T_k)
\end{align*}
where the infimum is taken over all admissible sequences of $T$.
\end{definition}

\begin{lemma}\label{lem:sup_ips_bound}
Let $\G \subset \H$. Fix some $\lambda \in \triangle_n$ and draw $\{ x_s \}_{i=1}^t \sim \lambda$ and then observe $y_s$ with mean $ x_s^\t \param$ and $|y_i| \leq 1$ with probability $1$ for $s=1,\ldots, t$. Then, there exists a universal constant $c > 0$ such that for any $u > 0$ with probability at most $\exp(-u)$
\begin{align*}
\sup_{h,h^\prime \in \G}  | (h-h^\prime)^\t & (\frac{1}{t}A(\lambda)^{-1} \sum_{s=1}^t x_{I_s} y_s -\param)|  \leq \\
&c[\E[\sup_{h \in \G} h^\t A(t\lambda)^{-1/2} \rad]  
  +(\log(\log(|\G|))\log(|\G|) +u) \frac{\max_{h,h^\prime \in \G} \max_{i \in h \Delta h^\prime} \frac{1}{\lambda_i}}{t} \\
  & + \norm{h-h^\prime}_{A(t\lambda)^{-1}} \sqrt{u} ]
\end{align*}
\end{lemma}

\begin{proof}
Using Corollary 7.9 of \cite{ledoux2001concentration}, we have that with probability at least $1-\exp(-u)$,
\begin{align}
 \sup_{h,h^\prime \in \G} |  (h-h^\prime)^\t (\frac{1}{t}A(\lambda)^{-1}& \sum_{s=1}^t x_{I_s} y_s -\param)|   \leq \nonumber \\
&c[\E[\sup_{h,h^\prime \in \G} | (h-h^\prime)^\t (\frac{1}{t}A(\lambda)^{-1} \sum_{s=1}^t x_{I_s} y_s -\param)|] 
  +u \frac{\max_{h,h^\prime \in \G} \max_{i \in h \Delta h^\prime} \frac{1}{\lambda_i}}{t}  \nonumber \\
  & +  \norm{h-h^\prime}_{A(t\lambda)^{-1}} \sqrt{u}). \label{eq:bousquet}
\end{align}
Using the standard technique of symmetrization, we have that
\begin{align}
\E[\sup_{h,h^\prime \in \G} | (h-h^\prime)^\t (\frac{1}{t}A(\lambda)^{-1} \sum_{s=1}^t x_{I_s} y_s -\param)|]  \leq 2\E_{\epsilon_1,\ldots, \epsilon_t}[\sup_{h,h^\prime \in \G} | (h-h^\prime)^\t (\frac{1}{t}A(\lambda)^{-1} \sum_{s=1}^t x_{I_s} \epsilon_s|] \label{eq:symmetrization}
\end{align}
where $\epsilon_1, \ldots, \epsilon_t$ are Rademacher random variables. By Bernstein's inequality, we have that
\begin{align}
\P(| (h-h^\prime)^\t (\frac{1}{t}A(\lambda)^{-1} \sum_{s=1}^t x_{I_s} \epsilon_s| \geq u) \leq 2 \exp(-c\frac{u^2}{\norm{h-h^\prime}^2_{A(t\lambda)^{-1}}}, \frac{u}{\max_{i \in h \Delta h^\prime} \frac{1}{t\lambda_i}}). \label{eq:bernstein_bound}
\end{align}
Define the set $\wt{\G} = \{A(\lambda)^{-1} h : h \in \G\}$. Then, using \eqref{eq:bernstein_bound} and applying Theorem 2.2.23 of \cite{talagrand2014upper}, we have that
\begin{align}
\E[\sup_{h,h^\prime \in \G} | (h-h^\prime)^\t (\frac{1}{t}A(\lambda)^{-1} \sum_{s=1}^t x_{I_s} \epsilon_s|] & \leq c [\gamma_2(\wt{\G}, d_2) + \gamma_1(\wt{\G},d_1)] \label{eq:chaining}
\end{align}
where $d_2$ is the metric induced by $\norm{\cdot}_2$ and $d_1$ is the metric induced $\norm{\cdot}_{\infty}$. Using Talagrand's majorizing measure theorem (Theorem 8.6.1 in \cite{vershynin2019high}), we have that
\begin{align}
\gamma_2(\wt{\G},d_2) \leq c\E_{\rad \sim N(0,I)}[ \sup_{h \in \H} h^\t A(\lambda)^{-1} \rad]. \label{eq:majorizing_measure}
\end{align}
From the definition of $\gamma_1$, it follows trivially that
\begin{align}
\gamma_1(\wt{\G},d_1) & \leq \log(\log(|\G|))\log(|\G|) \frac{\max_{h,h^\prime \in \G} \max_{i \in h \Delta h^\prime} \frac{1}{\lambda_i}}{t}. \label{eq:gamma_1}
\end{align}
Finally, putting together \eqref{eq:bousquet}, \eqref{eq:symmetrization}, \eqref{eq:chaining}, \eqref{eq:majorizing_measure}, and \eqref{eq:gamma_1}, we obtain the result.

\end{proof}

\section{Thresholds}
Let $n$ be a power of $2$. 
Let $\H = \{ \mb{e}_{[k]} : k \in [n] \}$ where $[\mb{e}_A]_i = \mathbf{1}\{ i \in A \}$. 
Assume that $\param_* \in \epsilon (2 \mb{e}_{[k_\star]} - 1)$ for some $k_\star \in [n]$.
Then
\begin{align*}
\gamma_* &:= \inf_\lambda \E\left[  \sup_{h \in \H} \frac{ \langle h - h_\star, \widehat{\param}-\param_* \rangle}{\langle h - h_\star, \param_* \rangle}  \right]^2 \\
&= \inf_\lambda \E\left[  \sup_{h \in \H} \frac{ \langle h - h_\star, A(\lambda)^{-1/2} \rad \rangle}{\langle h - h_\star, \param_* \rangle}  \right]^2 \\
&= \inf_\lambda \E\left[  \sup_{k > k_\star} \frac{ \sum_{i=k_\star + 1}^k  \frac{\rad_i}{\sqrt{\lambda_i}} }{ (k-k_\star) h } \vee \sup_{k < k_\star} \frac{ \sum_{i=k+1}^{k_\star}  \frac{\rad_i}{\sqrt{\lambda_i}} }{ (k_\star-k) h }  \right]^2 \\
&\leq \inf_\lambda  \frac{4}{\epsilon^2} \E\left[ \max_{k=1,\dots,n} \frac{1}{k} \sum_{i=1}^k \frac{\rad_i}{\sqrt{\lambda_i}} \right]^2.
\end{align*}
and
\begin{align*}
    \max_{k=1,\dots,n} \frac{1}{k} \sum_{i=1}^k \frac{\rad_i}{\sqrt{\lambda_i}} &= \max_{j=0,\dots,\log_2(n)} \max_{k=[2^j, 2^{j+1})} \frac{1}{k} \sum_{i=1}^k \frac{\rad_i}{\sqrt{\lambda_i}}  \\
    &\leq \max_{j=0,\dots,\log_2(n)} \max_{k=[2^j, 2^{j+1})} \frac{1}{2^j} \sum_{i=1}^k \frac{\rad_i}{\sqrt{\lambda_i}} \\
    &= \max_{j=0,\dots,\log_2(n)} \left( \frac{1}{2^j} \sum_{i=1}^{2^j -1} \frac{\rad_i}{\sqrt{\lambda_i}}  + \max_{k=[2^j, 2^{j+1})} \frac{1}{2^j} \sum_{i=2^j}^k \frac{\rad_i}{\sqrt{\lambda_i}} \right) \\
    &\leq \max_{j=0,\dots,\log_2(n)} \left( \frac{1}{2^j} \sum_{i=1}^{2^j -1} \frac{\rad_i}{\sqrt{\lambda_i}} \right)  +  \max_{j=0,\dots,\log_2(n)} \left( \max_{k=[2^j, 2^{j+1})} \frac{1}{2^j} \sum_{i=2^j}^k \frac{\rad_i}{\sqrt{\lambda_i}} \right).
\end{align*}
We now take $\lambda_i = \frac{1}{i \log_2(n)}$. 
Note that $\sum_{i=1}^n \lambda_i \geq 1/2$ and for any $k>1$ we have $\sum_{i=1}^{k-1} \frac{1}{\lambda_i} \leq k^2 \log_2(n)/2$.

For any $\nu$ we have
\begin{align*}
    \E\left[ \exp\left( \nu \frac{1}{2^j} \sum_{i=1}^{2^j -1} \frac{\rad_i}{\sqrt{\lambda_i}} \right) \right] &=  \exp\left( \nu^2 \frac{1}{2^{2j}} \sum_{i=1}^{2^j -1} \frac{1}{\lambda_i} \right)   \leq \exp\left( \nu^2 \log_2(n)/2\right)
\end{align*}
so we apply Proposition~1 with $\sigma_t^2 = \log_2(n)$ and $\mc{T} = \{0,1,\dots,\log_2(n)\}$ to obtain
\begin{align*}
    \E\left[ \max_{j=0,\dots,\log_2(n)} \left( \frac{1}{2^j} \sum_{i=1}^{2^j -1} \frac{\rad_i}{\sqrt{\lambda_i}} \right) \right] \leq \sqrt{ 2 \log_2(n) \log(\log_2(2n))}.
\end{align*}

The second term deserves a bit more care. 
First, note that
\begin{align*}
    \max_{j=0,\dots,\log_2(n)} \left( \max_{k=[2^j, 2^{j+1})} \frac{1}{2^j} \sum_{i=2^j}^k \frac{\rad_i}{\sqrt{\lambda_i}} \right) &\leq  \max_{j=0,\dots,\log_2(n)} \left( \max_{k=[2^j, 2^{j+1})} \sqrt{\frac{2 \log_2(n)}{2^j}} \sum_{i=2^j}^k \rad_i \right) \\
    &=  \max_{j=0,\dots,\log_2(n)} \left( \max_{k=1,\dots,2^j} \sqrt{\frac{2 \log_2(n)}{2^j}} \sum_{i=1}^k \rad_i^{(j)} \right)
\end{align*}
where each $\{ \rad_i^{(j)} \}_{i=1}^{2^j}$ are i.i.d. sequences of $N(0,1)$.
Now
\begin{align*}
  \hspace{4cm} &  \hspace{-4cm}    \P\left( \max_{j=0,\dots,\log_2(n)} \left( \max_{k=1,\dots,2^j} \sqrt{\frac{2 \log_2(n)}{2^j}} \sum_{i=1}^k \rad_i^{(j)} \right) > t \right) \\ 
    & \leq \sum_{j=0,\dots,\log_2(n)} \P\left( \max_{k=1,\dots,2^j} \sqrt{\frac{2 \log_2(n)}{2^j}} \sum_{i=1}^k \rad_i^{(j)}  > t \right) \\
    &\leq \sum_{j=0,\dots,\log_2(n)} \exp( -t^2 /4 \log_2(n) ) \\
    &= \log_2(2n) \exp( -t^2 /4 \log_2(n) )
\end{align*}
where the last inequality follows from Doob's maximal inequality.
Thus, 
\begin{align*}
   \hspace{4cm}  &  \hspace{-4cm}  \E\left[ \max_{j=0,\dots,\log_2(n)} \left( \max_{k=1,\dots,2^j} \sqrt{\frac{2 \log_2(n)}{2^j}} \sum_{i=1}^k \rad_i^{(j)} \right)\right] \\
    &\leq \int_{t=0}^\infty \P\left( \max_{j=0,\dots,\log_2(n)} \left( \max_{k=1,\dots,2^j} \sqrt{\frac{2 \log_2(n)}{2^j}} \sum_{i=1}^k \rad_i^{(j)} \right) > t \right) dt \\
    &\leq \int_{t=0}^\infty \min\{ 1, \log_2(2n) \exp( -t^2 /4 \log_2(n) )\} dt \\
    &\leq a + \log_2(2n) \sqrt{4 \pi \log_2(n)}  \int_{t=a}^\infty \frac{1}{\sqrt{4 \pi \log_2(n)}} \exp( -t^2 /4 \log_2(n) )\} dt \\
    &\leq a + \log_2(2n) \sqrt{4 \pi \log_2(n)}  \exp( -a^2 /4 \log_2(n) )\} \\
    &\leq \sqrt{4 \log_2(n) \log( \log_2(2n))} + \sqrt{4 \pi \log_2(n)}
\end{align*}
for $a = \sqrt{4 \log_2(n) \log( \log_2(2n))}$.

Putting it all together we have 
\begin{align*}
    \gamma_\star &\leq \frac{4}{\epsilon^2}\left( \sqrt{2 \log_2(n) \log( \log_2(2n))} + \sqrt{4 \log_2(n) \log( \log_2(2n))} + \sqrt{4 \pi \log_2(n)}  \right)^2\\
    &\leq \frac{194 \log_2(n) \log( \log_2(2n)) }{\epsilon^2}
\end{align*}

\section{Implementation Details} \label{sec:implement}
In this section we discuss the modifications and implementation details of Algorithm~\ref{alg:fixed_budget_comp}. We consider the slightly modified version present in Algorithm~\ref{alg:aced_mod}.

\subsection{Optimization}
First, we reiterate the following definition from Section~\ref{sec:efficient_optimization}: \[f(\lambda; h; \rad):= \frac{ \sum_{i \in [n]} (\tilde{h}_k(x_i)- h(x_i)) \frac{\rad_i}{n\lambda_i^{1/2}} }{2^{-k+1} +  \wt{\err}(h, \hat{\eta}_{k-1}) - \wt{\err}( \tilde{h}_k, \hat{\eta}_{k-1})}.\]

\subsubsection{Mirror Descent}
In Algorithm~\ref{alg:fixed_budget_comp}, we need to solve the optimization problem in~\eqref{eq:action_comp_2}, namely, \[\inf_{\lambda\in\Delta_{n}}\E_{\zeta\sim N(0,I)}[\max_{h\in \H} f(\lambda;h;\rad)].\] 
We use stochastic mirror descent method. 
Specifically, given a current iterate $\widehat{\lambda}$ and an i.i.d. sample $\zeta_1, \cdots, \zeta_{B}$ we computed an unbiased estimate of the sub-gradient 
\[ g(\widehat{\lambda}) = \frac{1}{B}\sum_{i=1}^B \nabla_{\lambda} \max_{h\in H} f(\lambda;h;\rad_i)\Big|_{\lambda = \widehat{\lambda}}\quad,\]
and then move to $\widehat{\lambda}_+ = \Pi( \exp( \log( \widehat{\lambda} ) - \eta g(\widehat{\lambda}) ) )$ where $\Pi(x) = x / \|x \|_1$ and $\log,\exp$ are applied element-wise.
The step size $\eta$ is chosen by back-tracking line search where two step sizes are equivalent if the difference in estimated function values is dominated by the the square root of the empirical variance, with respect to the finite batch size estimates.   

The batch size was grown adaptively. Let $\bar{f}(\lambda) = \E_{\zeta\sim N(0,I)}[ \max_{h \in \mc{H}} f(\lambda; h; \rad) ]$,  $\lambda_* \in \arg\min_{\lambda } \bar{f}(\lambda)$ and suppose the algorithm is at some current configuration $\widehat{\lambda}$. By convexity of $\bar{f}$ we have
\begin{align*}
    \bar{f}(\widehat{\lambda}) - \bar{f}(\lambda^*) &\leq  \langle  \nabla \bar{f}(\widehat{\lambda}) ,  \widehat{\lambda} - \lambda_*\rangle  \\
    &=  \langle  \nabla \bar{f}(\widehat{\lambda}) - g(\widehat{\lambda}),  \widehat{\lambda} - \lambda_*\rangle + \langle g(\widehat{\lambda}),  \widehat{\lambda} - \lambda_*\rangle \\
    &\leq 2 \max_k \big| [ \nabla \bar{f}(\widehat{\lambda}) - g(\widehat{\lambda}) ]_k \big| + \max_k \langle g(\widehat{\lambda}),  \widehat{\lambda} - \mb{e}_k \rangle \\
    &\approx 2 \max_k \widehat{\sigma}_k + \max_k \langle g(\widehat{\lambda}),  \widehat{\lambda} - \mb{e}_k \rangle 
\end{align*}
where $\widehat{\sigma}_k^2$ is the empirical variance of $[g(\widehat{\lambda}) ]_k$, namely, the sample variance of $\{ [ \nabla_{\lambda} \max_{h\in H} f(\lambda;h;\rad_i)\Big|_{\lambda = \widehat{\lambda}} ]_k \}_{i=1}^B$ divided by $B$.
Note that $\widehat{\sigma}_k = O(1/\sqrt{B})$.
Thus, we double the batch size $B \mapsto 2B$ whenever $2 \max_k \widehat{\sigma}_k \geq \max_k \langle g(\widehat{\lambda}),  \widehat{\lambda} - \mb{e}_k \rangle$. 
This also motivates our stopping condition: for input $\epsilon$, terminate when $2 \max_k \widehat{\sigma}_k + \max_k \langle g(\widehat{\lambda}),  \widehat{\lambda} - \mb{e}_k \rangle  \leq \epsilon$.
Because $B$ is increasing over time, this stopping condition will always, eventually, be met.

For a fixed $\zeta_i$, to compute the corresponding gradient, we computed $\bar{h} = \argmax_{h\in \mc{H}} f(\lambda;h;\rad)$ and used the fact that  $\nabla_{\lambda} \max_{h\in \H} f(\lambda;h;\rad) = \nabla_{\lambda} f(\lambda;\bar{h};\rad_i)$. As described in the next section, finding $\tilde{h}$ is difficult due to the use of surrogate loss functions. We provide a line search method, $\text{LineSearch}(\lambda, \rad)$, that finds an approximate value for it. Together, the full optimization is as follows:

\begin{algorithm}[H]\small
\begin{algorithmic}
\STATE {\bfseries Goal:} Solve for $\max_\lambda \E_{\zeta \sim N(0, I)}[f(\lambda; h; \zeta)]$, where $f$ inherently depends on $\tilde{h}_k$ and $\hat{\eta}_{k - 1}$.

\STATE {\bfseries Input:} Tolerance $\epsilon$ for stopping criteria.

\STATE {\bfseries Initialize:} $\widehat{\lambda} \longleftarrow \frac{1}{d}\1$.

\REPEAT

\STATE Sample $\rad_1, ..., \rad_B \sim N(0, I)$.

\STATE Compute $\bar{h}_i = \text{LineSearch}(\widehat{\lambda}, \zeta_i)$ for all $i \in [B]$ and their corresponding gradient estimates $\nabla_{\lambda} f(\widehat{\lambda}; \bar{h}_i, \rad_i)$.

\STATE $g(\widehat{\lambda}) \longleftarrow \frac{1}{B} \sum_{i = 1}^B \nabla_{\lambda} f(\widehat{\lambda}; \bar{h}_i, \rad_i)$.

\STATE $\widehat{\lambda} \longleftarrow \Pi( \exp( \log( \widehat{\lambda} ) - \eta g(\widehat{\lambda}) ) )$ where $\eta$ is chosen as described above.

\IF{$2 \max_k \widehat{\sigma}_k \geq \max_k \langle g(\widehat{\lambda}),  \widehat{\lambda} - \mb{e}_k \rangle$}

\STATE $B \longleftarrow 2B$.

\ENDIF

\UNTIL{$2 \max_k \widehat{\sigma}_k + \max_k \langle g(\widehat{\lambda}),  \widehat{\lambda} - \mb{e}_k \rangle  \leq \epsilon$}

\end{algorithmic}
\caption{$\text{SMD}(\tilde{h}_k, \hat{\eta}_{k - 1})$.}

\end{algorithm}

\subsubsection{Line Search}
 Then, computing $\max_{h \in \mc{H}} f(\lambda; h; \rad)$ is equivalent to
\begin{align*}
&\min_{r \in \R^+} r \quad \text{subject to} \enspace g(r)\leq 0
\end{align*}
where $a = -2^{-k+1} - \sum_{i\in [n]} (1 - 2\hat{\eta}_{k, i})\tilde{h}_k(x_i)$, $b = \sum_{i\in [n]} \frac{\rad}{n\lambda_i^{1/2}}\tilde{h}_k(x_i)$, $c_i=1 - 2\hat{\eta}_{k-1, i}$, $d_i=-\frac{\rad}{n\lambda_i^{1/2}}$ and 
\[g(r) = a r + b + \max_{h \in \H}\, \sum_{i \in [n]} (c_i r + d_i) h(x_i).\]
In particular at the optimal value of $r$, called $r^{\ast}$, $\argmax_{h \in \mc{H}} f(\lambda; h; \rad) = \argmax_{h \in \mc{H}} g(r^{\ast})$. 

As shown in Section~\ref{sec:efficient_optimization}, given access to a weighted classification oracle, computing $\max_{h \in \H}\, \sum_{i \in [n]} (c_i r + d_i) h(x_i)$ is equivalent to a $0/1$-loss minimization problem that is solvable using a weighted classification oracle. If we had such an oracle then since $g(r)$ is monotonically decreasing as a function of $r$ we can use a binary search procedure to solve this optimization problem. Indeed, for any given range $[r_{\min}, r_{\max}]$ where the optimal $r$ lies in, we are checking if $g(\frac{r_{\min} + r_{\max}}{2}) \leq 0$. If that is the case, we just set $r_{\min}\leftarrow \frac{r_{\min} + r_{\max}}{2}$, otherwise, we set $r_{\max}\leftarrow \frac{r_{\min} + r_{\max}}{2}$. We then check the sign of $g(\frac{r_{\min} + r_{\max}}{2})$ again and repeat this procedure until a sufficient tolerance is met.

However, in practice, we do not have access to a weighted $0/1$ loss oracle and must employ a convex surrogate loss that may not correctly solve the weighted classification problem. To be concrete, in all of our experiments we used Scikit-learn's \texttt{LogisticRegression} classifier. Given such a surrogate, which we denote as $\widetilde{\max}_{h\in H}$ to acknowledge that it may not find the optimal $h$, the resulting function
\[\tilde{g}(r) = a r + b +\widetilde{\max_{h \in \H}}\, \sum_{i \in [n]} (c_i r + d_i) h(x_i)\]
may no longer be monotonically decreasing in $r$ hence a binary search procedure would fail. However intuitively it suffices to look at a large enough set of $r$'s near a zero $\tilde{g}(r)$. To overcome this issue, we used the procedure in Algorithm~\ref{alg:linesearch}.

Algorithm~\ref{alg:linesearch} overcomes this issue by considering a large set of potential $r$ values and for each value computing the corresponding $h_r = \argmax \tilde{g}(r)$ and stores these in the array $S$. It then returns $\argmax_{h\in S} f(\lambda;h;\zeta)$. The set of $r$ values considered is chosen by a multi-scale procedure that looks at $r$ values on a finer and finer geometric grid given a budget $N_{\max}$.
\begin{algorithm}[H]\small
\begin{algorithmic}
\STATE {\bfseries Goal:} Solve for $\widetilde{\max}_{h \in \H} a r + b + \sum_{i \in [n]} (c_i r + d_i) h(x_i)$.
\STATE Compute $\bar{h} = \widetilde{\max}_{h \in \H}\, \sum_{i \in [n]} (c_i r + d_i) h(x_i)$ by the relaxed weighted classification oracle.
\STATE Set $\widehat{g} \leftarrow a r + b + \sum_{i \in [n]} (c_i r + d_i) \bar{h}(x_i)$.
\STATE $S \leftarrow S \cup \{\bar{h}\}$.
\STATE {\bfseries Return:} $\widehat{g}, S$.
\end{algorithmic}
\caption{Oracle(r, S).}
\end{algorithm}

\begin{algorithm}[H]\small
\begin{algorithmic}
\STATE {\bfseries Goal:} Solve for $\max_{h \in \mc{H}} f(\lambda; h; \rad)$.
\STATE {\bfseries Input:} fixed $\lambda$ and $\rad$; maximum number of iterations $N_{\max}$; tolerance $\epsilon$.
\STATE {\bfseries Initialize:} $S \leftarrow \{\}$, $r \leftarrow 100$, $\gamma \leftarrow 10$, $\delta \leftarrow \sqrt{2}$, and $t \leftarrow 0$.
\STATE $\widehat{g}, S \leftarrow Oracle(r, S)$.
\WHILE{$\widehat{g} < 0$ and $t < N_{\max}$} 
    \STATE $r\leftarrow r / 2$. \COMMENT{$r$ is too large}
    \STATE $\widehat{g}, S \leftarrow Oracle(r, S)$.
    \STATE $t \leftarrow t + 1$.
\ENDWHILE
\FOR{$j = t, ..., N_{\max} - 1$}
    \STATE $\widehat{g}, S \leftarrow Oracle(r, S)$.
    \IF{$\widehat{g} > 0$}
    \STATE $r \leftarrow \gamma \cdot r$. \COMMENT{$r$ is too small, need to increase $r$ so that $\widehat{g}$ decreases.}
    \ELSE
    \STATE $r \leftarrow r / \gamma^2$. \COMMENT{Reaches a point where $r$ is too large, scale back.}
    \STATE $\gamma \leftarrow \gamma / \delta$. \COMMENT{Start searching with a finer grid scale.}
    \ENDIF
\ENDFOR
\STATE {\bfseries Return:} $\argmax_{h \in S} f(\lambda; h; \rad)$.
\end{algorithmic}
\caption{$\text{LineSearch}(\lambda, \rad)$.}
\label{alg:linesearch}
\end{algorithm}

\subsection{Sampling}
The last portion of our algorithm is the sampling scheme. The algorithm described in Algorithm~\ref{alg:fixed_budget_comp} does not reuse samples between rounds to compute an estimate for $\hat{\eta}_k$. In practice, this can be very wasteful and we instead want an estimator based on all samples up to and including those taken in round $k$. In each round the algorithm computes $\lambda_k$ with the goal of $\lambda_k\approx \lambda_{\ast}$ for large enough values of $k$ to ensure that we are sampling from the optimal distribution in that round. Hence, since we are recycling samples, we need to ensure that the distribution of \emph{all} samples taken by the end of round $k$, including those in previous rounds, match $\lambda_k$.

To ensure those, we use a waterfulling technique. We set $p_1=\lambda_1$ and then for each $k\geq 1$ we set 
\begin{align}
    p_k = \argmin_{q\in \Delta_n} \max_{j\leq n} \max \{0, k \cdot \lambda_{k, j} - (\sum_{i=1}^{k-1} p_{i, j}) - q_{j} \}.
\end{align}
Our algorithm being implemented is shown below:

\begin{algorithm}[H]\small
\begin{algorithmic}
\STATE {\bfseries Input:} Budget $T$, tolerance $\epsilon > 0$, batch size $N$ (default 250).
\STATE $\hat\eta_{0} \longleftarrow 0$ 
\FOR{$k=1,2,\ldots, \floor{\log_2(\epsilon^{-1})}$}
    \STATE $\tilde{h}_{k} \longleftarrow \argmin_{h \in \H} \wt{\err}(h, \hat{\eta}_{k-1})$.
    \STATE \textbf{Optimization}
    \STATE $\lambda_k \longleftarrow \text{SMD}(\tilde{h}_k, \hat{\eta}_{k - 1})$.
    \STATE \textbf{Sampling}
    
    \STATE Sample till observing $N$ unique $\{x_1^{(k)}, \ldots, x_{N}^{(k)} \}  \sim p_k$ that has not been queried before (from rounds $1, ..., k-1$).
    \STATE Query $x_1, \ldots, x_{N} $ and observe $y_1, \ldots, y_{N}$.
    \STATE Compute an estimate $\hat\eta_k$ with the naive estimator $\hat\eta_k^{(\text{Naive})}$ as defined in Section~\ref{sec:experiment}.
\ENDFOR
\STATE {\bfseries Return:} $ \argmin_{h \in \H} \wt{\err}^{(k)}(h)$
\end{algorithmic}

\caption{Fixed Budget ACED with Waterfilling.}
\label{alg:aced_mod}
\end{algorithm}

\subsection{Batched IWAL}

In this section we explain our implementation of the IWAL algorithm\cite{beygelzimer2010agnostic} (and variants such as IWAL1 and oracular variants) for streaming active algorithms. In round $k$ we assume access to a (labeled) dataset $S_{k-1}=\{(x_i, y_i,p_i)_{i=1}^{n_k}\}\subset \mc{X}\times \{\pm 1\}\times [0,1]$, with $n_k\leq k$. Given a new point from a stream, $x_k$, the decision to label $x_k$ is made by computing two hypothesis. Firstly we compute \[h_k=\argmin_{h\in \H} \sum_{i=1}^{n_k} \frac{\1\{h(x_i)\neq y_i\}}{p_i}\] and second we compute \[h'_k=\argmin_{\substack{h\in \H\\ h_k(x_k)\neq h'_k(x_k)} } \sum_{i=1}^{n_k} \frac{\1\{h(x_i)\neq y_i\}}{p_i}.\] Thus to compute $h'_k$ we need access to a weighted classification oracle for 0/1 loss that can handle a single constraint.

In general including the constraint $h_k(x_k)\neq h_k'(x_k)$ is not easy for arbitrary classes and past methods have considered the optimization 
\[\argmin_{h\in \H } \sum_{i=1}^{n_k} \frac{\1\{h(x_i)\neq y_i\}}{p_i} + q \1\{h(x_k)\neq h_k(x_k)\} \]
for a sufficiently large weight $q\in \R$ to ensure the constraint\cite{karampatziakis2010online} . 

However in the case of linear classes under the surrogate convex logistic loss, the main focus of experiments in this paper, we take a different approach.
Assume that $\mc{X}\subset \mathbb{R}^p$ and that $\H = \{h(x) = w^{\top}x+b:w\in \R^p, b\in \R\}$. W.l.o.g. assume that $h_k(x_k)=-1$. Under this convex relaxation, we seek a classifier where $w^{\top}x_k+b = \epsilon$, where $\epsilon\geq 0$ i.e. the linear predictor flips the predicted sign of $x_k$ and has margin $\epsilon$. Thus we learn (an approximate) $h_k'$ by solving
\begin{align*}
    \min_{\substack{w\in \R^p,b\in \R\\ w^{\top}x+b=\epsilon}} \sum_{i=1}^{n_k} \frac{1}{p_i} \log(1+\exp(-y_i(w^{\top}x+b)))
    &= \min_{w\in\R^p} \sum_{i=1}^{n_k} \frac{1}{p_i} \log(1+\exp(-y_i(w^{\top}x_i-w^{\top}x_k+\epsilon)))\\
    &= \min_{w\in\R^p} \sum_{i=1}^{n_k} \frac{1}{p_i} \log(1+\exp(-y_i(w^{\top}(x_i-x_k)+\epsilon)))
\end{align*}
This is a convex optimization problem with no constraints that is easily solvable using a procedure for logistic regression where we assume that the intercept is some $\epsilon$ with really small magnitude.

\section{Full Scale Plots for Performances on the Pool} \label{sec:full_scale}
\begin{figure}[H]
    \centering
    \includegraphics[width=.7\linewidth]{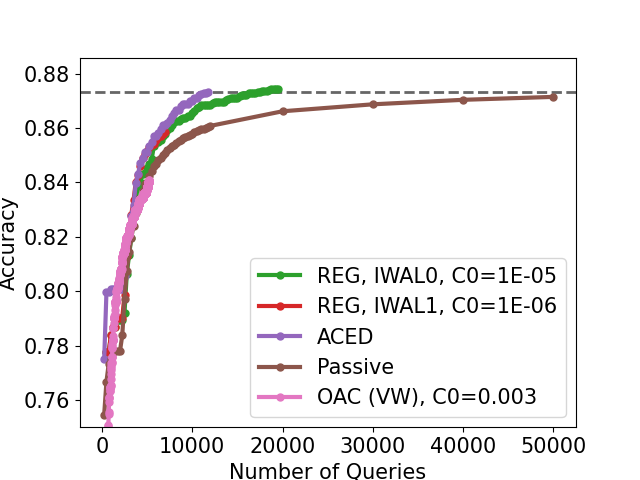}
    \caption{Full scale of Figure~\ref{fig:mnist}}
\end{figure}
\begin{figure}[H]
    \centering
    \includegraphics[width=.7\linewidth]{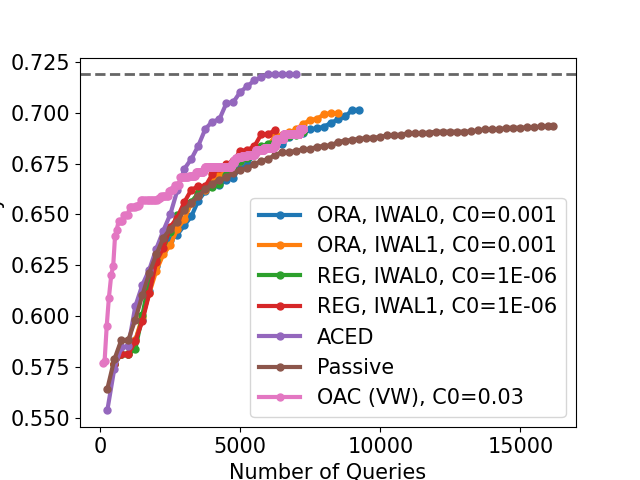}
    \caption{Full scale of Figure~\ref{fig:svhn}}
\end{figure}
\begin{figure}[H]
    \centering
    \includegraphics[width=.7\linewidth]{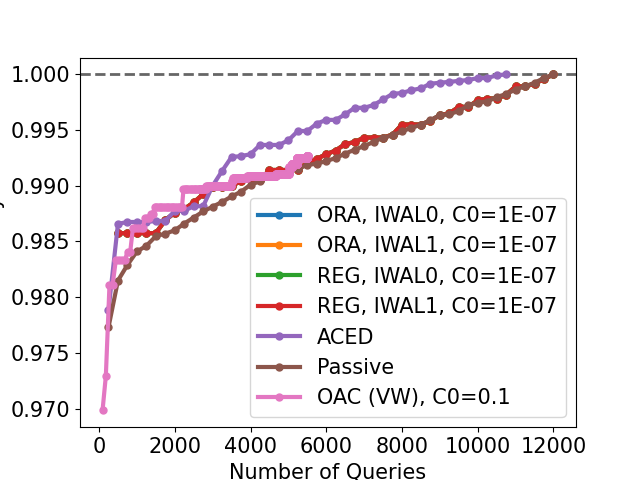}
    \caption{Full scale of Figure~\ref{fig:fashion}}
\end{figure}
\begin{figure}[H]
    \centering
    \includegraphics[width=.7\linewidth]{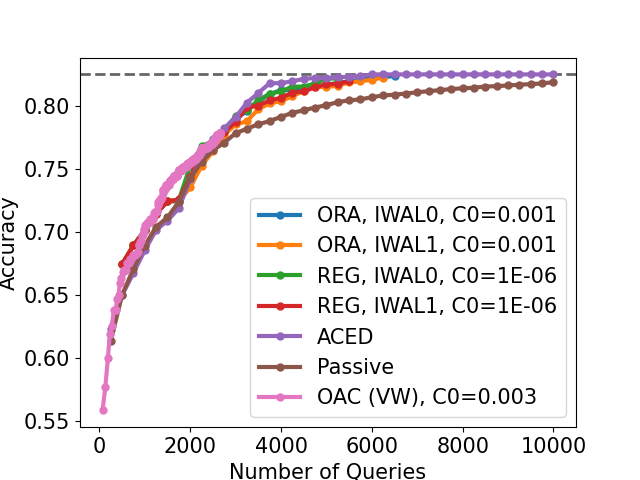}
    \caption{Full scale of Figure~\ref{fig:cifar}}
\end{figure}

\section{Hyperparameters for Baselines} \label{sec:base_hparam}
We searched in the following $C_0$ that uses the same grid fineness as \citep{huang2015efficient}:
\begin{table}[H]
    \centering
    \begin{tabular}{ccccc}
        Baseline & MNIST & SVHN & FashionMNIST & CIFAR \\
        \hline
        IWAL-0\&1 & $10^{-7}, 10^{-6}, ..., 1$ & $10^{-7}, 10^{-6}, ..., 1$ & $10^{-8}, 10^{-7}, ..., 1$ & $10^{-7}, 10^{-6}, ..., 1$\\
        ORA-IWAL-0\&1 & N/A & $10^{-4}, 10^{-3}, ..., 10^{-1}$ & $10^{-8}, 10^{-7}, ..., 1$ & $10^{-4}, 10^{-3}, ..., 10^{-1}$\\
        OAC & $.001, .003, .01, .03$ & $.01, .03, .1, .3$ & $.001, .003, ..., 1$ & $.001, .003, .01, .03$
    \end{tabular}
    \caption{Ranges of $C_0$ searched for different experiments.}
\end{table}

\begin{figure}[H]
    \centering
    \includegraphics[width=.7\linewidth]{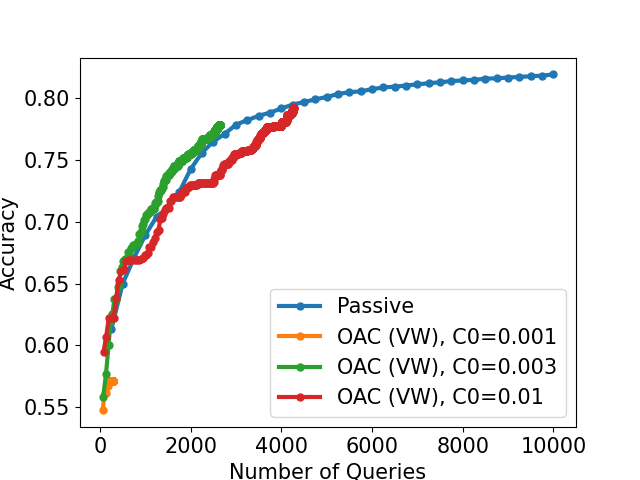}
    \caption{Sensitivity of OAC to $C_0$ on the CIFAR 2 vs 7 dataset (training accuracy).}
    \label{fig:cifar_sensitivity}
\end{figure}

We used the following amount of passes over dataset for the following baselines. For OAC, we made sure the number of passes is sufficient enough so that the algorithm is no longer taking more queries.
\begin{table}[H]
    \centering
    \begin{tabular}{ccccc}
        Baseline & MNIST & SVHN & FashionMNIST & CIFAR \\
        \hline
        IWAL-0\&1 and ORA-IWAL-0\&1 & $1$ (N/A for ORA variants) & $2$ & $2$ & $2$\\
        OAC & $5$ & $10$ & $10$ & $10$
    \end{tabular}
    \caption{Ranges of $C_0$ searched for different experiments.}
\end{table}
For Vowpal Wabbit, we used an initial learning rate of $0.5$ for CIFAR, and $1$ for every other experiments.

\section{Generalization Performance on Holdout Set} \label{sec:test_perf}
In the following figures, we show performances of the algorithms on a holdout test set. We note that there's no algorithm that is consistently the best among all four experiments, but ACED is consistently among the top two algorithms, and is the best in two of the four experiments. We also reiterate that the OAC curves have stopped taking more queries at the end, so their final accuracies are inferior than other algorithms in most cases.

\begin{figure}[H]
    \centering
    \includegraphics[width=.7\linewidth]{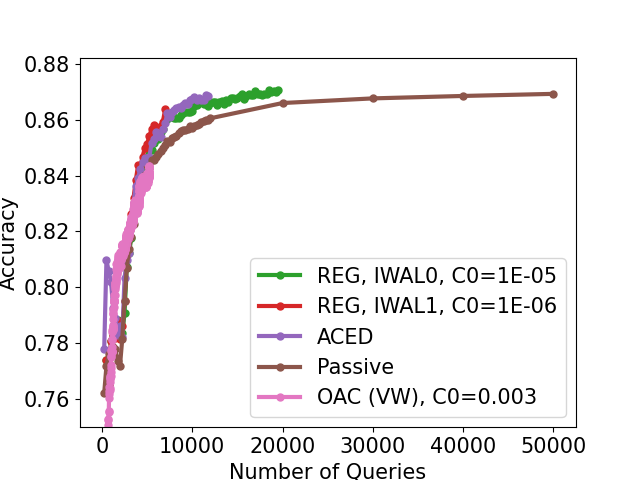}
    \caption{MNIST performance on test set}
\end{figure}
\begin{figure}[H]
    \centering
    \includegraphics[width=.7\linewidth]{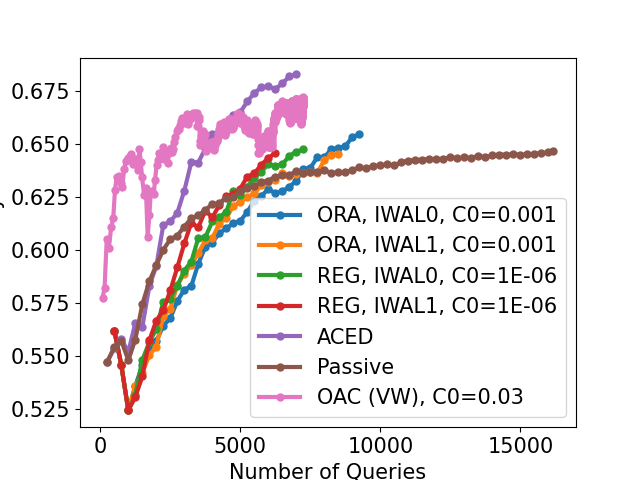}
    \caption{SVHN performance on test set}
\end{figure}
\begin{figure}[H]
    \centering
    \includegraphics[width=.7\linewidth]{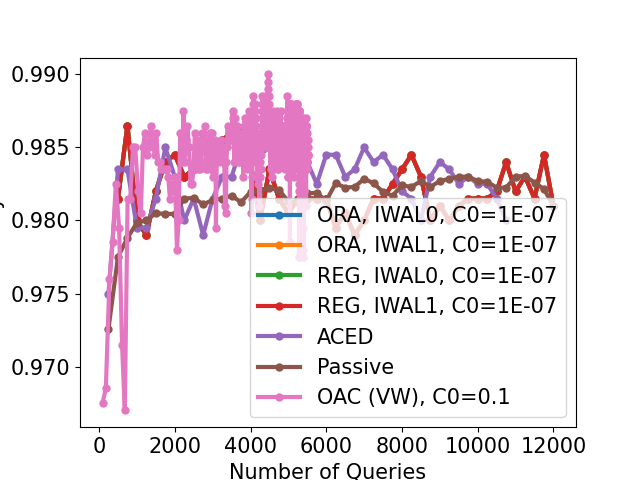}
    \caption{FashionMNIST performance on test set}
\end{figure}
\begin{figure}[H]
    \centering
    \includegraphics[width=.7\linewidth]{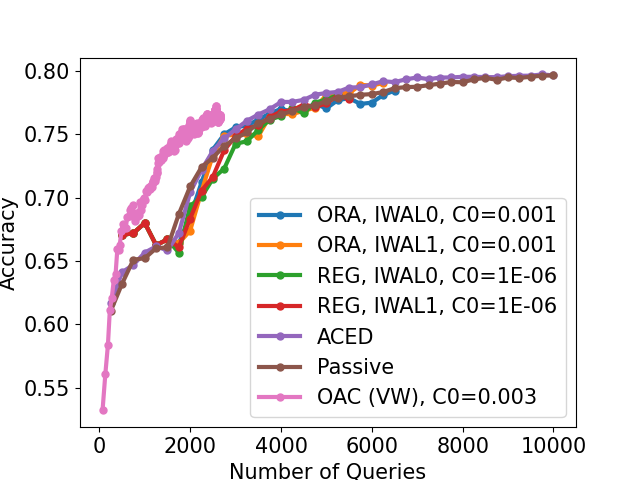}
    \caption{CIFAR performance on test set}
\end{figure}

\clearpage

\end{document}